\documentclass[twoside,11pt]{article}

\usepackage{blindtext}

\usepackage{jmlr2e}

\usepackage{lastpage}
\jmlrheading{27}{2026}{1-\pageref{LastPage}}{8/25; Revised
10/26}{10/26}{25-2066}{Yingzhen Yang and Ping Li}
\ShortHeadings{title}{Yang and Li}

\hypersetup{colorlinks=true,citecolor=blue,linkcolor=blue}
\newcommand{\circled}[1]{\small{\raisebox{.6pt}{\textcircled{\raisebox{-.8pt}{#1}}}}}

\usepackage{amsmath,mathrsfs,dsfont}
\usepackage{nicefrac}
\usepackage{algorithm}
\usepackage{algorithmicx}
\usepackage{algpseudocode}

\usepackage{booktabs}

\usepackage{color}
\usepackage{enumitem}

\PassOptionsToPackage{square,sort,comma,numbers}{natbib}
\usepackage{graphicx,tikz}
\usepackage[mathscr]{euscript}
\let\proof\relax

\usepackage{amsthm}

\usepackage{bm}
\usepackage{bbm}
\usepackage{color}

\usepackage{color}
\usepackage{epstopdf}
\usepackage[capitalize,noabbrev]{cleveref}
\usepackage{thmtools}
\usepackage{thm-restate}

\usepackage{bbding}
\usepackage{mathtools}
\usepackage{wrapfig}
\usepackage{colortbl}
\usepackage{slashbox}
\usepackage[htt]{hyphenat}
\allowdisplaybreaks
\newcommand{\cfrakR}{\mathfrak{R}} %\usepackage{eufrak}
\newcommand{\relu}[1]{\sigma\pth{#1}}

\newcommand{\bbx}{\overset{\rightharpoonup}{\bx}}
\newcommand{\bbw}{\overset{\rightharpoonup}{\bw}}
\newcommand{\bbe}{\overset{\rightharpoonup}{\be}}
\newcommand{\bbq}{\overset{\rightharpoonup}{\bq}}

\newcommand{\cFnn}{\cF_{\mathop\mathrm{NN}}}
\newcommand{\cFext}{\cF_{\mathop\mathrm{ext}}}

\newcommand{\Kint}{{K}^{\mathop\mathrm{(int)}}}
\newcommand{\hKint}{{\hat K}^{\mathop\mathrm{(int)}}}
\newcommand{\bKint}{{\bK}^{\mathop\mathrm{(int)}}}
\newcommand{\hbKint}{{\hat \bK}^{\mathop\mathrm{(int)}}}
\newcommand{\bEint}{{\bE}^{\mathop\mathrm{(int)}}}
\newcommand{\bSigmaint}{{\bSigma}^{\mathop\mathrm{(int)}}}
\newcommand{\bUint}{{\bU}^{\mathop\mathrm{(int)}}}
\newcommand{\lambdaint}{{\lambda}^{\mathop\mathrm{(int)}}}
\newcommand{\hlambdaint}{{\hat \lambda}^{\mathop\mathrm{(int)}}}

\newcommand{\tK}{\tilde K}
\newcommand{\cHKint}{\cH_{\Kint}}
\newcommand{\vint}{v^{\mathop\mathrm{(int)}}}
\newcommand{\range}{\mathop\mathrm{range}}

\graphicspath{{illustrations/}}
\makeatletter
\@for\jmlr@env:=theorem,lemma,remark,proposition,corollary,example,axiom,definition\do{%
  \expandafter\let\csname\jmlr@env\endcsname\relax
  \expandafter\let\csname end\jmlr@env\endcsname\relax
  \expandafter\let\csname c@\jmlr@env\endcsname\relax
  \expandafter\let\csname the\jmlr@env\endcsname\relax
}
\makeatother
\newcounter{optproblem}

\newtheoremstyle{mytheoremstyle} % name
    {\topsep}                    % Space above
    {\topsep}                    % Space below
    {\normalfont}                % Body font
    {}                           % Indent amount
    {\bfseries}                   % Theorem head font
    {.}                          % Punctuation after theorem head
    {.5em}                       % Space after theorem head
    {}  % Theorem head spec (can be left empty, meaning 'normal�)

\theoremstyle{mytheoremstyle}
\newtheorem{theorem}{Theorem}[section]
\newtheorem{remark}[theorem]{Remark}
\newtheorem{proposition}[theorem]{Proposition}
\newtheorem{corollary}[theorem]{Corollary}

\newtheorem*{theorem*}{Theorem}
\newtheorem*{lemma*}{Lemma}
\newtheorem*{remark*}{Remark}

\newtheorem{lemma}[theorem]{Lemma}%[section]

\theoremstyle{remark}
\newtheorem{definition}{Definition}[section]

\DeclareMathAlphabet{\pazocal}{OMS}{zplm}{m}{n}
\DeclareMathAlphabet{\mathpzc}{OMS}{pzc}{m}{it}

\setlist[itemize]{leftmargin=*}

\renewcommand{\hat}{\widehat}

\newcommand{\bfm}[1]{\ensuremath{\mathbf{#1}}}
\newcommand{\bfsym}[1]{\ensuremath{\boldsymbol{#1}}}

\def\ba{\boldsymbol a}   \def\bA{\bfm A}

   \def\bD{\bfm D}  
\def\be{\bfm e}   \def\bE{\bfm E}  
    
   \def\bG{\bfm G}  
   \def\bH{\bfm H}  
   \def\bI{\bfm I}  
     
   \def\bK{\bfm K}  
     
   \def\bM{\bfm M}  
     
   \def\bO{\bfm O}  
     
\def\bq{\bfm q}   \def\bQ{\bfm Q}  
   \def\bR{\bfm R}  \def\RR{\mathbb{R}}
\def\bs{\bfm s}   \def\bS{\bfm S}  
     
\def\bu{\bfm u}   \def\bU{\bfm U}  
\def\bv{\bfm v}     
\def\bw{\bfm w}   \def\bW{\bfm W}  
\def\bx{\bfm x}   \def\bX{\bfm X}  
\def\by{\bfm y}     
\def\bz{\bfm z}   \def\bZ{\bfm Z}  
\def\bzero{\bfm 0}

 \def\cB{{\cal  B}}

 \def\cE{{\cal  E}}
 \def\cF{{\cal  F}}
 
 \def\cH{{\cal  H}}
 \def\cI{{\cal  I}}

 \def\cN{{\cal  N}}
 \def\cO{{\cal  O}}

 \def\cR{{\cal  R}}

 \def\cV{{\cal  V}}
 \def\cW{{\cal  W}}
 \def\cX{{\cal  X}}

\def\bbeta{\bfsym \beta}

\def\bmu{\bfsym {\mu}}

\def\btheta{\bfsym {\theta}}

\def\bSigma{\bfsym \Sigma}

\def\+#1{\mathcal{#1}}
\def\-#1{\textup{#1}}

\def\set#1{\left\{ #1 \right\}}
\def\pth#1{\left( #1 \right)}
\def\bth#1{\left[ #1 \right]}
\def\abth#1{\left | #1 \right |}

\def\defeq {\coloneqq}

\newcommand{\vect}[1]{{\textup{vec}\pth{#1}}}

\def \longmid {\,\middle\vert\,}
\DeclareMathSymbol{\relcolon}{\mathrel}{operators}{"3A}

\newcommand{\diag}{\textup{diag}}
\def\dim{\textup{dim}}

\newcommand{\La}{\left\langle\kern-0.64ex\left\langle}
\newcommand{\Ra}{\right\rangle\kern-0.64ex\right\rangle}

\def\Norm#1#2{{\left\vert\kern-0.4ex\left\vert\kern-0.4ex\left\vert #1
    \right\vert\kern-0.4ex\right\vert\kern-0.4ex\right\vert}_{#2}}
\def\norm#1#2{{\left\|#1\right\|}_{#2}}

\def\ltwonorm#1{\norm{#1}{2}}
\def\fnorm#1{\norm{#1}{\textup{F}}}
\def\supnorm#1{\norm{#1}{\infty}}

\newcommand{\rank}{\textup{rank}}
\def \range {\textup{range}}

\def \supp#1{\textup{supp}\left(#1\right)}

\def\tr#1{\textup{tr}\left(#1\right)}

\newcommand{\1}{{\rm 1}\kern-0.25em{\rm I}}
\def\indict#1{{\rm 1}\kern-0.25em{\rm I}_{\set{#1}}}

\def \eps  {\epsilon}
\def \eps {\varepsilon}

\def \diff {{\rm d}}
\def \iprod#1#2{\left\langle #1, #2 \right\rangle}
\def\set#1{\left\{#1\right\}}

\def\ball#1#2#3{\bfm{B}^{#1}\left(#2; #3\right)}

\def\unitsphere#1{\mathbb{S}^{#1}}
\def \E {\mathbb{E}}
\def\Expect#1#2{\E_{#1}\left[#2\right]}

\def \Pr {\textup{Pr}}
\newcommand{\Prob}[1]{\Pr\left[#1\right]}
\def \Var#1{\textup{Var}\left[#1\right]}

\def \lsim {\lesssim}
\def \gsim {\gtrsim}

\newcommand{\Unif}[1]{{\mathop\mathrm{Unif}}\left( #1 \right)}

\newcommand{\beq}{\begin{equation}}
\newcommand{\eeq}{\end{equation}}
\newcommand{\beqa}{\begin{eqnarray}}
\newcommand{\eeqa}{\end{eqnarray}}
\newcommand{\beqas}{\begin{eqnarray*}}
\newcommand{\eeqas}{\end{eqnarray*}}
\def\bal#1\eal{\begin{align}#1\end{align}}
\def\bals#1\eals{\begin{align*}#1\end{align*}}
\def\bsal#1\esal{\begin{small}\begin{align}#1\end{align}\end{small}}
\def\bsals#1\esals{\begin{small}\begin{align*}#1\end{align*}\end{small}}
\def\bsfal#1\esfal{\begin{small}\begin{flalign}#1\end{flalign}\end{small}}

\providecommand{\BlackBox}{\rule{1.5ex}{1.5ex}}  % end of proof
\ifdefined\proof
    \renewenvironment{proof}{\par\noindent{\bf Proof\ }}{\hfill\BlackBox\\[2mm]}
\else
    \newenvironment{proof}{\par\noindent{\bf Proof\ }}{\hfill\BlackBox\\[2mm]}
\fi

\ShortHeadings{PGD Finds Optimal Neural Networks for Nonparametric Regression}{Yang and Li}
\firstpageno{1}

\begin{document}

\title{Sharp Generalization for Nonparametric Regression in Interpolation Space by Over-Parameterized Neural Networks Trained with Preconditioned Gradient Descent and Early Stopping}

\author{\name Yingzhen Yang \email yingzhen.yang@asu.edu \\
       \addr School of Computing and Augmented Intelligence\\
       Arizona State University\\
       Tempe, AZ 85281, USA
       \AND
       \name Ping Li \email pingli98@gmail.com  \\
       \addr VecML Inc.\\
       Bellevue, WA 98004, USA}

\editor{Brian Kulis}

\maketitle

\begin{abstract}%   <- trailing '%' for backward compatibility of .sty file
We study nonparametric regression using an over-parameterized two-layer neural network trained with algorithmic guarantees in this paper. We consider the setting where the training features are drawn uniformly from the unit sphere in $\RR^d$, and the target function lies in an interpolation space commonly studied in statistical learning theory. We demonstrate that training the neural network with a novel Preconditioned Gradient Descent (PGD) algorithm, equipped with early stopping, achieves a sharp regression rate of
$\cO(n^{-\frac{2\alpha s'}{2\alpha s'+1}})$, where $\alpha$ parametrizes the polynomial eigenvalue decay of the regular Neural Tangent Kernel (NTK), $\lambda_j\asymp j^{-2\alpha}$, and $2\alpha=d/(d-1)$ under the spherical uniform design considered here. The target function is assumed to belong to the interpolation space $\bth{\cH_K}^{s'}$ with $s'\ge3$. This rate is even sharper than the currently known nearly-optimal rate of
$\cO(n^{-\frac{2\alpha s'}{2\alpha s'+1}})\log^2(1/\delta)$~\citep{Li2024-edr-general-domain},
where $n$ is the size of the training data and $\delta \in (0,1)$ is a small probability. This rate is also
sharper than the standard kernel regression rate of $\cO(n^{-\frac{2\alpha}{2\alpha+1}})$ obtained under the regular Neural Tangent Kernel (NTK) regime when training the neural network with the vanilla gradient descent (GD), where $2\alpha = d/(d-1)$. Our analysis is based on two key technical contributions. First, we present a principled decomposition of the network output at each PGD step into a function in the reproducing kernel Hilbert space (RKHS) of a newly induced integral kernel, and a residual function with small $L^{\infty}$-norm. Second, leveraging this decomposition, we apply local Rademacher complexity theory to tightly control the complexity of the function class comprising all the neural network functions obtained in the
PGD iterates. Our results show how PGD changes the effective kernel in the reference training dynamics, inducing an integral kernel with lower complexity than the regular NTK arising from vanilla GD.

%{\color{red}{Section D of this supplementary is the section which contains the simulation of the main paper.}}
\end{abstract}

\vspace{0.3in}
\begin{keywords}
 Nonparametric Regression, Over-Parameterized Neural Network, Preconditioned Gradient Descent, Minimax Optimal Rate
\end{keywords}

\newpage

\section{Introduction}

The remarkable success of deep learning across numerous domains~\citep{YannLecunNature05-DeepLearning} has spurred intense interest in understanding the generalization properties of neural networks within statistical and deep learning theory. A substantial body of work has addressed the optimization dynamics of deep neural networks (DNNs), establishing that gradient descent (GD) and stochastic gradient descent (SGD) can achieve vanishing training loss under various settings~\citep{du2018gradient-gd-dnns,AllenZhuLS19-convergence-dnns,DuLL0Z19-GD-dnns,AroraDHLW19-fine-grained-two-layer,ZouG19,SuY19-convergence-spectral}. Parallel efforts have focused on generalization theory, aiming to provide algorithmic guarantees, that is, bounds for the generalization error of neural networks trained with gradient-based or other optimization methods.
A prominent line of research has shown that in the over-parameterized regime, where hidden layers have a sufficiently large number of neurons, the training dynamics of DNNs closely resemble those of kernel methods with the Neural Tangent Kernel (NTK)~\citep{JacotHG18-NTK}. Specifically, the NTK induced by the architecture of the network characterizes this behavior. The core insight behind NTK-based analysis is that, for highly over-parameterized networks, weights remain close to their random initialization throughout training. This enables first-order Taylor expansion around the initialization point, allowing for tractable approximation and generalization analysis~\citep{CaoG19a-sgd-wide-dnns,AroraDHLW19-fine-grained-two-layer,Ghorbani2021-linearized-two-layer-nn}. In addition, certain infinite-width networks can learn features~\citep{YangH21-feature-learning-infinite-network-width}.

While NTK captures important aspects of learning in wide networks, it operates within a linearized regime of the NTK, limiting its capacity to capture feature learning. Existing works make efforts to learn features beyond NTK's linear approximation. For instance, the QuadNTK~\citep{BaiL20-quadratic-NTK} employs higher-order approximations for improved generalization. Other works, such as~\citet{Nichani0L22-escape-ntk}, combine NTK and QuadNTK to learn structured polynomial targets, while~\citet{DamianLS22-nn-representation-learning} and~\citet{TakakuraS24-mean-field-two-layer} explore optimization procedures and mean-field analyses to capture the feature learning effects in two-layer neural networks.

Although much of the generalization literature focuses on clean data, a fundamental question in statistical learning concerns how neural networks perform in the presence of noise, particularly for nonparametric regression. Prior work has demonstrated that various DNN architectures achieve minimax optimal rates for both smooth~\citep{Yarotsky17-error-approximation-dnns,Bauer2019-regression-dnns-curse-dim,SchmidtHieber2020-regression-relu-dnns,Jiao2023-regression-polynomial-prefactors,ZhangW23-weight-decay-dnns} and non-smooth~\citep{ImaizumiF19-nonsmooth-regression} target functions. However, these results often lack algorithmic guarantees, that is, the network constructions are not generally attainable through standard optimization procedures such as GD. Furthermore, these constructions may assume sparse connectivity structures that are not representative of contemporary over-parameterized architectures. Studying the generalization of neural networks with algorithmic guarantees
is particularly important to understand the behavior of neural networks trained in practice, revealed by the efforts in
this direction such as~\citet{HuWLC21-regularization-minimax-uniform-spherical,
SuhKH22-overparameterized-gd-minimax,Li2024-edr-general-domain,
yang2024gradientdescentfindsoverparameterized,
Yang2025-generalization-two-layer-regression}.
%A notable exception is~\citet{HuWLC21-regularization-minimax-uniform-spherical}, which proves that GD with $\ell^2$ regularization on a two-layer over-parameterized network can achieve the minimax optimal rate for regression under a spherical uniform data distribution.

In this work, we study nonparametric regression using an over-parameterized two-layer neural network with algorithmic guarantees. When the target function resides in an interpolation space characterized by a spectral bias, and the distribution of the covariate admits a polynomial eigenvalue decay rate, we show that
early-stopped training by a novel Preconditioned Gradient Descent (PGD) algorithm achieves the minimax optimal risk rate of $\cO(n^{-\frac{2\alpha(s+2)}{2\alpha(s+2)+1}}) =
\cO(n^{-\frac{d(s+2)}{d(s+2)+d-1}})$ for $2\alpha = d/(d-1)$, where $s \ge 1$ is specified by the preconditioner in PGD. This result improves upon existing bounds, as summarized in Section~\ref{sec:summary-main-results}. 
Every function in the interpolation space exhibits a spectral bias, that is, its energy is predominantly concentrated on the leading eigenfunctions of the NTK of the neural network, with low-degree polynomials as examples of such functions. The interpolation space is widely studied in the statistics literature, and its formal definition is deferred to Section~\ref{sec:setup-kernels-target-function}.
Although GD-based methods with preconditioning have been used in other over-parameterized settings, such as low-rank matrix sensing~\citep{Xu2023-PGD-overparameterized-low-rank-matrix-sensing} and nonconvex factorization~\citep{Zhang2023-PGD-overparameterized-nonconvex-factorization}, this paper is among the first works to design a PGD method achieving sharp minimax optimal rates for nonparametric regression using over-parameterized neural networks. Furthermore, PGD induces a novel kernel, termed the integral kernel, rather than the
well-studied usual NTK induced by the vanilla GD. The integral kernel has lower kernel complexity than the vanilla NTK, and our rate of
$\cO(n^{-\frac{2\alpha(s+2)}{2\alpha(s+2)+1}})$ is in fact the minimax optimal rate of kernel regression with the~integral~kernel over the interpolation space.

At a high level, PGD reshapes the training dynamics so that the neural network behaves like kernel regression with a different, smoother kernel than the regular NTK induced by vanilla GD. In particular, PGD induces an integral kernel $\Kint$ whose eigenvalues decay faster than those of the regular NTK. This stronger spectral decay reduces the associated kernel complexity and enables a sharper regression rate when the target function has spectral bias. Our result provides an algorithmic route to a sharp rate under a structured target class, and it clarifies how an optimization design changes the reference kernel in the linearized training dynamics. The target class we study is exactly the interpolation space $\bth{\cH_K}^{s'}$, which is widely used in statistical learning theory, and in our setting it coincides with the RKHS associated with the integral kernel $\Kint$. To the best of our knowledge, this is among the first results showing that a carefully designed PGD method yields such a sharp rate for nonparametric regression with an over-parameterized neural network under this setting.

% \vspace{0.08in}
% \noindent\textbf{Contributions.}
% Our main contributions are as follows.
% \begin{itemize}[leftmargin=*]
%     \item \textbf{Statistical result.} We prove that PGD with early stopping achieves the sharp rate $\cO(n^{-\frac{2\alpha s'}{2\alpha s'+1}})$ for target functions in the interpolation space $\bth{\cH_K}^{s'}(\mu_0)$ with $s' \ge 3$ under spherical uniform inputs. We show that PGD induces an integral kernel $\Kint$ (rather than the regular NTK induced by vanilla GD), whose faster eigenvalue decay explains the sharper rate under spectral bias.
%     \item \textbf{Technical result.} We develop a decomposition of the PGD iterates into an RKHS($\Kint$) component plus a small residual, and combine it with localized Rademacher complexity to obtain a sharp risk bound.
% \end{itemize}

\vspace{0.1in}
\noindent
We organize this paper as follows. The rest of this section introduces the necessary notation. Section~\ref{sec:setup}
presents the problem setup, including the definition of the interpolation
space. Section~\ref{sec:summary-main-results} summarizes the main results.
Section~\ref{sec:training} introduces GD and PGD for the over-parameterized two-layer neural network,
and Section~\ref{sec:main-results} states the main results in detail.
Section~\ref{sec:proof-roadmap} presents the proof roadmap, the key technical results,
the differences from existing kernel learning theory, and the novel proof strategy.
Section~7 discusses the extension of PGD to general multi-layer neural networks.
Section~\ref{sec:simulation} reports the simulations, and Section~9 concludes the paper.

%\subsection{Notations}
%\label{sec:notations}

\vspace{0.1in}
\noindent \textbf{Notations.} We use bold letters for matrices and vectors, and regular lowercase letters for scalars throughout this paper. The bold letter with a single superscript indicates the corresponding column of a matrix, e.g., $\bA^{(i)}$ is the $i$-th column of matrix $\bA$, and bold letters with subscripts indicate the corresponding rows or elements of a matrix or a vector. We put an arrow on top of
a letter with subscript if it denotes a vector, e.g.,
$\bbx_i$ denotes the $i$-th training
feature. $\fnorm{\cdot}$ and
$\norm{\cdot}{p}$ denote the Frobenius norm and the vector $\ell^{p}$-norm or the matrix $p$-norm. $[m\relcolon n]$ denotes all the integers between $m$ and $n$ inclusively, and $[1\relcolon n]$ is also written as $[n]$. $\Var{\cdot}$ denotes the variance of a random variable. $\bI_n$ is a $n \times n$ identity matrix.  $\indict{E}$ is an indicator function which takes the value of $1$ if event $E$ happens, or $0$ otherwise. The complement of a set $A$ is denoted by $A^c$, and $\abth{A}$ is the cardinality of the set $A$. $\vect{\cdot}$ denotes the vectorization of a matrix or a set of vectors, and $\tr{\cdot}$ is the trace of a matrix.
We denote the unit sphere in $d$-dimensional Euclidean space by $\unitsphere{d-1} \defeq \{\bx \colon  \bx \in \RR^d, \ltwonorm{\bx} =1\}$. Let $L^2(\unitsphere{d-1}, \mu) $ denote the space of square-integrable functions on $\unitsphere{d-1}$ with probability measure $\mu$, and the inner product $\iprod{\cdot}{\cdot}_{L^2}$ and $\norm{\cdot}{L^2}^2$ are defined as $\iprod{f}{g}_{L^2} \coloneqq \int_{\unitsphere{d-1}}f(x)g(x) \diff \mu(x)$ and $\norm{f}{L^2}^2 \defeq \int_{\unitsphere{d-1}}f^2(x) \diff \mu (x) <\infty$. $\ball{}{\bx}{r}$ is the Euclidean closed ball centered at $\bx$ with radius $r$. Given a function $g \colon \unitsphere{d-1} \to \RR$, its $L^{\infty}$-norm is denoted by $\norm{g}{\infty} \defeq \sup_{\bx \in \unitsphere{d-1}} \abth{g(\bx)}$. $L^{\infty}$ is the function class whose elements have almost surely bounded $L^{\infty}$-norm. $\iprod{\cdot}{\cdot}_{\cH}$ and $\norm{\cdot}{\cH}$ denote the inner product and the norm in the Hilbert space $\cH$. $a = \cO(b)$ or $a \lsim b$ indicates that there exists a constant $c>0$ such that $a \le cb$. $\tilde \cO$ indicates there are specific requirements in the constants of the $\cO$ notation. $a = o(b)$ and $a = w(b)$ indicate that $\lim \abth{a/b}  = 0$ and $\lim \abth{a/b}  = \infty$, respectively. $a \asymp b$  or $a = \Theta(b)$ denotes that
there exist constants $c_1,c_2>0$ such that $c_1b \le a \le c_2b$.
Throughout this paper we let the input space $\cX = \unitsphere{d-1}$, and $\Unif{\cX}$ denotes the uniform distribution on $\cX$.
%${\rm supp}(\cdot)$ is the support of a vector, $\mathbb P_{\cS'}$ is the operator of orthogonal projection onto the subspace $\cS'$.
%$\sigma_{t}(\cdot)$ denotes the $t$-th largest singular value of a matrix, $\sigma_{\max}(\cdot)$ and $\sigma_{\min}(\cdot)$ indicate the largest and smallest singular value of a matrix respectively.
%For a function defined on $\cX$, we define $\norm{f}{\cX} \defeq \sup_{\bx \in \cX} \abth{f(\bx)}$.
The constants defined throughout this paper may change from line to line.
For a Reproducing Kernel Hilbert Space (RKHS) $\cH$, $\cH(\mu_0)$ denotes
the ball centered at the origin with radius $\mu_0$ in $\cH$.
We use $\Expect{P}{\cdot}$ to denote the expectation with respect to the distribution $P$.

\section{Problem Setup}
\label{sec:setup}
We first introduce the problem setup for nonparametric regression with a two-layer neural network where the target function belongs to an interpolation space with spectral bias.
\subsection{Two-Layer Neural Network}
\label{sec:setup-two-layer-nn}

We are given the training data $\set{(\bbx_i,  y_i)}_{i=1}^n$ where each data point is a tuple of a feature vector $\bbx_i \in \cX$ and its response $ y_i \in \RR$. 
%Throughout this paper we assume that no two training features coincide, that is, $\bbx_i \neq \bbx_j$ for all $i,j \in [n]$ and $i \neq j$.
We denote the training feature vectors by $\bS = \set{\bbx_i}_{i=1}^n$, and denote by $P_n$ the empirical distribution over $\bS$. All the responses are stacked as a vector $ \by = [ y_1, \ldots,  y_n]^\top \in \RR^n$.
The response $ y_i$ is given by $ y_i= f^*(\bbx_i) + w_i$ for
$i \in [n]$,
where $\set{w_i}_{i=1}^n$ are i.i.d. sub-Gaussian random variables representing the noise with mean $0$ and variance proxy $\sigma_0^2>0$, that is, $\Expect{}{\exp(\lambda w_i)} \le \exp(\lambda^2 \sigma_0^2/2)$ for any $\lambda \in \RR$. $f^*$ is the target function to be detailed later. We define $\bw \defeq \bth{w_1,\ldots,w_n}^{\top}$, and use $f^*(\bS) \defeq \bth{f^*(\bbx_1),\ldots,f^*(\bbx_n)}^{\top}$ to denote the clean target responses. The noise is independent of the covariates and the algorithmic randomness.
The feature vectors in $\bS$ are drawn i.i.d. according to a distribution $P$ with probability measure $\mu$. We assume $\supp\mu=\cX$, so the continuous Mercer eigenfunctions and RKHS functions are identified pointwise on $\cX$. This population support condition permits atoms and coincident training features.

We consider a two-layer Neural Network (NN) with the ReLU activation $\sigma\colon\RR\to\RR$, $\sigma(z)=\max\{0,z\}$. Its
mapping function is

\vspace{-0.2in}
\bal\label{eq:two-layer-nn}
&f(\bW,\bx) =  \frac{1}{\sqrt{m}}\sum_{r=1}^{m}
a_r \relu{{\bbw_r}^\top \bx},
\eal\vspace{-0.15in}

\noindent where $\bx \in \cX$ is the input, $m$ is the number of neurons, $\bW = \set{\bbw_r}_{r=1}^m$ with $\bbw_r \in \RR^d$ for $r \in [m]$ denotes the weight vectors in the first layer. $\ba = \bth{a_1, \ldots, a_m} \in \RR^m$ denotes the weights of the second layer. Our goal is to
establish a sharp convergence rate of $\Expect{P}{\pth{f(\bW,\cdot) - f^*}^2}$
with $f(\bW,\cdot)$ trained from the training data. We also write $\bW$ as $\bW_{\bS}$ to indicate that the weight vectors, or the weights $\bW$, are trained on the training features $\bS$.

\subsection{Kernel and Target Function}
\label{sec:setup-kernels-target-function}

The statistical rate we can achieve depends on two ingredients: the spectrum of the kernel induced by training and the regularity (or spectral bias) of the target function. We first introduce the standard kernel $K$ associated with the two-layer network under vanilla GD, and then define the integral kernel $\Kint$ that will arise from PGD. The key difference is spectral: $\Kint$ shares the same eigenfunctions as $K$ but has faster-decaying eigenvalues, which leads to lower kernel complexity under our target-class assumption.

We define the kernel function
\bal\label{eq:kernel-two-layer}
&K(\bu,\bv) \defeq  \frac{ \iprod{\bu}{\bv}}{2\pi} \pth{\pi -\arccos \iprod{\bu}{\bv}}, ~~~ \forall ~ \bu, \bv \in \cX,
\eal%
which is in fact the NTK associated with
the two-layer NN (\ref{eq:two-layer-nn})~\citep{du2018gradient-gd-dnns} when trained with the vanilla GD with fixed $\ba$.
Let the gram matrix of $K$ over the training features $\bS$ be $\bK \in \RR^{n \times n}, \bK_{ij} = K(\bbx_i,\bbx_j)$ for $i,j \in [n]$,
and $\bK_n \defeq \bK/n$. Let the eigendecomposition  of $\bK_n$ be $\bK_n = \bU \bSigma {\bU}^{\top}$ where $\bU$ is a $n \times n$ orthogonal matrix, and $\bSigma$ is a diagonal matrix with its diagonal elements $\set{\hat \lambda_i}_{i=1}^n$ being the eigenvalues of $\bK_n$ and sorted in a non-increasing order.  
%It is proved in existing works, such as~\citet{du2018gradient-gd-dnns}, that $\bK_n$ is non-singular.
It can be verified that $\hat \lambda_1 \in (0, 1/2]$.
%$\bU_{k}$ is the $k$-th eigenvector which corresponds to the $k$-th largest %eigenvalue $\hat \lambda_k$ for $k \in [n]$.
%We will also use the notation  $\bK_{\bS,n}$ to denote $\bK_n$ explicitly indicating the training features $\bS$ for which the  matrix $\bK_n$ is computed.We define the function class $L^2_{k_0}$ with $k_0 \ge 1$ by
%$L^2_{k_0} \defeq \set{f \in L^2(\cX,\mu)
%\colon f = \sum\limits_{j =1}^{m_{k_0}} \alpha_j e_j  }$,
%which comprises all the square integrable functions which can be
%spanned by eigenfunctions in the first $k_0$ eigenspaces of
%$T_K$. It is clear that $L^2_{k_0} \subseteq \cH_K$.
%We consider the case that $f^* \in \cH_{K}(\mu_0) \cap L^2_{k_0}$.
Let $\cH_{K}$ be the RKHS associated with
$K$. Because $K$ is continuous on the compact set $\cX \times \cX$, the integral operator $T_K \colon L^2(\cX,\mu) \to L^2(\cX,\mu), \pth{T_K f}(\bx) \defeq \int_{\cX} K(\bx,\bx') f(\bx') \diff \mu(\bx')$ is a positive, self-adjoint, and compact operator on $\overline{\range(T_K)} \subseteq L^2(\cX,\mu)$. By the spectral theorem, there is a countable orthonormal basis $\set{e_j}_{j \ge 1}$ of $\overline{\range(T_K)}=(\ker T_K)^\perp \subseteq L^2(\cX,\mu)$ and a sequence of positive eigenvalues $\set{\lambda_j}_{j \ge 1}$ with $\frac 12 \ge \lambda_1 \ge \lambda_2 \ge \ldots > 0$ such that $e_j$ is the eigenfunction of $T_K$ with
$\lambda_j$ being the corresponding eigenvalue. That is, $T_K e_j = \lambda_j e_j, j \ge 1$.
Let $\set{\mu_{\ell}}_{\ell \ge 1}$ be the distinct eigenvalues associated with $T_K$, and let
$m_{\ell}$ be the sum of the multiplicities of the eigenvalues $\set{\mu_{\ell'}}_{\ell'=1}^{\ell}$, so that
$m_{\ell} - m_{\ell-1}$ is the multiplicity of $\mu_{\ell}$ with $m_0 = 0$.
It is well known that $\set{v_j = \sqrt {\lambda_j} e_j}_{j\ \ge 1}$ is an orthonormal basis of $\cH_K$. For a positive constant $\mu_0$, we define $\cH_{K}(\mu_0) \defeq \set{f \in \cH_{K} \colon \norm{f}{\cH} \le \mu_0}$ as the closed  ball in $\cH_K$ centered at the origin with radius $\mu_0$. We note that $\cH_{K}(\mu_0)$ is also specified by $\cH_{K}(\mu_0) = \set{f \in L^2(\cX,\mu)\colon f = \sum_{j =1}^{\infty} \beta_j e_j,
\sum_{j = 1}^{\infty} \beta_j^2/\lambda_j \le \mu_0^2}$.

\vspace{0.15in}
\noindent  \textbf{Target Function in an Interpolation Space with Spectral Bias.} Extensive theoretical and empirical studies find that it is easy for neural networks to learn spectrally biased target functions or low-frequency information in the training data
\citep{rahaman19a-spectral-bias,AroraDHLW19-fine-grained-two-layer,CaoFWZG21-spectral-bias,
ChorariaD0MC22-spectral-bias-pnns}. For example, the studies in
\citet{AroraDHLW19-fine-grained-two-layer,
CaoFWZG21-spectral-bias} reveal that
it is easier for over-parameterized neural networks to learn target functions with spectral bias, for example, polynomials of low-degree with
spherical uniform data distribution on $\cX$, or the low-rank part of the ground truth training class labels, or simple patterns of low-frequency.
This observation motivates
us to restrict the target function $f^*$ to a smaller class than $\cH_K(\mu_0)$, which is $\cH_{\Kint}(\mu_0)$ to be studied in this paper.

We then define the integral kernel $\Kint$, and explain why functions in $\cH_{\Kint}(\mu_0)$
have stronger spectral bias than functions in $\cH_{K}(\mu_0)$.
Given the kernel $K$,  we first define another kernel $K^{(s)}$ with $s \ge 1$ by
\bal\label{eq:Ks-def}
K^{(s)}(\bx,\bx') \defeq \sum\limits_{j \ge 1} \lambda_j^{s} e_j(\bx)e_j(\bx'),
\quad \forall \bx,\bx' \in \cX.
\eal\vspace{-0.15in}

\noindent It can be verified that $K^{(s)}$ is positive-definite (PD), and it is shown in Theorem~\ref{theorem:Kint-bounded} deferred to Section~\ref{sec:RKHS-Kint-more-results} of the appendix that $K^{(s)}$ is well-defined since the RHS of (\ref{eq:Ks-def}) converges uniformly on $\cX \times \cX$ with $s \ge 1$, and $K^{(s)}$ is bounded with $\sup_{\bx,\bx' \in \cX} \abth{K^{(s)}(\bx,\bx')} \le 1/2$. We define $\bK^{(s)} \in \RR^{N \times N}$ as the kernel gram matrix of $K^{(s)}$ specified by $\bK^{(s)}_{ij} = K^{(s)}(\bbq_i,\bbq_j)$ for all $i,j \in [N]$, and
$\bK^{(s)}_N = \bK^{(s)}/N$.
We then define the integral kernel $\Kint $ as well as its empirical version of $\hKint$ as
\bal
\Kint (\bx,\bx') &\defeq  \int_{\cX \times \cX} K(\bx,\bv) K^{(s)}(\bv,\bv')  K(\bv',\bx')  \diff \mu(\bv) \otimes \mu(\bv'), \label{eq:Kint-def} \\
\hKint (\bx,\bx') &\defeq \frac{1}{N^2} \sum\limits_{i,j=1}^N K(\bx,\bbq_i)  K^{(s)}(\bbq_i,\bbq_j) K(\bbq_j,\bx'). \label{eq:hatKint-def}
\eal
An i.i.d. sample $\bQ = \set{\bbq_i}_{i=1}^N$ such that
$\bbq_i \sim P$ for all $i \in [N]$ is used in (\ref{eq:hatKint-def}). The auxiliary sample is independent of the training sample, the noise, and the initialization. For the spherical uniform design, $P=\Unif{\cX}$.

\vspace{0.08in}
\noindent\textbf{Intuition for the integral kernel $\Kint$.} 
The integral kernel $\Kint$ can be viewed as a spectrally reweighted version of the regular NTK $K$. It preserves the eigenfunctions of $K$ but transforms the eigenvalues from $\lambda_j$ to $\lambda_j^{s+2}$ (Theorem~\ref{theorem:spectrum-Kint}). Since $\lambda_j \in (0,1)$, this transformation suppresses higher-frequency directions more strongly, yielding a smoother kernel and a smaller associated kernel complexity. This is the spectral reason PGD can achieve a sharper rate when the target function has spectral bias.

In particular, Theorem~\ref{theorem:spectrum-Kint} in Appendix~\ref{sec:RKHS-Kint-more-results}  shows that the integral operator associated with $\Kint$, $T_{\Kint}$, has the same eigenfunctions $\set{e_j}_{j \ge 1}$  as $T_K$, and the eigenvalue corresponding to $e_j$ is $\lambdaint_j = \lambda^{s+2}_j \in (0,1/8]$ for all $j \ge 1$ and $s \ge 1$. Because $\Kint$ is still a PD kernel, the
RKHS associated with $\Kint$ is well-defined, and $\cH_{\Kint}(\mu_0)$ indicates a subset of $\cH_{\Kint}$ with the RKHS-norm $\norm{\cdot}{\Kint}$ bounded by $\mu_0$. It can
be verified that $\cH_{\Kint}(\mu_0) = \{f \in L^2(\cX,\mu) \colon f = \sum_{j =1}^{\infty} \beta_j e_j,\sum_{j =1}^{\infty} \beta_j^2/\lambda^{s+2}_j \le \mu_0^2\}$.
It is remarked that $\cH_{\Kint}(\mu_0)$ is also the interpolation
space $\bth{\cH_K}^{s'}$ (with $s' \defeq s+2$) of the RKHS, which characterizes the regularity of the regression function studied
in~\citet{Steinwart2012,Fischer2020-sobolev-norm-regularized-least-square}. The interpolation space for general $s' > 0$ is defined as $\bth{\cH_K}^{s'}(\mu_0) \defeq
\set{\sum_{j \ge 1} a_j \lambda_j^{s'/2} e_j \colon \sum_{j \ge 1} a^2_j \le \mu_0^2}$.
%In this work we consider $\bth{\cH_K}^{s'}$ with $s' \ge 3$, and it can %be verified that $\cH_{\Kint}(\mu_0) = \bth{\cH_K}^{s'}(\mu_0)$
%with $s' = s+2$ and $s \ge 1$.
The RKHS associated with $\Kint$ is exactly the interpolation space of $\cH_K$ with exponent $s' = s+2$, that is,
$\cH_{\Kint}(\mu_0)=\bth{\cH_K}^{s'}(\mu_0)$.
As $\lambda_j \to 0$ with $j \to \infty$ and $\lambda_j^{s+2} < \lambda_j < 1$, we have $\cH_{\Kint}(\mu_0) \subseteq \cH_K(\mu_0)$. Compared to a function in $\cH_K(\mu_0)$, a function in $\bth{\cH_K}^{s'}(\mu_0)$ has the
expansion coefficients
$\set{\beta_j}_{j \ge 1}$ which concentrate more on the leading eigenfunctions with smaller index $j$. In this sense, we say that the target function $f^* \in \cH_{\Kint}(\mu_0) =\bth{\cH_K}^{s'}(\mu_0) $ has stronger spectral bias than a function in $\cH_{K}(\mu_0)$. Intuitively, this assumption means the expansion coefficients of $f^*$ in the eigenbasis of $K$ decay faster, concentrating more energy on low-frequency or leading eigenfunctions.
 In summary, \textit{if $f^* \in \bth{\cH_K}^{s'}(\mu_0) \subseteq \cH_K(\mu_0)$, then we say that
$f^*$ is a target function in an interpolation space with spectral bias}. The motivation for introducing the interpolation space in \citet{Steinwart2012,Fischer2020-sobolev-norm-regularized-least-square} is to quantify how strongly a regression function  aligns with leading kernel eigenfunctions, enabling sharper statistical learning rates and connecting RKHS learning theory with classical smoothness spaces such as Sobolev spaces.
We note that such target functions with spectral bias have been widely studied in the deep learning literature, and a representative example is given below.

\vspace{0.1in}

\noindent \textbf{An Example of Functions with Spectral Bias.} Under the spherical uniform design, a spherical polynomial with no component in $\ker T_K$ has a finite expansion in positive-eigenvalue spherical harmonics.
\citet{Ghorbani2021-linearized-two-layer-nn,BaiL20-quadratic-NTK,
CaoFWZG21-spectral-bias} study learning low-degree polynomials by linearization or higher-order approximation of neural networks.
For a polynomial in the positive spectral subspace, there is a finite index $J$ such that
\bals
f = \sum\limits_{j=1}^{J} \beta_j e_j.
\eals
When $\sum_{j=1}^{J} \beta^2_j/\lambda_j^{s'} \le \mu_0^2$ ,
we have $f \in \bth{\cH_K}^{s'}(\mu_0)$.

\vspace{0.15in}
\noindent \textbf{The Task of Nonparametric Regression.} Our goal is to construct an estimator $\hat f$ from the training data $\set{(\bbx_i,y_i)}_{i=1}^n$ such that the risk $\Expect{P}{(\hat f-f^*)^2}$ converges to $0$ at a fast rate, under the target-class assumption $f^* \in \bth{\cH_K}^{s'}(\mu_0)$. In this paper, the estimator $\hat f$ is the two-layer neural network~\eqref{eq:two-layer-nn} trained by PGD with early stopping.
% \begin{figure}[!htbp]
%     \centering
%     \begin{minipage}{.42\columnwidth}
%         \begin{algorithm}[H]
%         \renewcommand{\algorithmicrequire}{\textbf{Input:}}
% \renewcommand\algorithmicensure {\textbf{Output:} }
% \caption{Training the Two-Layer NN by GD}
% \label{alg:PGD}
% \begin{algorithmic}[1]
% \State $\bW(T) \leftarrow$ Training-by-GD($T,\bW(0)$)
% \State \textbf{\bf input: } $T,\bW(0)$
% \State \textbf{\bf for } $t=1,\ldots,T$ \,\,\textbf{\bf do }
% \State \quad Perform the $t$-th step of GD by
% (\ref{eq:GD-two-layer-nn})
% \State \textbf{\bf end for }
% \State \textbf{\bf return} $\bW(T)$
% \end{algorithmic}
% \end{algorithm}
% \end{minipage}%
%     \hspace{2mm}
%     \begin{minipage}{.55\columnwidth}
% \begin{algorithm}[H]
% \renewcommand{\algorithmicrequire}{\textbf{Input:}}
% \renewcommand\algorithmicensure {\textbf{Output:} }
% \caption{Training the Two-Layer NN by PGD}
% \label{alg:PGD}
% \begin{algorithmic}[1]
% \State $\bW(T) \leftarrow$ Training-by-PGD($T,\bW(0),\bM$)
% \State \textbf{\bf input: } $T,\bW(0),\bM$
% \State \textbf{\bf for } $t=1,\ldots,T$ \,\,\textbf{\bf do }
% \State \quad Perform the $t$-th step of PGD by
% (\ref{eq:PGD-two-layer-nn})
% \State \textbf{\bf end for }
% \State \textbf{\bf return} $\bW(T)$
% \end{algorithmic}
% \end{algorithm}
%   \end{minipage}%
% \end{figure}

\section{Summary of Main Results}
\label{sec:summary-main-results}

\begin{algorithm}[h]
\renewcommand{\algorithmicrequire}{\textbf{Input:}}
\renewcommand\algorithmicensure {\textbf{Output:} }
\caption{Training the Two-Layer NN by PGD}
\label{alg:PGD}
\begin{algorithmic}[1]
\State $\bW(T) \leftarrow$ Training-by-PGD($T,\bW(0),\bM$)
\State \textbf{\bf input: } $T,\bW(0),\bM$
%\State \textbf{\bf for } $t=1,\ldots,T$ \,\,\textbf{\bf do }
\For{$t=1,\ldots,T$}
\State Perform the $t$-th step of PGD by
(\ref{eq:PGD-two-layer-nn})
%\State \textbf{\bf end for }
\EndFor
\State \textbf{\bf return} $\bW(T)$
\end{algorithmic}
\end{algorithm}
%\begin{table*}[t!]
\begin{table*}[!htbp]
\centering
\caption{Comparisons with existing results on risk rates and assumptions for nonparametric regression by training over-parameterized neural networks with algorithmic guarantees. The target-function assumptions are stated separately because the rates are class-dependent. The responses $\set{y_i}_{i=1}^n$ are corrupted by i.i.d. Gaussian or sub-Gaussian noise. Here $P$ is the distribution of the training features and, for the RKHS-based results, $\tK$ is the kernel induced by the neural architecture and optimization method in the corresponding work. It is the regular NTK for all such works except ours, for which $\tK=\Kint$.
We use $s'=s+2$ with $s\ge1$ and $\alpha=d/(2(d-1))$. The notation for the target class in~\citet[Proposition~13]{Li2024-edr-general-domain} is adapted to ours. Its displayed rate is expressed under matching source and spectral parameters. The original result uses a different, biased NTK with EDR $\lambda_j\asymp j^{-(d+1)/d}$. The row substitutes our EDR $\lambda_j\asymp j^{-d/(d-1)}$ for this exponent comparison and does not assert that the original theorem applies unchanged to our architecture. The rates of that result and ours hold with probability at least $1-\Theta(1/n)$.}
\resizebox{1\linewidth}{!}{
\begin{tabular}{|l|c|c|c|c|}
\hline
\textbf{Existing Works and Our Result}
& \textbf{Distributional Assumptions}
& \textbf{Target-Function Assumption}
& \textbf{Eigenvalue Decay Rate (EDR)}
& \textbf{Rate of Nonparametric Regression Risk}
\\ \hline
\begin{tabular}{@{}c@{}}
\citep[Theorem~5.2]{HuWLC21-regularization-minimax-uniform-spherical},\\
\citep[Theorem~3.11]{SuhKH22-overparameterized-gd-minimax}
\end{tabular}
& $P$ is $\Unif{\cX}$
& $f^*\in\cH_{\tK}(\mu_0)$
& $\lambda_j\asymp j^{-d/(d-1)}$
& $\cO(n^{-d/(2d-1)})$
\\ \hline
\citep[Proposition~13]{Li2024-edr-general-domain}
& $P$ is sub-Gaussian
& \begin{tabular}{@{}c@{}}
$f^*\in\bth{\cH_K}^{s'}(\mu_0)\cap L^\infty$\\
for $s'>1/(d+1)$
\end{tabular}
& $\lambda_j\asymp j^{-d/(d-1)}$
& $\cO(n^{-d s'/(d s'+d-1)})\log^2(n)$
\\ \hline
\citep[Corollary~5.2]{Yang2025-generalization-two-layer-regression}
& \begin{tabular}{@{}c@{}}
$P$ supported on $\cX$ admits the polynomial\\
EDR $\lambda_j\asymp j^{-d/(d-1)}$
\end{tabular}
& $f^*\in\cH_{\tK}(\mu_0)$
& $\lambda_j\asymp j^{-d/(d-1)}$
& $\cO(n^{-d/(2d-1)})$
\\ \hline
Our Result (Theorem~\ref{theorem:minimax-nonparametric-regression-Kint})
& \cellcolor{blue!15}
\begin{tabular}{@{}c@{}}
$P$ has full support on $\cX$ and admits the polynomial\\
EDR $\lambda_j\asymp j^{-d/(d-1)}$
\end{tabular}
& \cellcolor{blue!15}
\begin{tabular}{@{}c@{}}
$f^*\in\cH_{\Kint}(\mu_0)=\bth{\cH_K}^{s'}(\mu_0)$\\
for $s'=s+2\ge3$
\end{tabular}
& \cellcolor{blue!15}
$\lambdaint_j\asymp j^{-d(s+2)/(d-1)}$
& \cellcolor{blue!15}
$\cO(n^{-2\alpha s'/(2\alpha s'+1)})
=\cO(n^{-d s'/(d s'+d-1)})$
\\ \hline
Our Result (Corollary~\ref{corollary:minimax-nonparametric-regression-Kint-general-distribution})
& \cellcolor{blue!15}
\begin{tabular}{@{}c@{}}
$p(\bx)\lsim(1+\ltwonorm{\bx}^2)^{-(d+1)/2}$\\
for every $\bx\in\RR^d$
\end{tabular}
& \cellcolor{blue!15}
\begin{tabular}{@{}c@{}}
$f^*\in\cH_{\Kint}(\mu_0)=\bth{\cH_K}^{s'}(\mu_0)$\\
for $s'=s+2\ge3$
\end{tabular}
& \cellcolor{blue!15}
$\lambdaint_j\asymp j^{-d(s+2)/(d-1)}$
& \cellcolor{blue!15}
$\cO(n^{-2\alpha s'/(2\alpha s'+1)})
=\cO(n^{-d s'/(d s'+d-1)})$
\\ \hline
\end{tabular}
}
\label{table:main-results-comparison}
\end{table*}

This section summarizes our main result and places it in the context of prior work on nonparametric regression with over-parameterized neural networks and algorithmic guarantees. The key result is that, under a covariate distribution that admits a certain polynomial eigenvalue decay rate (EDR) and an interpolation-space target assumption $f^* \in \bth{\cH_K}^{s'}(\mu_0)$ with $s' \ge 3$, PGD with early stopping achieves the sharp rate $\cO(n^{-\frac{2\alpha s'}{2\alpha s'+1}})$ by inducing an integral kernel $\Kint$ with faster eigenvalue decay than the regular NTK.

It is remarked that all the results and discussions in this paper are for the setting with fixed dimension $d$, which is a widely adopted setting in the existing works \citep{HuWLC21-regularization-minimax-uniform-spherical,
SuhKH22-overparameterized-gd-minimax,Li2024-edr-general-domain,yang2024gradientdescentfindsoverparameterized,Yang2025-generalization-two-layer-regression}, and we consider $d \ge 5$ in this paper.
The main result of this paper is presented in
Theorem~\ref{theorem:minimax-nonparametric-regression-Kint} of
Section~\ref{sec:nonparametric-regression-Kint}, which shows that training the two-layer NN~\eqref{eq:two-layer-nn} by a new Preconditioned Gradient Descent (PGD) method with early stopping,
described by Algorithm~\ref{alg:PGD}, yields a sharper rate for nonparametric regression risk compared to existing results, under the assumption that the target function lies in the interpolation
spaces $\bth{\cH_K}^{s'}$ widely analyzed in the statistical learning literature
~\citep{Steinwart2012,Fischer2020-sobolev-norm-regularized-least-square}. The comparison with relevant existing  works is summarized in Table~\ref{table:main-results-comparison}. Prior works such as~\citet{RaskuttiWY14-early-stopping-kernel-regression} demonstrate that kernel regression with the kernel $\tK$, which is the NTK of the neural network considered in each specific work in Table~\ref{table:main-results-comparison} except for this paper, achieves the minimax optimal risk rate of $\cO(\eps_{\tK,n}^2)$, where $\eps_{\tK,n}^2$ denotes the kernel complexity of $\tK$ formalized in Section~\ref{sec:kernel-complexity}. When the eigenvalues of $\tK$ exhibit a polynomial EDR $\lambda_j \asymp j^{-2\alpha}$ for $\alpha > 1/2$, it is well known that $\eps_{\tK,n}^2 \asymp n^{-\frac{2\alpha}{2\alpha+1}}$~\citep[Corollary 3]{RaskuttiWY14-early-stopping-kernel-regression}. Moreover, for spherical uniform input distributions,~\citet[Lemma 3.1]{HuWLC21-regularization-minimax-uniform-spherical} shows that $2\alpha = d/(d-1)$, leading to the minimax rate $\cO(n^{-\frac{d}{2d-1}})$.
Consequently, nearly all existing works in Table~\ref{table:main-results-comparison} excluding~\citep{Li2024-edr-general-domain} report the same minimax rate of $\cO(n^{-\frac{d}{2d-1}})$ under the spherical uniformly
distributed training features when the target function lies in
a bounded ball of $\cH_{\tK}$. Such a rate arises because the EDR of the kernel $\tK$ satisfies
$\Theta(j^{-2\alpha}) = \Theta(j^{-\frac{d}{d-1}})$, and the vanilla GD is used during training. \citet{Li2024-edr-general-domain} improves upon the prior results by obtaining a nearly optimal rate $\cO(n^{-\frac{2\alpha s'}{2\alpha s'+1}})\log^2(1/\delta)$ when the target function belongs to the interpolation space $\bth{\cH_K}^{s'}(R)$ for $s' > 1/(d+1)$ and $R > 0$, which holds with probability at least $1-\delta$. However,
this rate is less sharp than the rate derived in this paper, $\cO(n^{-\frac{2\alpha s'}{2\alpha s'+1}})=\cO(n^{-\frac{d s'}{d s' +d-1}}) $, for all $s' \ge 3$. In particular, when our PGD is employed to train the two-layer NN~\eqref{eq:two-layer-nn}, the induced kernel is the integral kernel $\Kint$ instead of the
regular NTK $K$ (\ref{eq:kernel-two-layer}). Under certain mild conditions on the covariate distribution guaranteeing the polynomial EDR and the source condition that the target function $f^* \in \bth{\cH_K}^{s'}(\mu_0)$, the rate derived
in this paper is sharper than the prior risk rates,
including those in~\citet{HuWLC21-regularization-minimax-uniform-spherical,
SuhKH22-overparameterized-gd-minimax,Li2024-edr-general-domain,yang2024gradientdescentfindsoverparameterized,Yang2025-generalization-two-layer-regression}. %We note that the rate of $\cO(n^{-\frac{2\alpha s'}{2\alpha s'+1}})\log^2(1/\delta)$ in \citet[Proposition 13]{Li2024-edr-general-domain}, while originally developed under the sub-Gaussian distribution of the training features, still holds under the spherical uniform distribution of the training features considered in this paper.

While the rate of $\cO(n^{-\frac{d}{2d-1}})$ in~\citet{HuWLC21-regularization-minimax-uniform-spherical,SuhKH22-overparameterized-gd-minimax,
yang2024gradientdescentfindsoverparameterized}  remains minimax optimal in the context of kernel regression with the regular NTK $\tK$ when $f^* \in \cH_{\tK}(\mu_0)$, a faster rate is achievable if the target function lies in a stricter space. In particular, if $f^* \in \cH_{\Kint}(\mu_0) = \bth{\cH_K}^{s'}(\mu_0) \subseteq \cH_K(\mu_0)$, then kernel regression using the integral kernel $\Kint$ attains the sharper rate $\cO(n^{-\frac{d s'}{d s' +d-1}})$, which, as  in Theorem~\ref{theorem:minimax-nonparametric-regression-Kint}, is minimax optimal in the sense of  regression by the kernel $\Kint$ over the space $\cH_{\Kint}(\mu_0)$~\citep{Stone1985,Yang1999-minimax-rates-convergence,Yuan2016-minimax-additive-models}.
In fact, Theorem~\ref{theorem:spectrum-Kint} shows that $\Kint$ has the EDR of $\lambdaint_j = \lambda_j^{s'} \asymp j^{-2\alpha s'}$ with $s' = s+2$ for all $s \ge 1$ and $j \ge 1$. Accordingly, the associated fixed point of the kernel complexity function associated with the integral kernel $\Kint$ is $\cO(n^{-\frac{d s'}{d s' +d-1}})$, yielding the claimed sharper risk rate of $\cO(n^{-\frac{d s'}{d s' +d-1}})$.

Our findings show that PGD changes the effective kernel in the reference training dynamics from the regular NTK $K$ to $\Kint$. Under the common source and spectral conditions, the statistical improvement over the displayed vanilla-GD bound in \citet[Proposition~13]{Li2024-edr-general-domain} is removal of the logarithmic factor. The polynomial exponent is unchanged. The faster exponent relative to ordinary-RKHS results reflects the smaller target class. This comparison does not contradict a minimax lower bound or establish that logarithmic losses are unavoidable for regular-NTK methods. See Section~\ref{sec:difference-existing-kernel-theory}.

%Second, in Theorem~\ref{theorem:LRC-population-NN-K-fixed-point} in Section~\ref{sec:nonparametric-regression-K}, we prove that the neural network
%(\ref{eq:two-layer-nn}) trained by GD with early stopping using Algorithm~\ref{alg:PGD} achieves a sharp rate of the
% nonparametric regression risk \textit{without distributional assumption} for general target function. The neural network width and the training time are quantized by the critical population rate of the NTK (\ref{eq:kernel-two-layer}), where  is widely studied in the kernel learning literature~\citep{Yang2017-fast-sketch-kernel,RaskuttiWY14-early-stopping-kernel-regression}. Our result is more general than the current result~\citep{HuWLC21-regularization-minimax-uniform-spherical} which requires uniform data distribution on the unit sphere. Under the uniform data distribution on unit sphere, Theorem~\ref{theorem:minimax-nonparametric-regression-K} shows the same minimax optimal rate as~\citep{HuWLC21-regularization-minimax-uniform-spherical} with the neural network
% width explicitly quantized in terms of $n,d$ where $n$ is the training data size and $d$ is the data dimension.

\vspace{-0.1in}
\section{Training by Gradient Descent and Preconditioned Gradient Descent}
\label{sec:training}

%\begin{algorithm}[!hbt]
%\renewcommand{\algorithmicrequire}{\textbf{Input:}}
%\renewcommand\algorithmicensure {\textbf{Output:} }
%\caption{Algorithm that generates the Graph Adjacency Matrix $\bK_{\bQ}
% \in \RR^{N \times n}$}
%\label{alg:preconditioner}
%\begin{algorithmic}[1]
%\State $\bM \in \RR^{m \times m} \leftarrow$ Generate-Precondition-Matrix($N$)
%\State \textbf{\bf input: } $N,\bW(0)$
%\State Generate an i.i.d. sample $\bQ = \set{\bbq_i}_{i=1}^N$ such that
%$\bbq_i \sim \Unif{\cX}$ for $i \in [N]$,
%
%\State Compute the $\bth{\bK_{\bQ}}_{ij}
%= \frac{1}{\sqrt {nN}} K(\bbq_i, \bbx_j)$.
%\State \textbf{\bf return} $\bK_{\bQ}$
%\end{algorithmic}
%\end{algorithm}
%

We next describe the training procedures used in this paper and highlight the key difference between vanilla GD and the proposed PGD. Both methods optimize the same empirical quadratic loss on the training data. The distinction lies in the preconditioner  introduced by PGD which changes the training dynamics of the vanilla GD so as to induce the integral kernel.

\begin{algorithm}[t!]
\renewcommand{\algorithmicrequire}{\textbf{Input:}}
\renewcommand\algorithmicensure {\textbf{Output:} }
\caption{Algorithm that generates the preconditioner $\bM \in \RR^{md \times md}$}
\label{alg:preconditioner}
\begin{algorithmic}[1]
\State $\bM \in \RR^{md \times md} \leftarrow$ Generate-Precondition-Matrix($N$)
\State \textbf{\bf input: } $N,\bW(0)$
\State Generate an i.i.d. sample $\bQ = \set{\bbq_i}_{i=1}^N$ such that
$\bbq_i \sim P$ for all $i \in [N]$.
\State Compute the matrix $\bZ_{\bQ}(0) \in \RR^{md \times N}$ according to (\ref{eq:GD-Z-two-layer-nn}) with
$\bS$ replaced by $\bQ$, using the randomly initialized weight vectors $\bW(0) = \set{\bbw_{r}(0)}_{r=1}^m$.
\State Compute the preconditioner $\bM = \frac{1}{N}\bZ_{\bQ}(0) \bK^{(s)}_N \bZ_{\bQ}(0)^{\top} \in \RR^{md \times md}$.
\State \textbf{\bf return} $\bM$
\end{algorithmic}
\end{algorithm}

In the training process, only $\bW$ is optimized with $\ba$ randomly initialized to $\pm 1$ and then fixed. The following quadratic loss function is minimized during the training process:

\vspace{-0.22in}
\bal \label{eq:obj-dnns}
%&L(\bW) \defeq \frac{1}{2n {\sqrt N}}
%\ltwonorm{\bK_{N,n} \pth{f(\bW,\bS)-\by}}^2.
L(\bW) \defeq \frac{1}{2n} \sum_{i=1}^{n} \pth{f(\bW,\bbx_i) - y_i}^2.
\eal\vspace{-0.22in}

\noindent
In the $(t+1)$-th step of GD with $t \ge 0$, the weights of the neural network, $\bW_{\bS}$, are updated by one step of GD through

\vspace{-0.23in}
\bal\label{eq:GD-two-layer-nn}
\vect{\bW_{\bS}(t+1)} - \vect{\bW_{\bS}(t)} = - \frac{\eta}{n} \bZ_{\bS}(t)
(\hat \by(t) -  \by).
\eal
\vspace{-0.22in}

When the target function lies in the interpolation space $\bth{\cH_K}^{s'}(\mu_0)$ and the covariate distribution on $\cX$ satisfies the spectral assumptions of Section~\ref{sec:main-results}, we propose to use a novel Preconditioned Gradient Descent (PGD) instead of GD to train the network,  which achieves a provably sharper generalization bound than the regular
GD detailed in Section~\ref{sec:nonparametric-regression-Kint}. Let the preconditioner $\bM \in \RR^{md \times md}$ be generated by
Algorithm~\ref{alg:preconditioner}. $\bM$ is computed by the sample $\bQ$ of $N$ i.i.d. random variables
with distribution $P$ on $\cX$ used in the definition of $\hKint$ in (\ref{eq:hatKint-def}), and $\bQ$ is independent of $\bW(0)$. In the $(t+1)$-th step of PGD with $t \ge 0$, the weights of the neural network, $\bW_{\bS}$, are updated by one step of PGD,

\vspace{-0.23in}
\bal\label{eq:PGD-two-layer-nn}
\vect{\bW_{\bS}(t+1)} - \vect{\bW_{\bS}(t)} = - \frac{\eta}{n} \bM \bZ_{\bS}(t)
(\hat \by(t) -  \by),
\eal
%\vspace{-0.01in}
\noindent where $\by_i = y_i$, $\hat \by(t) \in \RR^n$ with $\bth{\hat \by(t)}_i = f(\bW(t),\bbx_i)$.
Compared to the vanilla GD update~\eqref{eq:GD-two-layer-nn}, the PGD update~\eqref{eq:PGD-two-layer-nn} inserts the preconditioner $\bM$, which reweights the update directions and thereby alters the effective kernel dynamics induced during training.
We also denote $f(\bW(t),\cdot)$ by
$f_t(\cdot)$~as the neural network function with weight vectors $\bW(t)$ obtained after the $t$-th step of PGD.
$\bZ_{\bS}(t) \in \RR^{md \times n}$ is defined as

\vspace{-0.22in}
\bal\label{eq:GD-Z-two-layer-nn}
\bth{\bZ_{\bS}(t)}_{[(r-1)d+1:rd]i} = \frac {1}{{\sqrt m}} \indict{\bbw_r(t)^\top \bbx_i \ge 0}  \bbx_i a_r, \, i \in [n], r \in [m],
\eal
\vspace{-0.2in}

\noindent where $\bth{\bZ_{\bS}(t)}_{[(r-1)d+1:rd]i} \in \RR^d$ is a vector with elements in the $i$-th column of $\bZ_{\bS}(t)$ with indices in
$[(r-1)d+1:rd]$.
We employ the following particular symmetric random initialization so that $\hat \by(0) = \bzero$, which has been used in existing works such as~\citet{Chizat2019-lazy-training-differentiable-programming,ZhangXLM20}. In our two-layer NN, $m$ is even, $\set{\bbw_{2r'}(0)}_{r'=1}^{m/2}$ and $\set{a_{2r'}}_{r'=1}^{m/2}$ are initialized randomly and independently according to
\bal\label{eq:random-init}
\bbw_{2r'}(0) \sim \cN(\bzero,\kappa^2 \bI_d), a_{2r'} \sim {\textup {unif}}\pth{\left\{-1,1\right\}}, \quad \forall r' \in [m/2],
\eal%
where $\cN(\bmu,\bSigma)$ denotes a Gaussian distribution with mean $\bmu$ and covariance $\bSigma$, ${\textup {unif}}\pth{\left\{-1,1\right\}}$ denotes the uniform distribution over $\set{1,-1}$, $\kappa = \Theta(1) \in (0,1)$ controls the magnitude of initialization. We set $\bbw_{2r'-1}(0) = \bbw_{2r'}(0)$ and $a_{2r'-1} = -a_{2r'}$ for all $r' \in [m/2]$. One can verify that $\hat \by(0) = \bzero$, that is, the initial output of the two-layer NN is zero. Similar to~\citet{du2018gradient-gd-dnns,yang2024gradientdescentfindsoverparameterized,
Yang2025-generalization-two-layer-regression}, once randomly initialized, $\ba$ is fixed during the training. We use $\bW(0)$ to denote the set of all the random weight vectors at initialization, that is, \ $\bW(0) = \set{\bbw_r(0)}_{r=1}^m$.
We run Algorithm~\ref{alg:PGD} to train the two-layer NN by PGD, where
$T$ is the total number of steps for PGD. Early stopping is enforced in Algorithm~\ref{alg:PGD} through a bounded $T$.

%It is remarked that in our Algorithm~\ref{alg:PGD} for PGD, we do not need to
%actually run $T$ steps of PGD by (\ref{eq:PGD-two-layer-nn}). As the preconditioner $\bM$ is independent of the initialization $\bW(0)$
%and the weights of the neural network, we can perform Algorithm~\ref{alg:PGD} to train the network by GD, and then multiply the weights difference
%${\textup{vec}(\tilde \bW(T)) - \vect{\bW(0)}}$ by $\bM$ and add it to $\bW(0)$ to obtain the weights of the neural network trained by PGD. In this manner, the preconditioner is only involved in one matrix
%multiplication instead of $T$ multiplications, considerably improving the efficiency of training by PGD.

\section{Main Results}
\label{sec:main-results}
We present the definition of kernel complexity in Section~\ref{sec:kernel-complexity}, and then introduce the main results for nonparametric regression of this paper in Section~\ref{sec:nonparametric-regression-Kint} for a covariate distribution $P$ supported on $\cX$. Furthermore, Section~\ref{sec:nonparametric-regression-Kint-general-covariate-distribution} presents an extension to a more general covariate distribution $P$ supported on the entire $\RR^d$.
%There is broad interest in the machine learning community
%to consider function class of finite dimension in $\cH_{\Kint}$ or $L^2(\cX,\mu)$.
%For example,~\citep{SuY19-convergence-spectral} considers optimization of a two-layer
%neural network with target functions in
%$\cH_{K}(\mu_0) \cap L^2_{k_0}$ so that the exponent in the
%linear convergence of the training loss is well bounded.
%The assumption about the target function $f^*$ is presented below.
%Throughout this paper we assume that $m \ge \max\set{d,n}$.

\subsection{Kernel Complexity}
\label{sec:kernel-complexity}
The local kernel complexity has been studied by~\citet{bartlett2005,koltchinskii2006,Mendelson02-geometric-kernel-machine}. For the PD kernel $K$, we define the empirical kernel complexity $\hat R_K$
and the population kernel complexity $R_K$ as
\bal\label{eq:kernel-LRC-empirical}
\hat R_K(\eps) \defeq \sqrt{\frac 1n
\sum\limits_{i=1}^n \min\set{\hat \lambda_i,\eps^2}},
\quad
R_K(\eps) \defeq \sqrt{\frac 1n \sum\limits_{i=1}^{\infty} \min\set{\lambda_i,\eps^2}}.
\eal
It can be verified that both
 $\sigma_0 R_K(\eps)$ and $\sigma_0 \hat R_K(\eps)$ are
 sub-root functions~\citep{bartlett2005} in terms of $\eps^2$.
 Sub-root functions are defined in
 Definition~\ref{def:sub-root-function}.
For a given noise ratio $\sigma_0$, the critical empirical radius
$\hat \eps_{K,n} > 0$ is the smallest positive solution to the inequality
$\hat R_K(\eps) \le {\eps^2}/{\sigma_0}$, where
$\hat \eps_{K,n}^2$ is also the fixed point of $\sigma_0 \hat R_K(\eps)$ as a function
of $\eps^2$: $\sigma_0 \hat R_K(\hat \eps_{K,n}) = \hat \eps_{K,n}^2$.
Similarly, the critical population rate $\eps_{K,n}$ is
defined to be the smallest positive solution to the inequality
$R_K(\eps) \le {\eps^2}/{\sigma_0}$, where
$\eps_{K,n}^2$ is the fixed point of $\sigma_0  R_K(\eps)$ as a function of
$\eps^2$: $\sigma_0 R_K(\eps_{K,n}) = \eps_{K,n}^2$.
Kernel complexity can also be defined for the integral kernel $\Kint$, leading to the
empirical kernel complexity $\hat R_{\Kint}$ and the population kernel complexity $R_{\Kint}$ for $\Kint$, with the critical empirical radius $\hat \eps_{\Kint,n}$ and
the critical population rate $\eps_{\Kint,n}$, respectively. For simplicity of notation, we use $\eps_n$ and $\hat \eps_n$ to denote $\eps_{\Kint,n}$ and $\hat \eps_{\Kint,n}$, respectively. In this paper we consider the kernel~$K$ such that
 $\min\set{\eps_{K,n},\eps_n} \cdot n \to \infty$ as $n \to \infty$, which covers most popular positive~semi-definite~kernels including the
 kernel (\ref{eq:kernel-two-layer}) and a broad range of data distributions~\citep{Yang2017-fast-sketch-kernel}.
%\bal\label{eq:kernel-LRC-empirical-inequality}
%\hat R_K(\eps) \le \frac{\eps^2}{\sigma_0}.
%\eal

Let $\eta_t \defeq \eta t$ for all $t \ge 0$. We then define the
 stopping time $\hat T$
 as
\bal\label{eq:stopping-time-hatT}
%\hat T_K \defeq \min\set{T \colon
%\hat R_K(\sqrt{1/\eta_t}) > (\sigma_0 \eta_t)^{-1}}-1,
\hat T \defeq \min\set{t\ge1 \colon \hat R_{\Kint}(\sqrt{1/\eta_t}) > (\sigma_0 \eta_t)^{-1}}-1.
\eal
The stopping time in fact limits the number of steps $T$ for Algorithm~\ref{alg:PGD}, which enforces the early stopping mechanism. In fact, as will
be shown later in this section, we need to have $T \le \hat T$ when training the two-layer NN (\ref{eq:two-layer-nn}) by PGD with Algorithm~\ref{alg:PGD}.
The stopping time $\hat T$ is chosen so that optimization progress and statistical complexity are balanced, which is the standard role of early stopping in localized kernel-complexity analyses.
\subsection{Nonparametric Regression with Covariate Distribution Supported on $\cX$}
\label{sec:nonparametric-regression-Kint}
We present in Theorem~\ref{theorem:minimax-nonparametric-regression-Kint} our main result for nonparametric regression where $f^* \in \bth{\cH_K}^{s'}(\mu_0)$ by training the two-layer NN (\ref{eq:two-layer-nn}) with PGD and early stopping using Algorithm~\ref{alg:PGD}, when the covariate distribution $P$ supported on $\cX$ admits the polynomial EDR, that is, $\lambda_j \asymp j^{-2\alpha}$ for $j \ge 1$ with $2\alpha = d/(d-1)$. It is shown by existing works such as~\citet[Lemma 3.1]{HuWLC21-regularization-minimax-uniform-spherical} that such polynomial EDR holds when $P$ is the spherical uniform distribution, that is, $P = \Unif{\cX}$.
%\begin{assumption}
%\label{assumption:Kint}
%Suppose $P$ is the uniform distribution on $\cX$, and $c_d \log n \le d \le n^{c'_d}$ for two positive constants $c_d, c'_d > 0$.
%\end{assumption}
We note that  such an assumption about spherical uniform distribution has been widely adopted in
the theoretical analysis of neural networks
\citep{BaiL20-quadratic-NTK,Ghorbani2021-linearized-two-layer-nn,
HuWLC21-regularization-minimax-uniform-spherical}.
%the standard results in nonparametric kernel regression~\citep{RaskuttiWY14-early-stopping-kernel-regression}
%to show that the shallow neural network (\ref{eq:two-layer-nn})
%with sufficient parameterization, that is, sufficiently enough $m$,
\begin{theorem}
\label{theorem:minimax-nonparametric-regression-Kint}
Suppose the covariate distribution $P$ has full support on $\cX$ and admits a polynomial EDR such that $\lambda_j \asymp j^{-2\alpha}$ for $j \ge 1$ with $\alpha = d/(2(d-1))$. Suppose $f^*\in\cH_{\Kint}(\mu_0)=\bth{\cH_K}^{s+2}(\mu_0)$ for a fixed radius $\mu_0>0$.
Let $s \ge 1$, $c_T, c_t \in (0,1]$ be positive constants, and
$c_T \hat T \le T \le \hat T$. Suppose $\delta \in (0,1)$,
\bal\label{eq:m-N-cond-LRC-population-NN-concrete}
m \gsim  n^{\frac{25\alpha(s+2)}{2\alpha(s+2)+1}}d^{\frac 52},
\quad
N \gsim n^{\frac{8\alpha(s+2)}{2\alpha(s+2)+1}} \log (4n/{\delta}),
\eal
and the neural network
$f_t = f(\bW(t),\cdot)$ is
trained by PGD in Algorithm~\ref{alg:PGD}
 with the learning rate $\eta = \Theta(1) \in (0,8)$ with
 $T \le \hat T$.
Then for every $t \in [c_t T \colon T]$, with probability (w.p.) at least
$1-\exp\pth{-\Theta(n)}-\delta
-7\exp\pth{-\Theta(n^{\frac{1}{2\alpha(s+2)+1}})} - 2/n$
over $\bw,\bS,\bQ,\bW(0)$, the stopping time satisfies
$\hat T \asymp n^{\frac{2\alpha(s+2)}{2\alpha(s+2)+1}} $, and
\vspace{-.2in}
\bal\label{eq:minimax-nonparametric-regression-Kint}
\Expect{P}{(f_t-f^*)^2} \lsim n^{-\frac{2\alpha(s+2)}{2\alpha(s+2)+1}}.
\eal
\end{theorem}
\begin{remark}
\label{remark:main-result-unitsphere-remark}
It follows from the integral-operator comparison in the proof of Proposition~\ref{proposition:EDR-general-distribution} that
if the Radon-Nikodym derivative  of $\mu$ with respect to $\sigma$, where $\sigma$ is the uniform probability measure on $\cX$, is
${\diff \mu}/{\diff \sigma} = \rho^2$ for a bounded non-zero Riemann-integrable function $\rho$, then the polynomial EDR holds. As a result, Theorem~\ref{theorem:minimax-nonparametric-regression-Kint} holds with such a probability measure $\mu$ of the covariate distribution $P$ whenever its remaining assumptions are satisfied, which covers the spherical uniform distribution as a special case.

\end{remark}

\subsection{Significance of Theorem~\ref{theorem:minimax-nonparametric-regression-Kint}}
\label{sec:main-result-significance}

\noindent\textbf{Statistical Significance.}
Theorem~\ref{theorem:minimax-nonparametric-regression-Kint} establishes the sharp rate
$ \cO(n^{-\frac{d s'}{d s'+d-1}})$
for target functions in $\bth{\cH_K}^{s'}(\mu_0)$ with $s'=s+2\ge 3$. Since Theorem~\ref{theorem:spectrum-Kint} shows that $\lambdaint_j=\lambda_j^{s'} \asymp j^{-2\alpha s'}$, this rate matches the minimax optimal rate associated with kernel regression using $\Kint$ over $\cH_{\Kint}$~\citep{Stone1985,Yang1999-minimax-rates-convergence,Yuan2016-minimax-additive-models}. In this sense, the theorem identifies a sharp statistical regime for the neural network under the interpolation-space target assumption.

\vspace{0.1in}
\noindent\textbf{Reason for Improved Rate.}
Our sharp rate in Theorem~\ref{theorem:minimax-nonparametric-regression-Kint} is sharper than  that in existing works with algorithm guarantees as discussed in Section~\ref{sec:summary-main-results}. Our rate arises because PGD changes the kernel induced by training. Vanilla GD leads to the regular NTK (\ref{eq:kernel-two-layer}), whose spectrum yields the standard rate under the usual RKHS target assumption. In contrast, PGD induces the integral kernel $\Kint$, whose eigenvalues decay faster and therefore have lower associated kernel complexity. In fact, the kernel induced by PGD is $\Kint$ instead of $K$ defined in (\ref{eq:kernel-two-layer}), whose eigenvalues satisfy $\lambdaint_j = \lambda^{s'}_j < \lambda_j$ with $s' \ge 3$. This spectral change is the mechanism that enables the improved rate under spectral bias.

\subsection{Extension to Covariate Distribution Supported on $\RR^d$}
\label{sec:nonparametric-regression-Kint-general-covariate-distribution}
In this section, we consider a covariate distribution $P$ supported on the entire space $\RR^d$ such that the distribution of the normalized covariate admits the polynomial EDR.

In order to handle data supported on $\RR^d$, we introduce a normalization map $\phi \colon \RR^{d} \to \cX= \unitsphere{d-1}, \phi(\bx) = \bx/\ltwonorm{\bx}$ for $\bx\ne\bzero$. Define $\phi(\bzero)$ arbitrarily. This choice has no effect for covariate distributions with a density. We recall that the probability measure of $P$ is still $\mu$, and the push-forward probability measure
on $\cX$, denoted as $\phi^* \mu$ which is
induced by $\phi$, is defined as $(\phi^* \mu)(A)
\defeq \mu(\phi^{-1}(A))$ for any set $A
\subseteq \cX$. For simplicity of notation, we let $\bar \bx = \phi(\bx)$. As a result, $\phi^* \mu$ is the probability measure of the normalized covariate $\bar \bx$.

We let $\mu_j(K,\cX,\mu')$ be the $j$-th eigenvalue of the integral operator
$T_K$ associated with $K$ under the distribution with probability measure
$\mu'$ supported on $\cX$.  Proposition~\ref{proposition:EDR-general-distribution}
shows that when the probability density function $p$ of the covariate distribution $P$ satisfies $p(\bx)\lsim(1+\ltwonorm{\bx}^2)^{-(d+1)/2}$ for all $\bx\in\RR^d$,
then the distribution of the normalized covariate admits the polynomial EDR, that is,  $\mu_j(K,\cX,\phi^*(\mu))\asymp j^{-d/(d-1)}$ holds for all $j \ge 1$.
\begin{proposition}
\label{proposition:EDR-general-distribution}
Suppose the covariate distribution $P$, supported on $\RR^d$, has a
bounded probability density function $p$ s.t. $p(\bx) \lsim ({1+\ltwonorm{\bx}^2})^{-(d+1)/2}$ for all $\bx \in \RR^d$. Then the polynomial EDR $\mu_j(K,\cX,\phi^*(\mu))\asymp j^{-d/(d-1)}$ holds for all $j \ge 1$.
\end{proposition}

When the target function, as a function of the normalized covariate, still lies in the interpolation space corresponding to the NTK (\ref{eq:kernel-two-layer}) defined on the normalized covariates,  then under the condition on the distribution $P$ in Proposition~\ref{proposition:EDR-general-distribution}, it can be shown as a direct consequence of Theorem~\ref{theorem:minimax-nonparametric-regression-Kint} that the sharp regression risk bound (\ref{eq:minimax-nonparametric-regression-Kint}) still holds by training the two-layer NN (\ref{eq:two-layer-nn}) with Algorithm~\ref{alg:PGD}. This result is presented as the following corollary. The algorithm is applied to the normalized training features, with the independent auxiliary sample drawn from $\phi^*\mu$.

\begin{corollary}
\label{corollary:minimax-nonparametric-regression-Kint-general-distribution}
Suppose $f^* = f^*(\bar \bx) \in \bth{\cH_K}^{s'}(\mu_0)$ where $K(\bar \bu, \bar \bv) =  \frac{ \iprod{\bar \bu}{\bar \bv}}{2\pi} \pth{\pi -\arccos \iprod{\bar \bu}{\bar \bv}}$ for every $\bu, \bv \in \RR^d$. Then
under the conditions of Theorem~\ref{theorem:minimax-nonparametric-regression-Kint} and Proposition~\ref{proposition:EDR-general-distribution},
(\ref{eq:minimax-nonparametric-regression-Kint}) still holds with the same high probability as that claimed in Theorem~\ref{theorem:minimax-nonparametric-regression-Kint}.
\end{corollary}

Corollary~\ref{corollary:minimax-nonparametric-regression-Kint-general-distribution} shows that PGD with early stopping still achieves a minimax optimal regression risk rate under  a general
covariate distribution $P$ supported on $\RR^d$ satisfying Proposition~\ref{proposition:EDR-general-distribution}. We remark that
Corollary~\ref{corollary:minimax-nonparametric-regression-Kint-general-distribution} includes the sub-Gaussian distribution considered in \citet{Li2024-edr-general-domain} as a special case. To see this, the density $p$ of the sub-Gaussian distribution considered in \citet{Li2024-edr-general-domain} is assumed to satisfy $p(\bx) \lsim ({1+\ltwonorm{\bx}^2})^{-(d+3)/2}$ for all $\bx \in \RR^d$, which satisfies the condition in our Proposition 5.3. Our Corollary 5.4 holds for such a sub-Gaussian distribution whenever the remaining conditions of Theorem~\ref{theorem:minimax-nonparametric-regression-Kint} are satisfied.

\subsection{Adaptivity to the Smoothness Parameter $s'$ of the Interpolation Space}

%In many practical scenarios, the ground truth smoothness parameter $s'$ of the interpolation space that the target function lies in, that is,
%$f^* \in \bth{\cH_K}^{s'}(\mu_0)$, is unknown. The following proposition states an important property of the interpolation space, that is, the interpolation space is monotonically shrinking with increasing parameter.

In practice, the regularity exponent $s'$ such that $f^*\in\bth{\cH_K}^{s'}(\mu_0)$ may be unknown, whereas the preconditioner in Algorithm~\ref{alg:preconditioner} depends on $s$.
The following nesting property shows that the validity of our guarantees does not require exact knowledge of $s'$, although attaining the rate
corresponding to the largest admissible value of $s$ may require such knowledge.
\begin{proposition}
\label{proposition:interpolation-shrink}
We have $\bth{\cH_K}^{s_1}(\mu_0) \subseteq \bth{\cH_K}^{s_2}(\mu_0)$ for all $s_1 \ge s_2 > 0$.
\end{proposition}
\begin{proof}
Every $f \in \bth{\cH_K}^{s_1}(\mu_0)$ such that
$f = \sum_{j \ge 1} \beta_j e_j$ satisfies $\sum_{j \ge 1} \beta^2_j/\lambda_j^{s_1} \le \mu_0^2$. Since $\lambda_j < 1$ for all $j \ge 1$, we know that $\sum_{j \ge 1} \beta^2_j/\lambda_j^{s_2} \le \mu_0^2$ for $s_1 \ge s_2 > 0$, hence
$f \in \bth{\cH_K}^{s_2}$.\qedhere
\end{proof}

%Proposition~\ref{proposition:interpolation-shrink} provides adaptivity of our PGD algorithm, descried in Algorithm~\ref{alg:PGD}, to the ground truth smoothness parameter $s'$ such that $f^* \in \bth{\cH_K}^{s'}(\mu_0)$. In particular, with $f^* \in \bth{\cH_K}^{s'}$ and $s' \ge 3$, as long as $s$ is chosen such that $1 \le s \le s'-2$ in Algorithm~\ref{alg:preconditioner}, then
%$f^* \in \bth{\cH_K}^{s'} \subseteq  \bth{\cH_K}^{s+2}$. As a result, our main results, including Theorem~\ref{theorem:minimax-nonparametric-regression-Kint} and Corollary~\ref{corollary:minimax-nonparametric-regression-Kint-general-distribution}, still hold with $s'$ replaced with $s+2$. Such adaptivity offers strong theoretical justification for the success of our PGD algorithm even though the ground truth smoothness parameter $s'$ is unknown.
Proposition~\ref{proposition:interpolation-shrink} implies robustness
of PGD to underspecification of the regularity exponent. In particular,
if
$f^*\in\bth{\cH_K}^{s'}(\mu_0)$ with $s'\ge3$, then for every
$1\le s\le s'-2$,
\bals
f^*\in\bth{\cH_K}^{s'}(\mu_0)
\subseteq\bth{\cH_K}^{s+2}(\mu_0).
\eals
Consequently, Theorem~\ref{theorem:minimax-nonparametric-regression-Kint}
and Corollary~\ref{corollary:minimax-nonparametric-regression-Kint-general-distribution}
remain valid with effective regularity exponent $s+2$, yielding the
risk rate
$\cO(n^{-\frac{2\alpha(s+2)}{2\alpha(s+2)+1}})$. In particular, the universal choice $s=1$ is valid for every
$s'\ge3$ and requires no knowledge of $s'$. It guarantees the rate
$\cO(n^{-\frac{6\alpha}{6\alpha+1}})$ uniformly over these target
classes. If $s'>3$, however, this choice does not generally attain the
faster oracle rate associated with the true exponent $s'$. Consequently,
Proposition~\ref{proposition:interpolation-shrink} establishes
robustness to underspecification, rather than full rate-adaptivity to
unknown smoothness. A practical possibility is to select $s$ from a
finite grid using validation data. A theoretical analysis of such a
data-dependent selection procedure is left for future work.

%\begin{remark}
%It follows from~\citep{Stone1985,Yang1999-minimax-rates-convergence,
%Yuan2016-minimax-additive-models} that the rate
% ${1}/n^{\frac{2\alpha}{2\alpha+1}}$ obtained
% in Theorem~\ref{theorem:minimax-nonparametric-regression-Kint} is
% is minimax-optimal for the case that the eigenvalue $\lambda_j$ of $T_k$ is decayed by $\lambda_j \asymp j^{-2\alpha}$. It is noted that
%~\citep{HuWLC21-regularization-minimax-uniform-spherical} obtains
%the minimax rate of $n^{-\frac{d}{2d-1}}$ under the special case
%that the training data are uniformed distributed on $\cX$ and
%$\lambda_j \asymp j^{-\frac{d}{d-1}}$ (in this case $\alpha = \frac{d}{2(d-1)}$).
%However, our minimax rate in Theorem~\ref{theorem:minimax-nonparametric-regression} is derived using more general result in Theorem~\ref{theorem:LRC-population-NN-detail},
%which does assume a particular distribution of the training data.
%\end{remark}

\section{Roadmap of Proofs}
\label{sec:proof-roadmap}
We present the roadmap of proofs for the main result,
 Theorem~\ref{theorem:minimax-nonparametric-regression-Kint}, in this section.
We first present the results about the uniform convergence to the NTK (\ref{eq:kernel-two-layer}) during the training process by PGD in Section~\ref{sec:uniform-convergence-ntk}, then introduce the basic definitions in Section~\ref{sec:basic-definition}.
The detailed roadmap and key technical results for the main result are presented in Section~\ref{sec:detailed-roadmap-key-results}. The difference
between this work and the existing kernel learning theory~\citep{Caponnetto2007_optimal_rates_least_square} is introduced in Section~\ref{sec:difference-existing-kernel-theory}, and our novel proof strategy is described in Section~\ref{sec:novel-proof-strategy}.

The proof of Theorem~\ref{theorem:minimax-nonparametric-regression-Kint} is based on the key technical results in Section~\ref{sec:detailed-roadmap-key-results}, in particular,
the uniform convergence of $\hKint(\cdot,\bx')$ to $\Kint(\cdot,\bx')$ for every fixed $\bx' \in \cX$ in Theorem~\ref{theorem:hatKint-close-to-Kint-supnorm-informal},
the generalization result in Theorem~\ref{theorem:LRC-population-NN-eigenvalue-informal}, and the optimization results in
Lemma~\ref{lemma:empirical-loss-convergence-informal},
Theorem~\ref{theorem:bounded-NN-class-informal}, and
Lemma~\ref{lemma:empirical-loss-bound-informal}.

\subsection{Uniform Convergence to the NTK and More}
\label{sec:uniform-convergence-ntk}
We define the following functions:
\bal
h(\bw,\bx,\by) &\defeq \bx^{\top} \by \indict{\bw^{\top} \bx \ge 0} \indict{\bw^{\top} \by \ge 0}, \quad &\hat h(\bW,\bx,\by) &\defeq \frac {1}{m} \sum\limits_{r=1}^m h(\bbw_r,\bx,\by), \label{eq:h-hat-h}  \\
v_R(\bw,\bx) &\defeq \indict{\abth{\bw^{\top}\bx} \le R}, \quad &\hat v_R(\bW,\bx) &\defeq  \frac 1m \sum\limits_{r=1}^m v_R(\bbw_r,\bx) \label{eq:v-hat-v}.
\eal%
Then we have the following theorem stating the uniform convergence of $\hat h(\bW(0),\cdot,\cdot)$ to the NTK $K$ defined in (\ref{eq:kernel-two-layer}) and uniform convergence
of $\hat v_R(\bW(0),\bx)$ to $0$.

\begin{theorem}
[{\citet[Theorem 6.1]{Yang2025-generalization-two-layer-regression}},{\citet[Theorem VI.7, VI.8]{yang2024gradientdescentfindsoverparameterized}}]
\label{theorem:good-random-initialization}
The following results hold with
$\eta \lsim 1$, $m \gsim \max\set{n^{2/d},\Theta(T^{\frac 53})}$, and $m/\log m \ge d$. Here $R=\eta c_{\bu}T/(2\sqrt m)$, with $c_{\bu}$ defined in Section~\ref{sec:basic-definition}. The learning-rate condition is used to specialize part~(2) to this training-dependent radius. The underlying concentration inequality at a fixed radius does not depend on $\eta$.

\begin{itemize}[leftmargin=.26in]
\item[(1)] With probability at least $1-1/n$ over the random initialization $\bW(0) = \set{\bbw_r(0)}_{r=1}^m$,
\bal\label{eq:good-initialization-sup-hat-h}
\sup_{\bx \in \cX,\by \in \cX} \abth{ K(\bx,\by) - \hat h(\bW(0),\bx,\by) } \le C_1(m/2,d,1/n) \lsim \sqrt{\frac{d \log m}{m}},
\eal%

\item[(2)] With probability at least $1-1/n$ over the random initialization $\bW(0) = \set{\bbw_r(0)}_{r=1}^m$,
\bal\label{eq:good-initialization-sup-hat-V_R}
&\sup_{\bx \in \cX}\abth{\hat v_R(\bW(0),\bx)} \le \frac{2R}{\sqrt {2\pi} \kappa} + C_2(m/2,d,1/n) \lsim \sqrt{d} m^{-\frac 15} T^{\frac 12}.
\eal%
 \end{itemize}
\end{theorem}
%\begin{proof}%[\textup{Proof of Theorem~\ref{theorem:good-random-initialization}}]
%Note that $\hat h(\bW,\bx,\by) = \frac {1}{m} \sum\limits_{r=1}^m h(\bbw_r,\bx,\by) = \frac {1}{m/2} \sum\limits_{r'=1}^{m/2} h(\bbw_{2r'}(0),\bx,\by)$,
%then part (1) directly follows from Theorem~\ref{theorem:sup-hat-g} in the appendix. Similarly, part (2) directly follows from Theorem~\ref{theorem:V_R} and noting that $R \le \sqrt  R$ with $M  \gsim T^2$.
%\end{proof}

\vspace{-0.2in}
\subsection{Basic Definitions}\label{sec:basic-definition}
We introduce the following definitions for the proof of the main result. Let the gram matrix of $\Kint$ over the training features $\bS$ be $\bKint \in \RR^{n \times n}, \bKint_{ij} = \Kint(\bbx_i,\bbx_j)$ for $i,j \in [n]$, and $\bKint_n \defeq \bKint/n$. Similarly, $\hbKint \in \RR^{n \times n}$ is the gram matrix of $\hKint$ over $\bS$, and
$\hbKint_n = \hbKint/n$. Let the eigendecomposition of $\bKint_n$ be $\bKint_n = \bUint \bSigmaint {\bUint}^{\top}$, where $\bSigmaint$ is a diagonal matrix with its diagonal elements $\set{\hlambdaint_i}_{i=1}^n$ being the eigenvalues of $\bKint_n$ and sorted in a non-increasing order. We have $\hlambdaint_1 \in (0,1/8]$. 
%and we show in Proposition~\ref{proposition:Kint-gram-nonsingular} deferred to Section~\ref{sec:RKHS-Kint-more-results} that $\bKint_n$ is always non-singular.

%\bal\label{eq:ut}
%\bu(t) \defeq \hat \by(t) - \by.
%\eal
We define $\bu(t) \defeq \hat \by(t) - \by$. Let $\tau \le 1$ be a positive number. We define
$c_{\bu} \defeq {\mu_0}/{\min\set{2{\sqrt 2},\sqrt {2e \eta }}}
+ \sigma_0 + \tau + 1$.
For $t \ge 0$ and $T \ge 1$, we define
\bal\label{eq:def-R}
&R \defeq \frac{\eta c_{\bu}T }{2\sqrt m},
\eal\vspace{-0.2in}
\bal\label{eq:cV_S}
%&\cV_t \defeq \set{\bv \in \RR^n \colon \bv \in \Span(\bU), \bv = -\pth{\bI-\frac{\eta}{n} \bK }^t f^*(\bS)}.
\cV_t \defeq \set{\bv \in \RR^n \colon \bv = -\pth{\bI_n- \eta \bKint_n }^t f^*(\bS)},
\eal\vspace{-0.2in}
\bal\label{eq:cE_S}
%\cE_{t,\tau} \defeq &\left\{\be \colon \be = \bbe_1 + \bbe_2, \bbe_1,\bbe_2 \in \RR^n, \right. \nonumber \\
%&\left. \bbe_1 = -\pth{\bI-\eta\bKint_n}^t \bw,
%\ltwonorm{\bbe_2} \le \pth{c_{\bu}-c_{\bu} + \tau} {\sqrt n} \right\}.
\cE_{t,\tau} \defeq &\set{\be \colon \be = \bbe_1 + \bbe_2 \in \RR^n, \bbe_1 = -\pth{\bI_n-\eta\bKint_n}^t \bw,
\ltwonorm{\bbe_2} \lsim {\sqrt n} \tau },
\eal
where the upper bound for $\ltwonorm{\bbe_2}$ only hides an absolute positive constant independent of~$t$.
The set $\cV_t$ represents the clean signal trajectory in the training dynamics of PGD. The set $\cE_{t,\tau}$ captures the controlled residual and noise components which are added to the clean signal trajectory.
We define
\bal\label{eq:set-of-good-random-initialization}
\cW_0 \defeq \set{\bW(0) \colon (\ref{eq:good-initialization-sup-hat-h}) ,(\ref{eq:good-initialization-sup-hat-V_R}) \textup { hold}}
\eal%
as the set of all the good random initializations which satisfy (\ref{eq:good-initialization-sup-hat-h}) and (\ref{eq:good-initialization-sup-hat-V_R}) in Theorem~\ref{theorem:good-random-initialization}.
Theorem~\ref{theorem:good-random-initialization} shows that we have good random initialization with high probability, that is, $\Prob{\bW(0)
\in \cW_0} \ge 1-2/n$. When $\bW(0) \in \cW_0 $, the uniform convergence results,~(\ref{eq:good-initialization-sup-hat-h})~and~(\ref{eq:good-initialization-sup-hat-V_R}), hold, which guarantees our main result about optimization of the neural network~by~PGD.

We define the set of neural network weights and the set of functions represented by the neural network during training as follows.
%\vphantom{ \int_1^2 }
\bal\label{eq:weights-nn-on-good-training}
&\cW(\bS,\bQ,\bW(0),T) \defeq \left\{\bW \colon \exists t \in [T] {\textup{ s.t. }}\vect{\bW} = \vect{\bW(0)} - \sum_{t'=0}^{t-1} \frac{\eta}{n} \bM \bZ_{\bS}(t') \bu(t'), \right. \nonumber \\
& \left. \bu(t') \in \RR^{n}, \bu(t') = \bv(t') + \be(t'),
\bv(t') \in \cV_{t'}, \be(t') \in \cE_{t',\tau}, {\textup { for all } } t' \in [0,t-1] \vphantom{\frac12}  \right\}.
\eal%

$\cW(\bS,\bQ,\bW(0),T)$ is the set of weights of the neural network trained by PGD on the training features $\bS$ and random initialization $\bW(0)$ with
the preconditioner $\bM$ generated by $\bQ$ and the steps of PGD not greater than $T$.
%By (\ref{eq:weights-nn-on-good-training}), the vectorization of any weight in $\cW(\bS,\cW_0,T)$ is expressed by $\vect{\bW(0)} - \sum_{t'=0}^{t-1} \frac{\eta}{n}  \bZ_{\bS}(t') \bu(t')$.
The set of functions represented by the two-layer NN with weights in $\cW(\bS,\bQ,\bW(0),T)$ is then defined as
\bal\label{eq:random-function-class}
\cFnn(\bS,\bQ,\bW(0),T) \defeq \set{f_t = f(\bW(t),\cdot) \colon \exists \, t \in [T], \bW(t) \in \cW(\bS,\bQ,\bW(0),T)}.
\eal%

We define the function class $\cFext(B,w)$ for any $B,w > 0$ as
\bal\label{eq:def-cF-ext-general}
\cFext(B,w) &\defeq \set{f \colon f = h+e, h \in \cH_{\Kint}(B),
\supnorm{e} \le w}.
\eal
In Theorem~\ref{theorem:bounded-NN-class-informal}, we will show that
the neural network function at each step of PGD, $f_t$, belongs to the space $\cFext(B_h,w)$ with high probability where
\bal\label{eq:B_h}
B_h &\defeq \mu_0 +1+ {\sqrt 2}.
\eal

\subsection{Detailed Roadmap and Key Technical Results}
\label{sec:detailed-roadmap-key-results}

%Our main result, Theorem~\ref{theorem:LRC-population-NN-fixed-point}, is proved by
%Theorem~\ref{theorem:LRC-population-NN-eigenvalue} and Theorem~\ref{theorem:empirical-loss-bound} deferred to Section~\ref{sec:detailed-roadmap-key-results}.
The summary of the  approaches and key technical results in the proofs is presented as follows.
%we illustrate in Figure~\ref{fig:proof-roadmap} the roadmap containing the intermediate key theoretical results which lead
%to Theorem~\ref{theorem:LRC-population-NN-fixed-point}.
%\vspace{0.1in}\noindent \textbf
%\noindent \textbf{Summary of the  approaches and key technical results in the proofs.}
Our main result is built upon the following three significant technical results of independent interest. In the following text, $\delta \in (0,1)$ denotes a small probability.

First, we establish the uniform convergence of the empirical integral kernel $\hKint(\cdot,\bx')$ to the  integral kernel $\Kint(\cdot,\bx')$ for every fixed $\bx' \in \cX$. Such uniform convergence is
presented in the following theorem, Theorem~\ref{theorem:hatKint-close-to-Kint-supnorm-informal}, which is built upon the martingale concentration inequality in Hilbert spaces~\citep[Theorem~3.5]{Pinelis1994}.
\begin{theorem}
\label{theorem:hatKint-close-to-Kint-supnorm-informal}
For every fixed $\bx' \in \cX$ and every $\delta \in (0,1)$, with probability at least
$1-\delta$ over $\bQ = \set{\bbq_i}_{i=1}^N$,  we have
\bal\label{eq:hatKint-close-to-Kint-supnorm-informal}
\sup_{\bx \in \cX}\abth{\hKint(\bx,\bx') - \Kint(\bx,\bx')} \lsim \sqrt{\frac{\log(4/\delta)}{N}}.
\eal
\end{theorem}
Theorem~\ref{theorem:hatKint-close-to-Kint-supnorm-informal}  is the key tool needed to analyze the training dynamics of PGD which involves the integral kernel $\Kint$.
Then, using the novel PGD algorithm and the uniform convergence in Theorem~\ref{theorem:hatKint-close-to-Kint-supnorm-informal}, we have the following optimization result
in Lemma~\ref{lemma:empirical-loss-convergence-informal}.
\begin{lemma}\label{lemma:empirical-loss-convergence-informal}
Suppose the network width $m$ and $N$ are sufficiently large and finite,
and the neural network $f(\bW(t),\cdot)$ is trained by PGD using Algorithm~\ref{alg:PGD}
with the learning rate $\eta = \Theta(1) \in (0,8)$ on the random initialization $\bW(0) \in \cW_0$. Then for every $\delta \in (0,1)$, with
high probability over the random noise $\bw$ and the random sample $\bQ$,
$\bW(t) \in \cW(\bS,\bQ,\bW(0),T)$.
Moreover, for all $t \in [0,T]$, $\bu(t) = \bv(t) + \be(t)$ where
$\bu(t) = \hat \by(t) -  \by$, $\bv(t) \in \cV_t$, $\be(t) \in \cE_{t,\tau}$,
and $\ltwonorm{\bu(t)} \le c_{\bu} \sqrt n$.
\end{lemma}
Lemma~\ref{lemma:empirical-loss-convergence-informal} formalizes that the PGD residual dynamics $\bu(t)$ can be decomposed into a structured clean signal term $\bv(t) $ and a controlled error term $\be(t)$, which is the basis for the function-level decomposition to be established next.
Based on Lemma~\ref{lemma:empirical-loss-convergence-informal}, we obtain a novel and principled decomposition of the neural network function at any step of PGD into a function in  $\cH_{\Kint}$ with bounded RKHS-norm, and an error function with a small $L^{\infty}$-norm with high probability.

\begin{theorem}\label{theorem:bounded-NN-class-informal}
Suppose the network width $m$ and $N$ are sufficiently large and finite,
and the neural network
$f_t = f(\bW(t),\cdot) $ is trained by PGD using Algorithm~\ref{alg:PGD} with the learning rate $\eta = \Theta(1) \in (0,8)$ and the random initialization $\bW(0) \in \cW_0$. Then for every $t \in [T]$, with high probability over $\bw$ and $\bQ$,
$f_t$ has the following decomposition on $\cX$: $f_t = h_t + e_t \in \cFext(B_h,w)$,
where $h_t \in \cH_{\Kint}(B_h)$ with $B_h$ defined in (\ref{eq:B_h}),
$e_t  \in L^{\infty}$ with sufficiently small magnitude on $\cX$, that is, $\supnorm{e_t} \le w$.
\end{theorem}

Second, a new technique based on local Rademacher complexity is developed to tightly bound the nonparametric regression risk in
Theorem~\ref{theorem:LRC-population-NN-eigenvalue-informal}, which is based on
 the Rademacher complexity of a localized subset of a much larger function class presented in
Lemma~\ref{lemma:LRC-population-NN} in the appendix. We then use
Theorem~\ref{theorem:bounded-NN-class-informal} and Lemma~\ref{lemma:LRC-population-NN}
to derive Theorem~\ref{theorem:LRC-population-NN-eigenvalue-informal}.

\begin{theorem}\label{theorem:LRC-population-NN-eigenvalue-informal}
Suppose $w  \in (0,1)$, the network width $m$ and $N$ are sufficiently large and finite,
and the neural network
$f_t = f(\bW(t),\cdot)$ is
trained by PGD using Algorithm~\ref{alg:PGD} with the learning rate  $\eta = \Theta(1) \in (0,8)$
on the random initialization $\bW(0)$, and $T \le \hat T$.
Then for every $t \in [T]$, with high probability over $\bw$, the random training features $\bS$, $\bQ$, and $\bW(0)$,
\bal\label{eq:LRC-population-NN-bound-eigenvalue-informal}
&\Expect{P}{(f_t-f^*)^2} - 2 \Expect{P_n}{(f_t-f^*)^2}
\lsim \min_{0 \le Q \le n} \pth{\frac{ Q}{n} +
\pth{\frac{\sum\limits_{q = Q+1}^{\infty}\lambdaint_q}{n}}^{1/2}} + \eps_n^2 +w,
\eal
where  $\eps_n=\eps_{\Kint,n}$ is the population critical radius associated with $\Kint$, namely the smallest positive $\eps$ satisfying $R_{\Kint}(\eps)\le\eps^2/\sigma_0$.
\end{theorem}
Theorem~\ref{theorem:LRC-population-NN-eigenvalue-informal} converts the spectral decay of $\Kint$ into a sharp population risk control for the neural network trained by PGD via localized Rademacher complexity.

Third, we obtain the following sharp upper bound for the training loss $\Expect{P_n}{(f_t-f^*)^2}$. Lemma~\ref{lemma:empirical-loss-bound-informal} provides the empirical loss decay needed to complement the local complexity-based generalization bound in Theorem~\ref{theorem:LRC-population-NN-eigenvalue-informal} and close the final rate.
\begin{lemma}\label{lemma:empirical-loss-bound-informal}
 For every $t \in [T]$, with high probability over $\bw$,
\bal\label{eq:empirical-loss-bound-informal}
\Expect{P_n}{(f_t-f^*)^2} &\le
\frac{3}{\eta t} \pth{\frac{\mu_0^2}{2e} +\frac{1}{\eta T}+2},
\eal
where $\mu_0$ is the target-class radius so that
$f^*\in\cH_{\Kint}(\mu_0)=\bth{\cH_K}^{s'}(\mu_0)$.
\end{lemma}

\noindent\textbf{Proof Sketch of Theorem~\ref{theorem:minimax-nonparametric-regression-Kint}.} We can then prove Theorem~\ref{theorem:minimax-nonparametric-regression-Kint} using the upper bound for the regression risk in (\ref{eq:LRC-population-NN-bound-eigenvalue-informal}) of Theorem~\ref{theorem:LRC-population-NN-eigenvalue-informal} where $w$ is set to $n^{-\frac{2\alpha(s+2)}{2\alpha(s+2)+1}}$. With $\hat T \asymp n^{\frac{2\alpha(s+2)}{2\alpha(s+2)+1}}$,
it follows from (\ref{eq:empirical-loss-bound-informal}) of Lemma~\ref{lemma:empirical-loss-bound-informal} that
$\Expect{P_n}{(f_t-f^*)^2} \lsim n^{-\frac{2\alpha(s+2)}{2\alpha(s+2)+1}}$. Then we have
$\Expect{P}{(f_t-f^*)^2}  \lsim n^{-\frac{2\alpha(s+2)}{2\alpha(s+2)+1}}$ with
$Q \asymp n^{\frac{1}{2\alpha(s+2)+1}}$ in (\ref{eq:LRC-population-NN-bound-eigenvalue-informal}). The formal proof of
Theorem~\ref{theorem:minimax-nonparametric-regression-Kint} is presented below.

\vspace{.1in}
\begin{proof}
{\textbf{\textup{\hspace{-3pt}of
Theorem~\ref{theorem:minimax-nonparametric-regression-Kint}}}.}
%[\textup{\textbf{Detailed Proof of Theorem~\ref{theorem:minimax-nonparametric-regression-Kint}}}]
%\begin{proof}
We apply Theorem~\ref{theorem:LRC-population-NN} and Lemma~\ref{lemma:empirical-loss-bound} in the appendix to prove this theorem,
and Theorem~\ref{theorem:LRC-population-NN} and Lemma~\ref{lemma:empirical-loss-bound} are the formal versions of
Theorem~\ref{theorem:LRC-population-NN-eigenvalue-informal} and Lemma~\ref{lemma:empirical-loss-bound-informal}, respectively.

First, let $\mathcal A_n$ be the critical-radius comparison event used in the proof of Theorem~\ref{theorem:LRC-population-NN}. On this event, the stopping-time estimate below gives $T\le\hat T\lsim n^{2\alpha(s+2)/(2\alpha(s+2)+1)}$. The width condition implies $m\gsim\max\set{n^{2/d},T^{5/3}}$ and $m/\log m\ge d$ for all sufficiently large $n$. Applying Theorem~\ref{theorem:good-random-initialization} conditionally on $\bS$ gives $\Prob{\mathcal A_n\cap\set{\bW(0)\notin\cW_0}}\le2/n$. It follows from
Theorem~\ref{theorem:spectrum-Kint} in Section~\ref{sec:RKHS-Kint-more-results}
that $\lambdaint_j  = \lambda_j^{s+2} \asymp j^{-2\alpha(s+2)}$ for $j \ge 1$.
With $\eta = \Theta(1)$, it follows by
Lemma~\ref{lemma:empirical-loss-bound} that with probability at least
$1-\exp\pth{- \Theta(n\hat\eps_n^2)}$ over the random noise $\bw$,
\bals
\Expect{P_n}{(f_t-f^*)^2} &\lsim \frac{1}{\eta t}.
\eals
Plugging such bound for $\Expect{P_n}{(f_t-f^*)^2}$ in
(\ref{eq:LRC-population-NN-bound})
of Theorem~\ref{theorem:LRC-population-NN}
leads to
\bal\label{eq:LRC-population-NN-fixed-point-seg1}
&\Expect{P}{(f_t-f^*)^2} \lsim \frac{1}{\eta t}  + \pth{\frac{Q}{n} + \pth{\frac{\sum\limits_{q = Q+1}^{\infty}\lambdaint_q}{n}}^{1/2}} + \eps_n^2 +w,
\eal
where $Q \asymp n^{\frac{1}{2\alpha(s+2)+1}}$ on the RHS
of (\ref{eq:LRC-population-NN-fixed-point-seg1}).
%\begin{sloppypar}
%\end{sloppypar}
It can be verified that
\bals
\sum\limits_{q = Q+1}^{\infty} \lambdaint_q
\lsim \int_Q^{\infty} x^{-2\alpha(s+2)} \diff x \lsim 1/(2\alpha(s+2)-1)
 \cdot n^{-\frac{2\alpha(s+2)-1}{2\alpha(s+2)+1}}.
\eals
As a result, we have
\bal\label{eq:LRC-population-NN-fixed-point-seg2}
&\Expect{P}{(f_t-f^*)^2} - \frac{6}{\eta t} \pth{\frac{\mu_0^2}{2e}+\frac{1}{\eta}+2}
\lsim n^{-\frac{2\alpha(s+2)}{2\alpha(s+2)+1}} + \eps_n^2 +w.
\eal
Due to the definition of $\hat T$ and $\hat \eps_n^2$, we have
\bal\label{eq:LRC-population-NN-fixed-point-seg3}
\hat \eps_n^2 \le \frac {1}{\eta \hat T} \le \frac {2}{\eta (\hat T+1)}
\le 2 \hat \eps_n^2.
\eal
It is well known, such as~\citet[Corollary 3]{RaskuttiWY14-early-stopping-kernel-regression},
that $\eps_n^2 \asymp n^{-\frac{2\alpha(s+2)}{2\alpha(s+2)+1}}$. It follows from Lemma~\ref{lemma:hat-eps-eps-relation} 
that $\hat \eps_n^2 \asymp \eps_n^2$ with probability
at least $1-4\exp(-\Theta(n\eps_n^2))$ over $\bS$. In addition,
combined with the fact that $T \asymp \hat T$, for any
$t \in [c_t T,T]$, we have
\bals
\frac{1}{\eta t} \asymp \frac{1}{\eta \hat T} \asymp
\frac{1}{\eta T} \asymp \hat \eps_n^2 \asymp \eps_n^2 \asymp n^{-\frac{2\alpha(s+2)}{2\alpha(s+2)+1}}.
\eals
Also, we set $w = n^{-\frac{2\alpha(s+2)}{2\alpha(s+2)+1}}$ in
(\ref{eq:LRC-population-NN-fixed-point-seg2}),
then with $\eta = \Theta(1)$,
\bal\label{eq:LRC-population-NN-fixed-point-seg4}
&\Expect{P}{(f_t-f^*)^2} \lsim n^{-\frac{2\alpha(s+2)}{2\alpha(s+2)+1}}.
\eal
With
$N \gsim n^{\frac{8\alpha(s+2)}{2\alpha(s+2)+1}} \log (4n/{\delta})$,
the requirement on $N$, (\ref{eq:N-cond-bounded-NN-class})
in Theorem~\ref{theorem:bounded-NN-class}~that
$N \gsim \max\set{T^2\log{(4n/{\delta})}/w^2, T^4\log{(4n/{\delta})}}$
is satisfied.
In addition, with $m \gsim  n^{\frac{25\alpha(s+2)}{2\alpha(s+2)+1}}d^{\frac 52}$
and $w = n^{-\frac{2\alpha(s+2)}{2\alpha(s+2)+1}}$, the condition
(\ref{eq:m-cond-bounded-NN-class}) on $m$ in
Theorem~\ref{theorem:bounded-NN-class} that
$m \gsim \max\{ {T^{\frac {15}{2}} d^{\frac 52}}/{w^5}, \newline T^{\frac{25}{2}}d^{\frac 52}\}$ is satisfied.
The joint event constructed in the proof of Theorem~\ref{theorem:LRC-population-NN} includes the initialization, critical-radius comparison, trajectory, and local-complexity events. On $\mathcal A_n$, the additional noise-filter failure probability in Lemma~\ref{lemma:empirical-loss-bound} is bounded by $\exp(-\Theta(n\eps_n^2))$. A union bound adds this term to the six exponential terms in Theorem~\ref{theorem:LRC-population-NN}, while counting the initialization failure $2/n$ and the failure of $\mathcal A_n$ only once, yielding the stated probability.
\end{proof}

\subsection{Difference from Existing Kernel Learning Theory}
\label{sec:difference-existing-kernel-theory}
In this section, we compare our PGD guarantee with existing kernel learning theory and vanilla-GD results under matching source and spectral parameters. The comparison distinguishes logarithmic confidence factors from the polynomial risk exponent. % Under the same source condition on the target function that $f^* \in \cH_{\Kint}(\mu_0) =\bth{\cH_K}^{s'}(\mu_0)$, the existing upper bound for regression risk rate in~\citet[Theorem 1]{Caponnetto2007_optimal_rates_least_square} for kernel regression using the regular NTK (\ref{eq:kernel-two-layer}) is $\cO(\tau(\delta) n^{-\frac{d s'}{ds'+d-1}})$ with probability $1-\delta$, and $\tau(\delta) \to \infty$ as $\delta \to 0$. 
% Furthermore, under the same source condition, the spectral algorithm in
% \citet[Theorem 1]{Zhang2025-misspecified-spectral-algorithm} achieves the regression risk rate of $\cO(n^{-\frac{d s'}{ds'+d-1}})\log^2(1/\delta)$ with probability $1-\delta$ for every $\delta \in (0,1)$. 
%  Therefore, the rate obtained by~\citet[Proposition 13]{Li2024-edr-general-domain}, $\cO(n^{-\frac{d s'}{ds'+d-1}})\log^2(1/\delta)$ with the additional logarithmic factor $\log^2(1/\delta)$ compared to our rate in Theorem~\ref{theorem:minimax-nonparametric-regression-Kint}, is consistent with the upper bound in \citet[Theorem 1]{Caponnetto2007_optimal_rates_least_square}
% and the upper bound in \citet[Theorem 1]{Zhang2025-misspecified-spectral-algorithm}. On the other hand, \citet[Theorem 2]{Zhang2025-misspecified-spectral-algorithm} establishes the lower bound of $\Theta(n^{-\frac{d s'}{ds'+d-1}})$, which is achieved by our Theorem~\ref{theorem:minimax-nonparametric-regression-Kint}.

Under the source condition
$f^*\in\cH_{\Kint}(\mu_0)=\bth{\cH_K}^{s'}(\mu_0)$
with $s'\ge3$, and certain spectral, embedding, and noise
assumptions,
the spectral algorithm in \citet[Theorem~1]{Zhang2025-misspecified-spectral-algorithm} using the regular NTK
\eqref{eq:kernel-two-layer}, with qualification at least
$s'/2$, achieves the regression risk bound
$\cO(n^{-\frac{ds'}{ds'+d-1}}\log^2(6/\delta))$
with probability at least $1-\delta$, for every fixed
$\delta\in(0,1)$ and sufficiently large $n\ge n_0(\delta)$.
This logarithmic confidence dependence $\log^2(6/\delta)$ is consistent with
the  vanilla-GD upper bound in
\citet[Proposition~13]{Li2024-edr-general-domain},
expressed under the matching source and spectral parameters
in Table~\ref{table:main-results-comparison}.
Related confidence dependence appears in the finite-sample
regularized least-squares upper bound of
\citet[Theorem~4]{Caponnetto2007_optimal_rates_least_square},
which uses the multiplier $32\log^2(6/\delta)$.

Furthermore,
\citet[Theorem~2]{Zhang2025-misspecified-spectral-algorithm}
establishes the minimax lower bound
$\Omega(n^{-\frac{ds'}{ds'+d-1}})$ over the corresponding
source ball.
Our Theorem~\ref{theorem:minimax-nonparametric-regression-Kint}
matches this minimax lower bound without a logarithmic
factor in the risk bound, under its stated assumptions
and probability guarantee.

%\vspace{-.01in}
In the comparisons above, canonical kernel regression methods~\citep{Caponnetto2007_optimal_rates_least_square,Yao2007_gd_early_stopping} are applied to the regular NTK (\ref{eq:kernel-two-layer}) with EDR $\lambda_j\asymp j^{-2\alpha}$. They are not applied to the PGD-induced kernel $\Kint$ in these comparisons. The general kernel frameworks can also be applied to $\Kint$ when their assumptions are satisfied.
%In fact, it is required in~\citet{Caponnetto2007_optimal_rates_least_square,Yao2007_gd_early_stopping,Li2024-edr-general-domain} that the target function satisfies a condition specified by the kernel $K$ with: $f^* \in T_K^{1/2} f’ $ with all $f’ \in \cH_K$, where $\cH_K $ is the Reproducing Kernel Hilbert Space (RKHS) associated with $K$.
On the other hand, the two-layer NN, trained by our novel PGD, approximately performs kernel regression with a completely different new kernel, the integral kernel $\Kint$ with the smoother capacity condition ($\lambdaint_j = \lambda_j^{s'} \asymp j^{-2\alpha s'}$), compared to the original capacity condition.  Our key insight is that, the interpolation space $\bth{\cH_K}^{s'}$ is in fact the RKHS associated with the integral kernel: $\bth{\cH_K}^{s'} = \cH_{\Kint}$. Kernel regression with the integral kernel and the target function $f^* \in \cH_{\Kint} (\mu_0)$ renders the minimax optimal rate of $\cO(n^{-\frac{ds'}{d s'+d-1}})$ according to the analytical results in~\citet{Stone1985,Yang1999-minimax-rates-convergence,
Yuan2016-minimax-additive-models}, which is the same rate as that in Theorem~\ref{theorem:minimax-nonparametric-regression-Kint}.
%As a result, one cannot apply the existing theories including those in Rosasco and Caponetto [1-2] to obtain our sharper rate of $n^{-\frac{2d}{3d-1}}$.

\vspace{0.1in}
\noindent \textbf{Learning in the Induced RKHS.} Our PGD analysis decomposes the trained network as $f_t=h_t+e_t$, with $h_t\in\cH_{\Kint}(B_h)$ and $\norm{e_t}{\infty}\le w$. In the matched-kernel notation of Table~\ref{table:main-results-comparison}, existing vanilla-GD methods such as \citet[Proposition~13]{Li2024-edr-general-domain} approximate kernel regression in $\cH_K$ and, under the comparable source condition $s'\ge3$, learn target functions in $\cH_K$. The common target class satisfies $\cH_{\Kint}=\bth{\cH_K}^{s'}\subseteq\cH_K$. PGD changes the reference kernel and controls the RKHS component of its iterates in $\cH_{\Kint}$.
%In addition, the integral kernel $\Kint$ is not the regular NTK of any known neural network.
Under the matching source and spectral parameters in Table~\ref{table:main-results-comparison}, the existing vanilla-GD bound in \citet[Proposition~13]{Li2024-edr-general-domain} contains the logarithmic confidence factor $\log^2(12/\delta)$, whereas our Theorem~\ref{theorem:minimax-nonparametric-regression-Kint} gives a log-free bound under its stated assumptions and probability guarantee. These statements compare proved upper bounds. It is technically nontrivial to induce such a new kernel $\Kint$ when training the two-layer NN with the novel PGD algorithm, to be detailed by our novel proof strategies in the next subsection. Our results also give a smaller sample-size exponent in the sufficient condition for network width compared with the cited analyses, as detailed in the next subsection. We remark that PGD is used in the training process to change the effective kernel from $K$ to $\Kint$, and the trained two-layer NN approximates kernel regression with this integral kernel.

\subsection{Novel Proof Strategy and Insights}
\label{sec:novel-proof-strategy}
Our results are built upon two significantly novel proof strategies. First, uniform convergence to the NTK is established during the training process by PGD, so that we can have a novel decomposition of the neural network function at any step of PGD into a function in the RKHS associated with the integral kernel $\Kint$ and an error function with small $L^{\infty}$-norm with high probability in Theorem~\ref{theorem:bounded-NN-class-informal}.
We remark that PGD is carefully designed so that the induced NTK of the two-layer NN (\ref{eq:two-layer-nn}) is the integral kernel $\Kint$. In a strong contrast, the vanilla GD extensively used in the existing works for the analysis of over-parameterized neural networks, such as~\citet{SuhKH22-overparameterized-gd-minimax,Li2024-edr-general-domain,yang2024gradientdescentfindsoverparameterized,
Yang2025-generalization-two-layer-regression},
induces the regular NTK such as (\ref{eq:kernel-two-layer}).
The uniform convergence of $\hKint$ to $\Kint$ is established in this work, using the martingale concentration inequality in Hilbert spaces~\citep[Theorem~3.5]{Pinelis1994}.
As explained in Section~\ref{sec:summary-main-results} and Section~\ref{sec:main-results}, PGD leads to a sharper rate of the nonparametric regression risk due to the lower kernel complexity of $\Kint$ than that of the regular NTK such as (\ref{eq:kernel-two-layer}).

Second, based on such novel decomposition in Theorem~\ref{theorem:bounded-NN-class-informal}, a new technique based on local Rademacher complexity is developed
to tightly bound the Rademacher complexity of the function class comprising all the neural network functions obtained by PGD iterations, leading to the sharp bound for the regression risk in Theorem~\ref{theorem:LRC-population-NN-eigenvalue-informal}.

PGD changes the effective kernel within the reference training dynamics: it induces $\Kint$, whose complexity is lower than that of the regular NTK $K$. The resulting log-free guarantee is proved under the stated source condition, without asserting that the network must leave the linearized training regime.
%In fact, Theorem~\ref{theorem:bounded-NN-class} in the appendix shows that the kernel induced by the neural network (\ref{eq:two-layer-nn}) trained with PGD can be approximated by $\Kint$ with high probability. As explained in ``Significance of Theorem~\ref{theorem:minimax-nonparametric-regression-Kint}'', the kernel complexity of $\Kint$ is lower than that of $K$, leading to the sharper rate of the risk (\ref{eq:minimax-nonparametric-regression-Kint}) by PGD.

\vspace{0.1in}
\noindent \textbf{Comparison of Sufficient Network Widths.} Theorem~\ref{theorem:minimax-nonparametric-regression-Kint} gives the sufficient condition $m\gsim d^{5/2}n^{25\alpha(s+2)/(2\alpha(s+2)+1)}$, whose sample-size exponent is less than $25/2$. By comparison, \citet[Theorem~3.11]{SuhKH22-overparameterized-gd-minimax} uses $m/\log^3m\gtrsim L^{20}n^{24}$. The proof of \citet[Proposition~13]{Li2024-edr-general-domain} imposes $m^{-1/24}\sqrt{\log m}\lesssim n^{-1}$, equivalently $m/(\log m)^{12}\gtrsim n^{24}$, as one of its approximation requirements. The proposition states an unspecified polynomial condition that also depends on the inverse smallest empirical NTK eigenvalue. Consequently, our explicit sufficient width has a smaller sample-size exponent than these particular sufficient requirements. This comparison does not establish a necessary width separation. PGD also requires $N\gsim n^{8\alpha(s+2)/(2\alpha(s+2)+1)}\log(4n/\delta)$ auxiliary observations, so the width comparison alone does not establish a reduction in total computational cost.

\section{Discussion about Extension to General Multi-Layer Neural Networks}
This section gives a conditional roadmap for extending PGD to general multi-layer neural networks. The spectral construction parallels the two-layer argument, but a quantitative finite-width theorem for discrete deep-network PGD or stochastic PGD is not proved here. The main additional task is to establish kernel stability and accumulated approximation-error bounds along the preconditioned trajectory.

We use $f(\cdot,\btheta)$ to denote the neural network function and $\btheta$ is the collection of all the $M$ parameters in the neural network, and we use $\btheta_t$ to denote the parameters of the neural network at a continuous time  point $t \ge 0$. $\set{\bbx_i,y_i}_{i=1}^n$ are the training data. Then the gradient flow of the neural network by the vanilla GD is
\bal\label{eq:GD-theta-gflow}
\frac{\diff \btheta_t}{\diff t} =-\nabla_{\btheta} L(\btheta_t) = -\frac{1}{n} \nabla_{\theta} f(\bX,\btheta_t) \pth{f(\bX,\btheta_t)-\by}, \quad t \ge 0,
\eal
where $f(\bX,\btheta_t) = \bth{f(\bbx_1,\btheta_t),\ldots, f(\bbx_n,\btheta_t)} \in \RR^n$, $\by = \bth{y_1,\ldots,y_n} \in \RR^n$, and $\nabla_{\theta} f(\bX,\btheta_t) \in \RR^{M \times n}$.
It follows that the gradient flow of the neural network function $f(\cdot,\btheta_t)$ is
\bal\label{eq:GD-dnn-gflow}
\frac{\diff f(\cdot,\btheta_t)}{\diff t} =\nabla^{\top}_{\btheta} f(\cdot,\btheta_t) \frac{\diff \btheta_t}{\diff t} &=
-\frac{1}{n} \nabla^{\top}_{\btheta} f(\cdot,\btheta_t) \nabla_{\theta} f(\bX,\btheta_t)  \pth{f(\bX,\btheta_t)-\by} \nonumber \\
 &=-\frac{1}{n} K_t(\cdot,\bX) \pth{f(\bX,\btheta_t)-\by},
\eal
where $K_t$ is the empirical kernel defined by
\bals
K_t(\bx,\bx') \defeq \nabla^{\top}_{\btheta} f(\bx,\btheta_t) \nabla_{\theta} f(\bx',\btheta_t), \forall \bx,\bx' \in \cX.
\eals
Extensive existing works  show the uniform convergence of $K_t$ to the NTK, denoted as $K^{\textup{NT}}$, of the corresponding multi-layer neural network, such as \citet[Proposition 45]{Li2024-edr-general-domain}.

Using PGD instead of the vanilla GD in (\ref{eq:GD-theta-gflow}), we have the gradient flow for PGD as follows,
\bal\label{eq:PGD-theta-gflow}
\frac{\diff \btheta^{\textup{(PGD)}}_t}{\diff t} = -\frac{1}{n} \bM \nabla_{\theta} f(\bX,\btheta^{\textup{(PGD)}}_t) \pth{f(\bX,\btheta^{\textup{(PGD)}}_t)-\by}, \quad t \ge 0
\eal
where $\btheta^{\textup{(PGD)}}_t$ denotes the PGD parameters and $\bM\in\RR^{M\times M}$ is fixed at initialization. Let $\bQ=\set{\bbq_j}_{j=1}^N$ be an independent auxiliary sample from the covariate distribution, and set $J_0(\bQ)=\nabla_\theta f(\bQ,\btheta_0)\in\RR^{M\times N}$. In the PGD formulas below, use the mixed kernel $K_t(\bx,\bq)=\nabla_\theta f(\bx,\btheta_t^{\textup{(PGD)}})^\top\nabla_\theta f(\bq,\btheta_0)$ and its matrix $\bK_t(\bX,\bQ)\in\RR^{n\times N}$. If the limiting NTK has Mercer expansion $K^{(\textup{NT})}(\bx,\bx\prime)=\sum_{j\ge1}\lambda_j e_j(\bx)e_j(\bx\prime)$, define
\bals
K^{(\textup{NT},s)}(\bx,\bx\prime)=\sum_{j\ge1}\lambda_j^s e_j(\bx)e_j(\bx\prime),\quad \forall\bx,\bx\prime\in\cX,
\eals
with $s \ge 1$. We also define $\bK^{(\textup{NT},s)} \in \RR^{N \times N}$ as the normalized kernel gram matrix of $K^{(\textup{NT},s)}$ specified by $\bK^{(\textup{NT},s)}_{ij} = K^{(\textup{NT},s)}(\bbq_i,\bbq_j)/N$ for all $i,j \in [N]$.
Then the preconditioner $\bM$, similar to Algorithm~\ref{alg:preconditioner}, is computed by
\bals
\bM=\frac1N J_0(\bQ)\bK^{(\textup{NT},s)}J_0(\bQ)^\top.
\eals
As a result, the gradient flow of the neural network trained by PGD is

\bal\label{eq:PGD-dnn-gflow}
\frac{\diff f^{\textup{}}(\cdot,\btheta^{\textup{(PGD)}}_t)}{\diff t} &= \nabla^{\top}_{\btheta}  f^{\textup{}}(\cdot,\btheta^{\textup{(PGD)}}_t) \frac{\diff \btheta^{\textup{(PGD)}}_t}{\diff t} \nonumber \\
&=-\frac{1}{nN}K_t(\cdot,\bQ)\bK^{(\textup{NT},s)}\bK_t(\bX,\bQ)^\top\pth{f(\bX,\btheta_t^{\textup{(PGD)}})-\by},
\eal
If the mixed kernels converge uniformly to $K^{(\textup{NT})}$ along the PGD trajectory, and the auxiliary approximation is uniformly controlled, then the effective kernel $N^{-1}K_t(\cdot,\bQ)\bK^{(\textup{NT},s)} \newline \bK_t(\cdot,\bQ)^\top$ approximates the integral kernel
\bal
K^{(\textup{NT,int})}(\bx,\bx')\defeq \int_{\cX \times \cX} K^{(\textup{NT})}(\bx,\bv)K^{(\textup{NT},s)}(\bv,\bv')K^{(\textup{NT})}(\bv',\bx')  \diff \mu(\bv) \otimes \mu(\bv'), \,\forall \bx,\bx' \in \cX.
\eal
Writing the exact flow as the reference integral-kernel flow plus its discrepancy gives
\bal\label{eq:PGD-dnn-gflow-Kint}
\frac{\diff f^{\textup{}}(\cdot,\btheta^{\textup{(PGD)}}_t)}{\diff t} =-\frac{1}{n} K^{(\textup{NT,int})}(\cdot,\bX) \pth{f(\bX,\btheta_t^{\textup{(PGD)}})-\by} + E_t,
\eal
where $E_t=-n^{-1}(\widehat K_t^{(\textup{NT,int})}(\cdot,\bX)-K^{(\textup{NT,int})}(\cdot,\bX))(f(\bX,\btheta_t^{\textup{(PGD)}})-\by)$ and $\widehat K_t^{(\textup{NT,int})}=N^{-1}K_t(\cdot,\bQ)\bK^{(\textup{NT},s)}\bK_t(\cdot,\bQ)^\top$. If the kernel discrepancy is at most $\varepsilon_t$ uniformly, then $\|E_t\|_\infty\le\varepsilon_t\|f(\bX,\btheta_t^{\textup{(PGD)}})-\by\|_2/\sqrt n$. Controlling this error up to the stopping time requires a trajectory-wise stability theorem for PGD.

If the deep NTK has positive eigenvalues $\lambda_j\asymp j^{-2\alpha}$ and the target satisfies $f^*\in\bth{\cH_{K^{(\textup{NT})}}}^{s+2}(\mu_0)$, the corresponding integral kernel has eigenvalues $\lambda_j^{s+2}$. This identifies the candidate risk rate $n^{-2\alpha(s+2)/(2\alpha(s+2)+1)}$. Obtaining this rate for a finite deep network additionally requires the preceding PGD stability bounds, a function decomposition, and statistical error control. The value $2\alpha=d/(d-1)$ must be verified for the specific deep NTK. The two-layer covariate corollary alone does not establish it.

\noindent\textbf{Relation to finite-width deep GD and SGD.} \citet{zhou2026optimalratesgeneralizationgradient} provide uniform NTK approximation and polynomial-width analysis for deep ReLU networks trained by discrete GD and SGD. In their source and capacity notation, $\beta=(s+2)/2$ and $\gamma=1/(2\alpha)$ match the displayed exponents: their Corollary~6 bounds the excess risk by $\cO(n^{-2\beta/(2\beta+\gamma)}\log^4(n/\delta))$, under its stated sample-size, width, and model assumptions. This exponent comparison does not identify their deep NTK with our two-layer kernel. Extending their tools to PGD requires simultaneous control of (1) layerwise movement under preconditioned updates, (2) mixed-kernel stability on the training and auxiliary samples, (3) auxiliary-sample approximation, and (4) discrete-time error accumulation. Stochastic PGD also requires a martingale analysis of its stochastic updates. A quantitative theorem with explicit dependence on depth, width, auxiliary-sample size, and stopping time is left for future work.

\section{Simulation Results}
\label{sec:simulation}

We present simulation results in this section. The simulations are designed to validate the theoretical results in this paper: (1) improved generalization of PGD over vanilla GD with spectrally biased target functions, and (2) the scaling behavior of the early stopping time. We use spherical uniform inputs to match the theoretical setting, and we present results for linear target function and quadratic target function.

\subsection{Linear Target Function}
\label{sec:simulation-linear}

We randomly sample $n$ points $\set{\bbx_i}_{i=1}^n$ distributed uniformly on the unit sphere $\unitsphere{49}$ in $\RR^{50}$. The sample size $n$ ranges from $100$ to $1000$ with a step size of $100$.
We set the target function to $f^*(\bx) = \bs^{\top} \bx$, where $\bx \in \unitsphere{49}$ and $\bs \sim \Unif{\cX}$ is randomly sampled. The variance of the noise is set to $\sigma_0^2 = 1$. We also uniformly and independently sample $1000$ points on the unit sphere in $\RR^{50}$ to serve as the test set.
We train the two-layer NN (\ref{eq:two-layer-nn}) by Algorithm~\ref{alg:PGD} with the network width $m = 10000$, $N = 8000$, and the learning rate is set to $\eta = 1$. We also set $s = 1$, so that $K^{(s)} = K$. The training is performed on an NVIDIA \bA100 GPU, and we report the test loss in
Figure~\ref{fig:test-loss-GD-PGD} and Figure~\ref{fig:simulation}. It can be observed from Figure~\ref{fig:test-loss-GD-PGD} that PGD always demonstrates better generalization than the vanilla GD through lower test losses across different training data sizes. Figure~\ref{fig:simulation} illustrates the test loss with respect to different numbers of PGD steps for $n = 100, 300, 500, 1000$.
It can be observed from Figure~\ref{fig:simulation} that early stopping consistently improves generalization for neural network training, as the test loss initially decreases and then increases due to overfitting.
%\newpage
%\vspace{-0.2in}

\begin{figure}[t!]
\begin{center}
\includegraphics[width=5.5in]{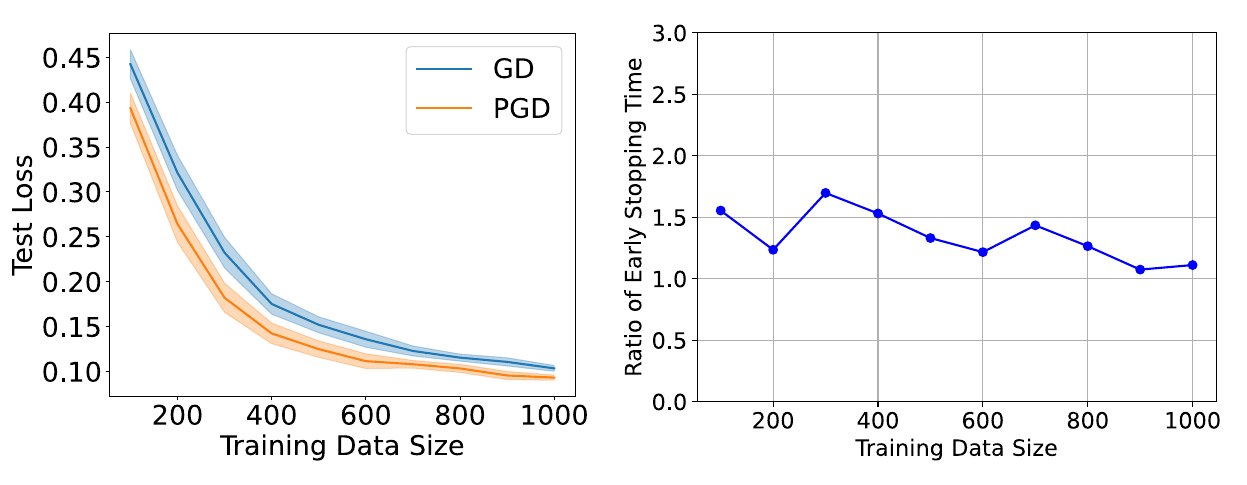}
\end{center}
\vspace{-0.4in}
\caption{Left: illustration of the test losses by GD and PGD for varying $n$ in $[100,1000]$ with a step size of $100$ for a linear target function. The shaded area in each plot indicates the standard deviation across $10$ random initializations of the neural network.
Right: illustration of the ratio of early stopping time.}
\label{fig:test-loss-GD-PGD}\vspace{-0.2in}
\end{figure}

\begin{figure}[t!]
\begin{center}
\includegraphics[width=5in]{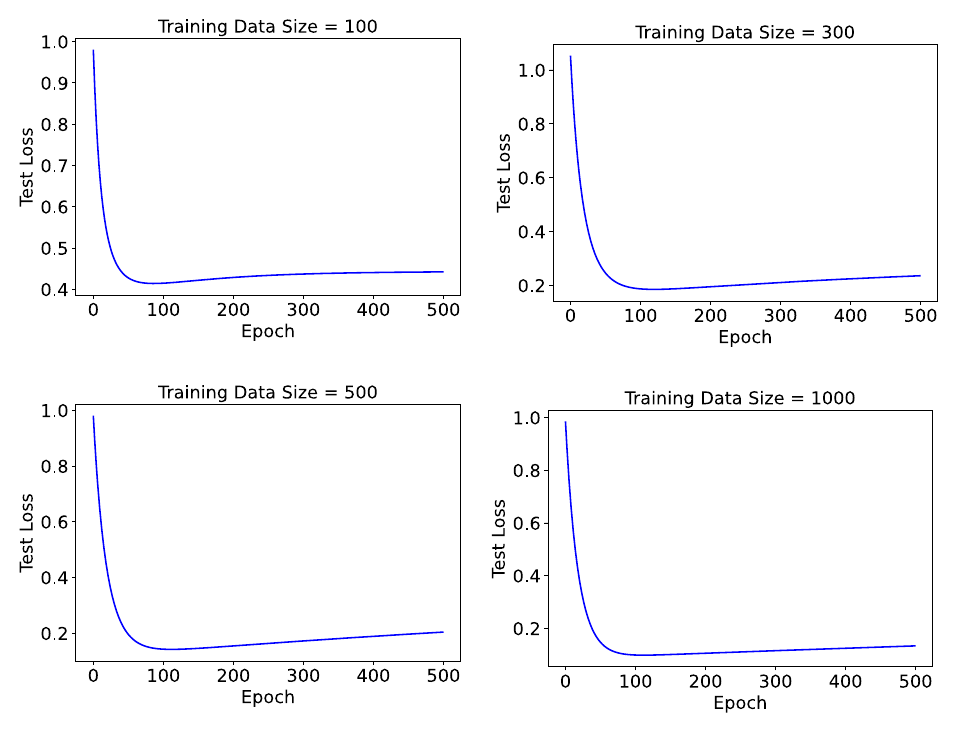}
\end{center}
\vspace{-0.4in}
\caption{Illustration of the test losses by PGD for a linear target function, which are averaged over $10$ random initializations of the neural network.}
\label{fig:simulation}\vspace{-0.25in}
\end{figure}

For each $n \in \set{100, 200, \ldots, 1000}$, we identify the PGD step that achieves the minimum test loss, denoted as $\hat{t}_n$, which serves as the empirical early stopping time. Theoretically, the early stopping time is predicted to scale as $\hat T \asymp n^{\frac{2\alpha(s+2)}{2\alpha(s+2)+1}} \asymp n^{\frac{3d}{4d-1}}$ with $s = 1$ and $2\alpha = d/(d-1)$. We compute the ratio of early stopping time, $\hat{t}_n /n^{\frac{3d}{4d-1}}$, averaged over $10$ random initializations of the neural network for each $n$ and illustrated in the right plot of Figure~\ref{fig:test-loss-GD-PGD}. We observe that the ratio of early stopping time is relatively stable and it lies within the interval $[1.0736, 1.6972]$, suggesting that the theoretically predicted early stopping time is empirically proportional to the empirical early stopping time.

\subsection{Quadratic Target Function}
\label{sec:simulation-quadratic}

We present additional simulation results under the same experimental setup as that in Section~\ref{sec:simulation-linear} except for the target function, where the target function is a quadratic function specified by $f^*(\bx) = (\bs^{\top} \bx)^2$ with $\bs \sim \Unif{\cX}$. The purpose of this study is to demonstrate the potential of PGD in learning more complex target functions than simple linear functions. Figure~\ref{fig:test-loss-GD-PGD-quadratic} Illustrates the test loss by GD and PGD for varying $n$ in $[100,1000]$ with a step size of $100$.
Figure~\ref{fig:simulation-quadratic} illustrates the test loss with respect to different numbers of PGD steps for $n = 100, 300, 500, 1000$.
Similar to the case with the linear target function, it can be observed from Figure~\ref{fig:test-loss-GD-PGD-quadratic} and Figure~\ref{fig:simulation-quadratic} that PGD with early stopping consistently improves generalization for neural network training compared with the vanilla GD, and the ratio of early stopping times is stable and lies within the interval $[1.4560, 3.1001]$.

\begin{figure}[h!]
\begin{center}
\includegraphics[width=5.5in]{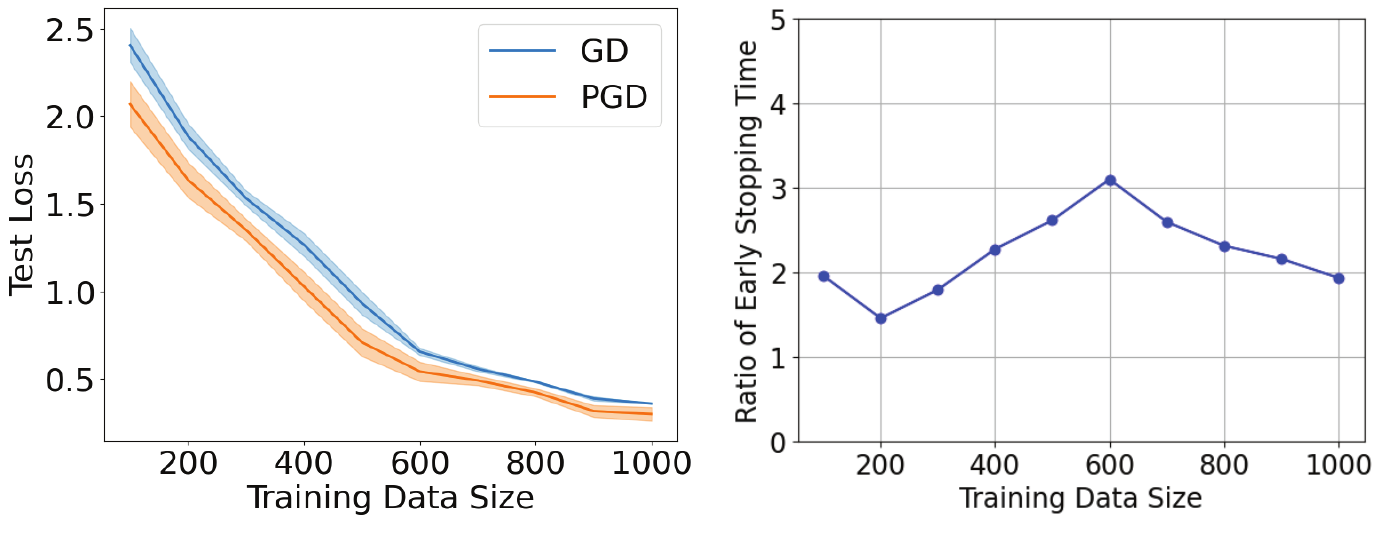}
\end{center}
\vspace{-0.3in}
\caption{Left: illustration of the test losses by GD and PGD for varying $n$ in $[100,1000]$ with a step size of $100$ for a quadratic target function. The shaded area in each plot indicates the standard deviation across $10$ random initializations of the neural network.
Right: illustration of the ratio of early stopping time.}
\label{fig:test-loss-GD-PGD-quadratic}%\vspace{-0.2in}
\end{figure}

\begin{figure}[h!]
\begin{center}
\includegraphics[width=5in]{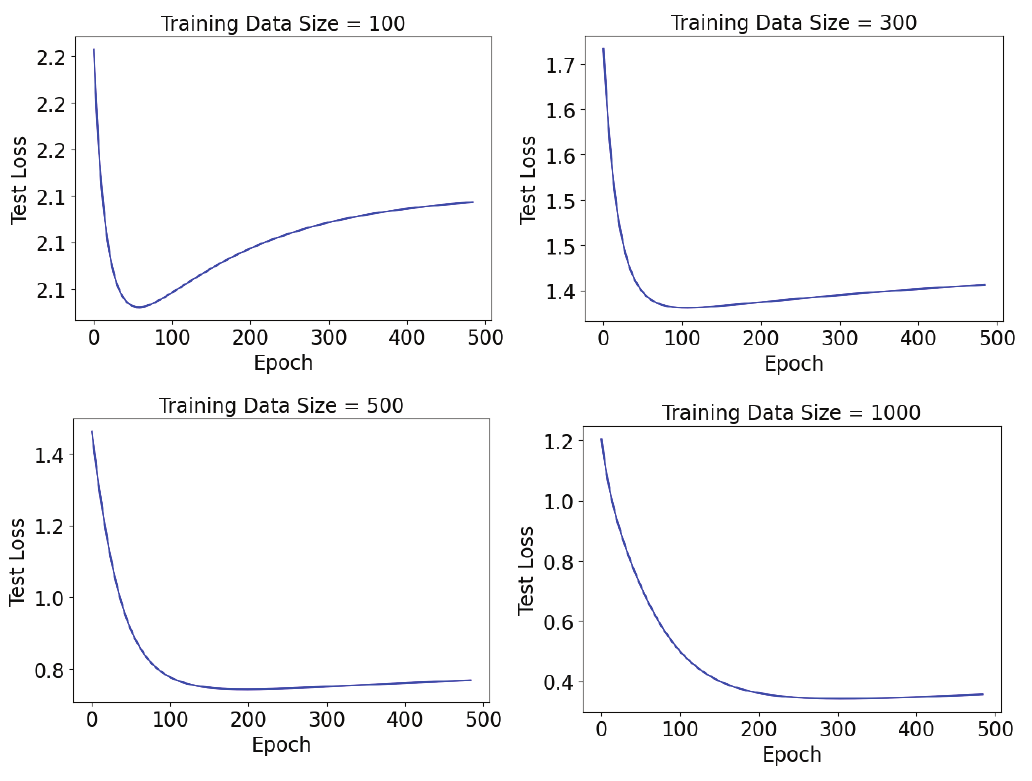}
\end{center}
\vspace{-0.4in}
\caption{Illustration of the test losses by PGD for a quadratic target function, which are averaged over $10$ random initializations of the neural network.}
\label{fig:simulation-quadratic}%\vspace{-0.25in}
\end{figure}

\subsection{Memory-Efficient PGD}
\label{sec:memory-efficient-PGD}

\begin{algorithm}[h]
\renewcommand{\algorithmicrequire}{\textbf{Input:}}
\renewcommand\algorithmicensure {\textbf{Output:} }
\caption{Memory-Efficient PGD}
\label{alg:memory-efficient-PGD}
\begin{algorithmic}[1]
\State $\bW_{\bS}(t+1)\leftarrow$ Memory-Efficient-PGD($t$)
\State \textbf{\bf input: } $t \in [T]$
%\State \textbf{\bf for } $b=1,\ldots,B$ \,\,\textbf{\bf do }
\For{$b=1,\ldots,B$}
\State Update the $b$-th batch of the neurons by
\bal\label{eq:memoery-efficient-PGD}
\bth{\vect{\bW_{\bS}(t+1)} - \vect{\bW_{\bS}(t)}}_{[(b-1)Sd+1:bSd]}
=-\frac{\eta}n \bM_{[(b-1)Sd+1:bSd]}\bZ_{\bS}(t)(\hat \by(t) -  \by)
\eal
%\State \textbf{\bf end for }
\EndFor
\State \textbf{\bf return} $\bW_{\bS}(t+1)$
\end{algorithmic}
\end{algorithm}

We now describe a memory-efficient implementation of the PGD algorithm described in Algorithm~\ref{alg:PGD}. The most computationally expensive step of
Algorithm~\ref{alg:PGD} is the computation of $\bG(t) \defeq \bM \bZ_{\bS}(t)(\hat \by(t) -  \by) \in \RR^{md}$ in each PGD iteration (\ref{eq:PGD-two-layer-nn}). %With such $\bG(t)$, (\ref{eq:PGD-two-layer-nn}) is reduced to
%$\vect{\bW_{\bS}(t+1)} - \vect{\bW_{\bS}(t)} = -{\eta}/{n} \cdot \bG(t)(\hat \by(t) -  \by)$.
A straightforward computation of $\bG(t)$ requires that a large-scale matrix $\bM \in {\RR^{md \times md}}$ be stored in memory, resulting in a potentially prohibitive memory requirement for PGD with large network width $m$. To address this computational challenge, we propose a memory efficient PGD algorithm described in
Algorithm~\ref{alg:memory-efficient-PGD} for each iteration $t \in [T]$ of Algorithm~\ref{alg:PGD}. Instead of updating all the
$m$ neurons at one step, Algorithm~\ref{alg:memory-efficient-PGD} splits all the $m$ neurons into $B = m/S$ batches where $S$ is the batch size and $m$ is divisible by $S$, and then updates the neurons in the $b$-th batch by (\ref{eq:memoery-efficient-PGD}) for every $b \in [B]$. In every batch update, only a submatrix of $\bZ_{\bQ}(0)$, $\bZ_{\bQ}(0)_{[(b-1)Sd+1:bSd]}$ formed by rows of $\bZ_{\bQ}(0)$ with indices in $[(b-1)Sd+1:bSd]$,  is in memory, and such memory is released before the execution of the next batch update. Such batch updates of all the neurons significantly reduce the memory need. In particular, for each iteration $t \in [T]$, Algorithm~\ref{alg:memory-efficient-PGD} has the same time complexity, $\Theta(mdN)$, under $N \ge n$ and $N \ll md$, while its memory complexity is only $\Theta(SdN + N^2)$, in contrast to the significantly higher memory complexity of $\Theta(mdN)$ of a straightforward implementation of Algorithm~\ref{alg:PGD} with $m \gg S$. We note that the same idea of using neuron batches is also used in the computation of $\bZ_{\bQ}(0)^{\top}\bZ_{\bS}(t)$ in (\ref{eq:memoery-efficient-PGD}). On the NVIDIA A100 GPU, our memory-efficient PGD algorithm only needs $1.12$ GB memory with a training speed of $9.2$ seconds for each PGD iteration for the simulation results in Section~\ref{sec:simulation-quadratic}. We note that the memory cost and the training speed are the same for simulation results in Section~\ref{sec:simulation-linear} and Section~\ref{sec:simulation-quadratic}. It is also worthwhile to mention that, although each iteration of our PGD algorithm requires a higher time complexity, $\Theta(mdN)$, than that of the vanilla GD, $\Theta(mdn)$, this increased computational cost yields sharper and optimal regression rate, as discussed in 
Section~\ref{sec:summary-main-results} and Section~\ref{sec:main-results}.

%gradient descent (GD) described in Algorithm~\ref{alg:GD}
%\begin{algorithm}[!htbp]
%        \renewcommand{\algorithmicrequire}{\textbf{Input:}}
%\renewcommand\algorithmicensure {\textbf{Output:} }
%\caption{Training the Two-Layer NN by GD}
%\label{alg:GD}
%\begin{algorithmic}[1]
%\State $\bW(T) \leftarrow$ Training-by-GD($T,\bW(0)$)
%\State \textbf{\bf input: } $T,\bW(0)$
%\State \textbf{\bf for } $t=1,\ldots,T$ \,\,\textbf{\bf do }
%\State \quad Perform the $t$-th step of GD by
%\bals
%\vect{\bW(t+1)} - \vect{\bW(t)} = - \frac{\eta}{n} \bZ_{\bS}(t)
%(\hat \by(t) -  \by)
%\eals
%\State \textbf{\bf end for }
%\State \textbf{\bf return} $\bW(T)$
%\end{algorithmic}
%\end{algorithm}

%The code for the simulation results of this section is available at \url{https://anonymous.4open.science/r/PGD-two-layer-M-FD12/}.

%\vspace{-.1in}
\section{Conclusion}
%\vspace{-.05in}
We study nonparametric regression using an over-parameterized two-layer neural network when the target function lies in an interpolation space exhibiting spectral bias, that is, $\bth{\cH_K}^{s'}(\mu_0)$ where $s'\ge 3$ is the smoothness parameter and $K$ is the Neural Tangent Kernel (NTK) of the network. We show that, when the network is trained using a novel and carefully designed Preconditioned Gradient Descent (PGD) algorithm with early stopping, the resulting estimator achieves a sharper and minimax optimal nonparametric regression risk rate of $\cO(n^{-\frac{d s'}{d s'+d-1}})$. This result holds under general covariate distributions supported either on the unit sphere $\unitsphere{d-1}$ or the entire $\RR^d$ satisfying certain mild conditions, extending beyond the commonly assumed spherical uniform distribution. Our result improves upon current state-of-the-art risk bounds achieved by training neural networks for nonparametric regression. We also provide a detailed comparison with prior works and outline the key technical ideas underlying our analysis, highlighting how the proposed PGD training dynamics approximate kernel regression with the integral kernel $\Kint$, rather than the regular NTK $K$, and achieve the minimax optimal risk rate over $\cH_{\Kint}(\mu_0)$.

% Acknowledgements and Disclosure of Funding should go at the end, before appendices and references

%\acks{All acknowledgements go at the end of the paper before appendices and references.
%Moreover, you are required to declare funding (financial activities supporting the
%submitted work) and competing interests (related financial activities outside the submitted work).
%More information about this disclosure can be found on the JMLR website.}

% Manual newpage inserted to improve layout of sample file - not
% needed in general before appendices/bibliography.

\newpage
\appendix

We present the basic mathematical results required in our proofs in Section~\ref{sec::math-tools}, then present proofs in the subsequent sections.

\section{Mathematical Tools}
\label{sec::math-tools}
We introduce the basic definitions and mathematical results as the
basic tools for the subsequent results in the next sections of this appendix.
\begin{definition}\label{def:RC}
Let $\{\sigma_i\}_{i=1}^n$ be $n$ i.i.d. random variables such that $\Pr[\sigma_i = 1] = \Pr[\sigma_i = -1] = \frac{1}{2}$. The Rademacher complexity of a function class $\cF$ is defined as
\bal\label{eq:RC}
&\cfrakR(\cF) = \Expect{\set{\bbx_i}_{i=1}^n, \set{\sigma_i}_{i=1}^n}{\sup_{f \in \cF} {\frac{1}{n} \sum\limits_{i=1}^n {\sigma_i}{f(\bbx_i)}} }.
\eal%
The empirical Rademacher complexity is defined as
\bal\label{eq:empirical-RC}
&\hat \cfrakR(\cF) = \Expect{\set{\sigma_i}_{i=1}^n} { \sup_{f \in \cF} {\frac{1}{n} \sum\limits_{i=1}^n {\sigma_i}{f(\bbx_i)}} },
\eal%
For simplicity of notation, Rademacher complexity and empirical Rademacher complexity are also denoted by $\Expect{}{\sup_{f \in \cF} {\frac{1}{n} \sum\limits_{i=1}^n {\sigma_i}{f(\bbx_i)}} }$ and $\Expect{\sigma}{\sup_{f \in \cF} {\frac{1}{n} \sum\limits_{i=1}^n {\sigma_i}{f(\bbx_i)}} }$ respectively. %We also denote the Rademacher complexity by $\Expect{\bS,\set{\sigma_i}_{i=1}^n}{}$ to emphasize the variables with respect to which the expectation is computed.
\end{definition}

For data $\set{\bbx}_{i=1}^n$ and a function class $\cF$, we define the notation $R_n \cF$ by $R_n \cF \coloneqq \sup_{f \in \cF} \frac{1}{n} \sum\limits_{i=1}^n \sigma_i f(\bbx_i)$.

\begin{theorem}[{\citet[Theorem 2.1]{bartlett2005}}]
\label{theorem:Talagrand-inequality}
Let $\cX,P$ be a probability space, $\set{\bbx_i}_{i=1}^n$ be independent random variables distributed according to $P$. Let $\cF$ be a class of functions that map $\cX$ into $[a, b]$. Assume that there is some $r > 0$
such that for every $f \in \cF$,$\Var{f(\bbx_i)} \le r$. Then, for every $x > 0$, with probability at least $1-e^{-x}$,
\bal\label{eq:Talagrand-inequality}
&\sup_{f \in \cF} \big( \E_{P} [f(\bx)] - \E_{\bx \sim P_n } [f(\bx)] \big) \le \inf_{\alpha > 0} \Bigg( 2(1+\alpha) \E_{\set{\bbx_i}_{i=1}^n,\set{\sigma_i}_{i=1}^n} [R_n \cF] + \sqrt{\frac{2rx}{n}} \nonumber \\
&\phantom{\quad \quad}+ (b-a) \pth{\frac{1}{3}+\frac{1}{\alpha}} \frac{x}{n} \Bigg),
\eal%
and with probability at least $1-2e^{-x}$,
\bal\label{eq:Talagrand-inequality-empirical}
\sup_{f \in \cF} \big( \E_{P} [f(\bx)] - \E_{\bx \sim P_n } [f(\bx)] \big)
&\le \inf_{\alpha \in (0,1)} \Bigg(\frac{2(1+\alpha)}{1-\alpha} \E_{\set{\sigma_i}_{i=1}^n} [R_n \cF] + \sqrt{\frac{2rx}{n}}
\nonumber \\
&\phantom{\quad \quad} + (b-a) \pth{ \frac{1}{3}+\frac{1}{\alpha} + \frac{1+\alpha}{2\alpha(1-\alpha)} } \frac{x}{n} \Bigg).
\eal%
$P_n$ is the empirical distribution over $\set{\bbx_i}_{i=1}^n$ with
$\Expect{\bx \sim P_n}{f(\bx)} = \frac{1}{n} \sum\limits_{i=1}^n f(\bbx_i)$. Moreover, the same results hold for $\sup_{f \in \cF} \big( \E_{\bx \sim P_n } [f(\bx)] - \E_{P} [f(\bx)] \big)$.
\end{theorem}

In addition, we have the contraction property for Rademacher complexity, which is due to Ledoux
and Talagrand~\citep{Ledoux-Talagrand-Probability-Banach}.

\begin{theorem}\label{theorem:RC-contraction}
Let $\phi$ be a contraction,that is, $\abth{\phi(x) - \phi(y)} \le \mu \abth{x-y}$ for $\mu > 0$. Then, for every function class $\cF$,
\bal\label{eq:RC-contraction}
&\Expect{\set{\sigma_i}_{i=1}^n} {R_n \phi \circ \cF} \le \mu \Expect{\set{\sigma_i}_{i=1}^n} {R_n \cF},
\eal%
where $\phi \circ \cF$ is the function class defined by $\phi \circ \cF = \set{\phi \circ f \colon f \in \cF}$.
\end{theorem}

\begin{definition}[Sub-root function, {\citep[Definition 3.1]{bartlett2005}}]
\label{def:sub-root-function}
\begin{sloppypar}
A function $\psi \colon [0,\infty) \to [0,\infty)$ is sub-root if it is nonnegative,
nondecreasing and if $\frac{\psi(r)}{\sqrt r}$ is nonincreasing for $r >0$.
\end{sloppypar}
\end{definition}

\begin{theorem}[{\citet[Theorem 3.3]{bartlett2005}}]
\label{theorem:LRC-population}
Let $\cF$ be a class of functions with ranges in $[a, b]$ and assume
that there are some functional $T \colon \cF \to \RR+$ and some constant $\bar B$ such that for every $f \in \cF$ , $\Var{f} \le T(f) \le \bar B P(f)$. Let $\psi$ be a sub-root function and let $r^*$ be the fixed point of $\psi$.
Assume that $\psi$ satisfies, for any $r \ge r^*$,
$\psi(r) \ge \bar B \cfrakR(\set{f \in \cF \colon T (f) \le r})$. Fix $x > 0$.
Then for any $K_0 > 1$, with probability at least $1-e^{-x}$,
\bals
\forall f \in \cF, \quad \Expect{P}{f} \le \frac{K_0}{K_0-1} \Expect{P_n}{f} + \frac{704K_0}{\bar B} r^*
+ \frac{x\pth{11(b-a)+26 \bar B K_0}}{n}.
\eals
Also, with probability at least $1-e^{-x}$,
\bals
\forall f \in \cF, \quad \Expect{P_n}{f} \le \frac{K_0+1}{K_0} \Expect{P}{f}  + \frac{704K_0}{\bar B} r^*
+ \frac{x\pth{11(b-a)+26 \bar B K_0}}{n}.
\eals
\end{theorem}

\section{Detailed Proofs}
\label{sec:proofs}

\begin{figure}[!htbp]
\begin{center}
\includegraphics[width=0.7\textwidth]{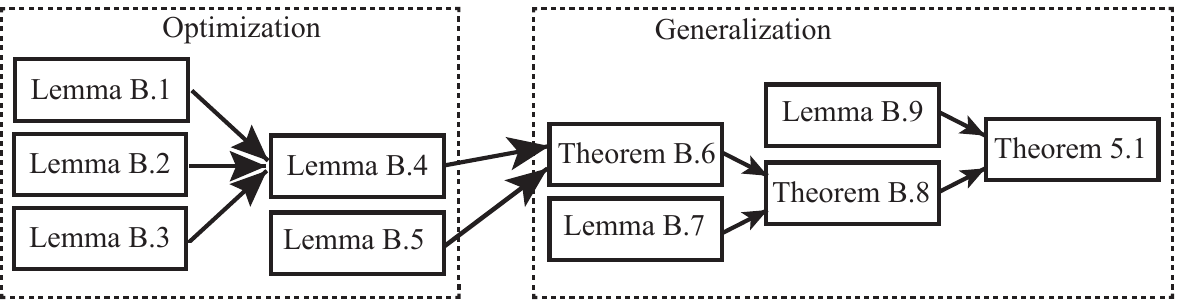}
\end{center}
\caption{Roadmap of Major Results Leading to Theorem~\ref{theorem:minimax-nonparametric-regression-Kint}.}
\label{fig:proof-roadmap-Kint}
\end{figure}

%\section{Proof of Theorem~\ref{theorem:minimax-nonparametric-regression-Kint}}
%\label{sec:detailed-proof-minimax-nonparametric-regression-Kint}

%\subsection{Proof of Theorem~\ref{theorem:LRC-population-NN}}

Because Theorem~\ref{theorem:minimax-nonparametric-regression-Kint}
is proved by Theorem~\ref{theorem:LRC-population-NN} and Lemma~\ref{lemma:empirical-loss-bound}, we illustrate in Figure~\ref{fig:proof-roadmap-Kint} the roadmap containing the intermediate theoretical results which lead to Theorem~\ref{theorem:LRC-population-NN}.
Theorem~\ref{theorem:bounded-NN-class} is the formal version of
Theorem~\ref{theorem:bounded-NN-class-informal},
Theorem~\ref{theorem:LRC-population-NN} is the formal version of
Theorem~\ref{theorem:LRC-population-NN-eigenvalue-informal}, and
Lemma~\ref{lemma:empirical-loss-bound} is the formal version of
Lemma~\ref{lemma:empirical-loss-bound-informal}. Lemma~\ref{lemma:empirical-loss-convergence} is the formal version of
Lemma~\ref{lemma:empirical-loss-convergence-informal}.

\subsection{Theorem~\ref{theorem:LRC-population-NN}, Lemma~\ref{lemma:empirical-loss-bound}, and related theoretical results}
\label{sec:proofs-details}

\noindent\textbf{Relation to the proof framework of
\citet{Yang2025-generalization-two-layer-regression}.}
The high-level proof architecture in this subsection builds on
\citet{Yang2025-generalization-two-layer-regression}: first decompose a
finite-width neural-network iterate into an RKHS component and a
uniformly small residual, and then control the resulting localized
function class using local Rademacher complexity. Some general-purpose
arguments are therefore inherited. Most notably, Lemma~B.7 is adapted
from \citet[Lemma~C.9]{Yang2025-generalization-two-layer-regression}.
After replacing
$(K,\cH_K,\{\lambda_j\}_{j\ge1})$ by
$(\Kint,\cH_{\Kint},\{\lambdaint_j\}_{j\ge1})$, its
localized-Rademacher-complexity calculation is unchanged.

The arguments preceding Lemma~B.7 require substantial new elements and modifications.
The previous work analyzes vanilla GD, whose reference dynamics are
governed by the regular NTK $K$. In the present paper, PGD uses a random
preconditioner constructed from an independent auxiliary sample
$\bQ$, and its reference dynamics are governed by the integral kernel
$\Kint$. Consequently, the residual recursion contains not only the
finite-width error caused by variation of the empirical NTK along the
training trajectory, as in the previous work, but also a new
approximation error associated with the auxiliary sample and its
interaction with the time-dependent empirical NTK. Lemmas~B.1--B.5
modify the earlier GD analysis to control these coupled errors
uniformly up to the early-stopping time. In particular, the spectral
filters, critical radii, and stopping rule must now be analyzed using
$\bKint_n$ and the spectrum of $\Kint$, rather than $\bK_n$ and the
spectrum of $K$.

The main results specific to the present setting of this paper are the uniform
concentration of the empirical composite kernel around $\Kint$ in
Theorem~\ref{theorem:hatKint-close-to-Kint-supnorm}, the PGD function decomposition in Theorem~B.6, and the
$\Kint$-based empirical-loss and population-risk bounds in
Lemma~B.9 and Theorem~B.8. Theorem~B.6 shows that every relevant PGD
iterate can be written as a function in $\cH_{\Kint}$ with controlled
RKHS norm plus a residual with small $L^\infty$-norm. Lemma~B.7 is then
applied to this new function class. Finally, the identity
$\lambdaint_j=\lambda_j^{s+2}$ converts the eigenvalue-tail term in
Lemma~B.7 into the smaller local complexity associated with $\Kint$,
which is the key spectral input leading to the sharp rate in
Theorem~5.1. Therefore, compared to \citet{Yang2025-generalization-two-layer-regression}, the reusable component is the general localized
Rademacher-complexity calculation, whereas the PGD dynamics, the
auxiliary-sample approximation, the integral-kernel decomposition, and
the resulting sharp-rate analysis are specific to the present paper and novel.

In this subsection, we first present 
The signal-and-noise norm bound in Lemma~\ref{lemma:yt-y-bound}, followed by the neuron-movement bound in Lemma~\ref{lemma:weight-vector-movement}, the one-step perturbed PGD recursion in Lemma~\ref{lemma:empirical-loss-convergence-contraction},
and finally the inductive convergence result in Lemma~\ref{lemma:empirical-loss-convergence}.

\begin{lemma}\label{lemma:yt-y-bound}
Let $t \in [0\colon T]$, $\bv = -\pth{\bI-\eta \bKint_n}^{t} f^*(\bS)$,
 $\be = -\pth{\bI-\eta \bKint_n }^{t} \bw$, and $\eta
  \in (0,8)$.
Then with probability at least
$1 -  \exp\pth{-\Theta(n)}$
over $\bw$,
\bal\label{eq:yt-y-bound}
\ltwonorm{\bv} + \ltwonorm{\be} \le \pth{\frac{\mu_0}{\min\set{2{\sqrt 2}, \sqrt{2e\eta }} } + \sigma_0+1} {\sqrt n}.
\eal
\end{lemma}

\begin{lemma}\label{lemma:weight-vector-movement}
Suppose that $t \in [0\relcolon T-1]$ for $T \ge 1$, and $\ltwonorm{\hat \by(t') - \by} \le {\sqrt n} c_{\bu} $ holds for all $0 \le t' \le t$. Then
\bal\label{eq:R}
\ltwonorm{\bbw_{\bS,r}(t') - \bbw_r(0)} \le R, \quad \forall\, 0 \le t' \le t+1.
\eal
\end{lemma}

\begin{lemma}\label{lemma:empirical-loss-convergence-contraction}
Let $\eta=\Theta(1)\in(0,8)$, $0 \le t \le T-1$ for $T \ge 1$, and suppose that $\ltwonorm{\hat \by(t') - \by} \le
 c_{\bu}{\sqrt{n}}  $ holds for all $0 \le t' \le t$ and
 the random initialization $\bW(0) \in \cW_0$. Then
\bal\label{eq:empirical-loss-convergence-contraction}
\hat \by(t+1) - \by  &= \pth{\bI- \eta \bKint_n }\pth{\hat \by(t) - \by} +\bEint(t)\pth{\hat \by(t) - \by}+\bE(t+1),
\eal
where
\bal\label{eq:empirical-loss-Et-bound}
&\ltwonorm {\bE(t+1)} \lsim \bE_{m,\eta,\delta},
\eal
and $\bE_{m,\eta,\delta}$ is defined by
\bal\label{eq:empirical-loss-Et-bound-Em}
\bE_{m,\eta,\delta} &\defeq \eta c_{\bu} {\sqrt n}
\pth{ 2 \pth{\frac{2R}{\sqrt {2\pi} \kappa}+
C_2(m/2,d,1/n)} +C_1(m/2,d,1/n)}
\nonumber \\
&\lsim {\sqrt {dn}} m^{-\frac 15} T^{\frac 12}
\eal
for $\delta \in (0,1)$.
In addition, for every $\delta \in (0,1)$,
with probability at least
$1-\delta$ over $\bQ$,
\bal\label{eq:empirical-loss-EKint-bound}
\ltwonorm{\bEint(t)} \lsim \eta \sqrt{\frac{\log (4n/{\delta})}{N}}.
\eal
\end{lemma}

\begin{lemma}\label{lemma:empirical-loss-convergence}
Suppose $N \gsim T^2\log{(4n/{\delta})}/\tau^2$,
and
\bal\label{eq:m-cond-empirical-loss-convergence}
m \gsim T^{\frac {15}{2}} d^{\frac 52}/{\tau^5}.
\eal
Suppose the neural network $f(\bW(t),\cdot)$ is trained by PGD  using Algorithm~\ref{alg:PGD}
with the learning rate $\eta = \Theta(1) \in (0,8)$ on the random initialization $\bW(0) \in \cW_0$. Then for every $\delta \in (0,1)$, with probability at least
$1 -  \exp\pth{-\Theta(n)}-\delta$
over $\bw$ and $\bQ$,
$\bW(t) \in \cW(\bS,\bQ,\bW(0),T)$.
Moreover, for all $t \in [0,T]$, $\bu(t) = \bv(t) + \be(t)$ where
$\bu(t) = \hat \by(t) -  \by$, $\bv(t) \in \cV_t$, $\be(t) \in \cE_{t,\tau}$,
and $\ltwonorm{\bu(t)} \le c_{\bu} \sqrt n$.
\end{lemma}

\begin{proof}
{\textbf{\textup{\hspace{-3pt}of
Lemma~\ref{lemma:empirical-loss-convergence}}}.}
First,  when $m \gsim T^{\frac {15}{2}} d^{\frac 52}/{\tau^5}$ with a proper constant, it can be verified that $\bE_{m,\eta,\delta} \le {\tau {\sqrt n}}/(2T)$ where $\bE_{m,\eta,\delta}$ is defined by
(\ref{eq:empirical-loss-Et-bound-Em}) of
Lemma~\ref{lemma:empirical-loss-convergence-contraction}.
We then use mathematical induction to prove the lemma.
%We will first prove that $\bu(t) = \bv(t) + \be(t)$ where $\bv(t) \in \cV_t$,
%$\be(t) \in \cE_{t,\tau}$, and $\ltwonorm{\bu(t)} \le c_{\bu} \sqrt n$ for %for all $t \in [0,T]$.
%for all $t \in [0,T]$, where $\sum_{t'=1}^{t} \cdot = 0$ for $t < 1$, and

When $t = 0$, we have
\bal\label{eq:empirical-loss-convergence-seg1}
\bu(0) = - \by &= \bv(0) + \be(0),
\eal
where $\bv(0) \defeq -f^*(\bS) = -\pth{\bI-\eta \bKint_n}^0 f^*(\bS)$,
$\be(0) = -\bw = \bbe_1(0) + \bbe_2(0)$ with
$\bbe_1(0) = -\big(\bI-\eta \bKint_n \big) ^0 \bw$
and $\bbe_2(0) = \bzero$. Therefore,
$\bv(0) \in \cV_{0}$ and $\be(0) \in \cE_{0,\tau}$.
Also, it follows from the proof of Lemma~\ref{lemma:yt-y-bound}
that $\ltwonorm{\bu(0)} \le c_{\bu}\sqrt n$ w.p. at least
$1 -  \exp\pth{-\Theta(n)}$
over $\bw$.

Suppose that for all $t_1 \in[0,t]$ with $t \in [0,T-1]$, $\bu(t_1) = \bv(t_1) + \be(t_1)$ where $\bv(t_1) \in \cV_{t_1}$, and
$\be(t_1) = \bbe_1(t_1) + \bbe_2(t_1)$ with
$\bv(t_1) \in \cV_{t_1}$ and $\be(t_1) \in \cE_{t_1,\tau}$, and
$\ltwonorm{\bu(t_1)} \le c_{\bu} {\sqrt n} $ for all $t_1 \in[0,t]$.
Then it follows from Lemma~\ref{lemma:empirical-loss-convergence-contraction} that the recursion
$\bu(t'+1)  = \pth{\bI- \eta \bKint_n }\bu(t') +\bEint(t')\bu(t')
 +\bE(t'+1)$ holds for all $t' \in [0,t]$.
 As a result, we have
\bal\label{eq:empirical-loss-convergence-seg5}
\bu(t+1)  &= \pth{\bI- \eta \bKint_n }\bu(t) +\bEint(t)\bu(t)
 +\bE(t+1) \nonumber \\
& = -\pth{\bI-\eta \bKint_n}^{t+1} f^*(\bS)
 -\pth{\bI-\eta \bKint_n }^{t+1} \bw \nonumber \\
&\phantom{=}+\sum_{t'=1}^{t+1}
 \pth{\bI-\eta \bKint_n}^{t+1-t'} \bE(t')
 +\sum_{t'=1}^{t+1} \pth{\bI-\eta \bKint_n}^{t+1-t'} \bEint(t'-1)\bu(t'-1)
\nonumber \\
&=\bv(t+1) + \be(t+1),
\eal
where $\bv(t+1)$ and $\be(t+1)$ are defined as
\bal\label{eq:empirical-loss-convergence-vt-et-def}
\bv(t+1) \defeq-\pth{\bI-\eta \bKint_n}^{t+1} f^*(\bS)\in \cV_{t+1},
\eal
\bal\label{eq:empirical-loss-convergence-et-pre}
&\be(t+1) \defeq \underbrace{-\pth{\bI-\eta \bKint_n }^{t+1} \bw}_{\bbe_1(t+1)}
\nonumber \\
&\phantom{}+ \underbrace{ \sum_{t'=1}^{t+1}
 \pth{\bI-\eta \bKint_n}^{t+1-t'} \bE(t') +\sum_{t'=1}^{t+1} \pth{\bI-\eta \bKint_n}^{t+1-t'}
\bEint(t'-1)\bu(t'-1)}_{\bbe_2(t+1)}.
\eal
We now prove the upper bound for $\bbe_2(t+1)$.
With $\eta \in (0,8)$, we have $\ltwonorm{\bI - \eta \bKint_n} \in (0,1]$.
It follows that
\bal\label{eq:empirical-loss-convergence-et-bound}
&\ltwonorm{\bbe_2(t+1)} \nonumber \\
&\le \sum_{t'=1}^{t+1} \ltwonorm{\bI-\eta \bKint_n}^{t+1-t'}\ltwonorm{\bE(t')} +\sum_{t'=1}^{t+1} \ltwonorm{\bI-\eta \bKint_n}^{t+1-t'}
\ltwonorm{\bEint(t'-1)}\ltwonorm{\bu(t'-1)}
\nonumber \\
&\stackrel{\circled{1}}{\le}  \frac{\tau {\sqrt n}}{2} +c_{\bu}{\sqrt n} \cdot
T \cdot \max_{t' \in [T]}\ltwonorm{\bEint(t'-1)}
\stackrel{\circled{2}}{\le} {\sqrt n}\tau,
%\nonumber \\
%&\le \pth{c_{\bu}-c_{\bu} + \tau} {\sqrt n},
\eal
where $\circled{1}$ follows from the fact
that $\ltwonorm{\bE(t)} \le \bE_{m,\eta,\delta} \le {\tau {\sqrt n}}/(2T)$ for all $t \in [T]$, (\ref{eq:empirical-loss-EKint-bound}) in
Lemma~\ref{lemma:empirical-loss-convergence-contraction},
and the induction hypothesis. $\circled{2}$ follows from
$N \gsim T^2\log{(4n/{\delta})}/\tau^2$
with $\max_{t' \in [T]}\ltwonorm{\bEint(t'-1)} \lsim \eta \sqrt{\frac{\log (4n/{\delta})}{N}}$ by Lemma~\ref{lemma:empirical-loss-convergence-contraction} and $c_{\bu} \eta = \Theta(1)$.
It follows from (\ref{eq:empirical-loss-convergence-et-bound}) that $\be(t+1) \in \cE_{t+1,\tau}$.
Also, it follows from Lemma~\ref{lemma:yt-y-bound}
that
\bals
\ltwonorm{\bu(t+1)} &\le \ltwonorm{\bv(t+1)} + \ltwonorm{\bbe_1(t+1)}
+\ltwonorm{\bbe_2(t+1)} \nonumber \\
&\le\pth{\frac{\mu_0}{ \sqrt{2e\eta } } + \sigma_0+\tau+1} {\sqrt n}
\le c_{\bu}{\sqrt n},
\eals
which completes the induction step and the entire proof.
\end{proof}

\begin{proof}
{\textbf{\textup{\hspace{-3pt}of
Lemma~\ref{lemma:yt-y-bound}}}.}
When $t \ge 1$, we have $\bv = -\pth{\bI-\eta\bKint_n}^t f^*(\bS)$,
and
\bal\label{eq:yt-y-bound-seg1}
\ltwonorm{\bv(t)}^2 &=\sum\limits_{i=1}^{n}
\pth{1-\eta \hlambdaint_i }^{2t}
\bth{{\bUint}^{\top} f^*(\bS)}_i^2  \nonumber \\
&\stackrel{\circled{1}}{\le}
\sum\limits_{i=1}^{\rank(\bKint_n)}
\frac{1}{2e\eta \hlambdaint_i  t}
\bth{{\bUint}^{\top} f^*(\bS)}_i^2
\stackrel{\circled{2}}{\le}
\frac{n \mu_0^2}{ 2e\eta t }.
\eal

Here $\circled{1}$ follows from Lemma~\ref{lemma:bounded-Ut-f-in-RKHS}, which shows that the signal has zero components in the empirical null space, and Lemma~\ref{lemma:auxiliary-lemma-1} applied to the positive empirical eigenvalues. $\circled{2}$ follows
by Lemma~\ref{lemma:bounded-Ut-f-in-RKHS}.
Moreover, it follows from the concentration inequality about quadratic forms of sub-Gaussian random variables in~\citet{quadratic-tail-bound-Wright1973} that
\bal\label{eq:yt-y-bound-eps-1}
\Pr\bth{\ltwonorm{\bw}^2 -
\Expect{}{\ltwonorm{\bw}^2} > n}
\le \exp\pth{-\Theta(n)},
\eal
and $\Expect{}{\ltwonorm{\bw}} \le
\sqrt{\Expect{}{\ltwonorm{\bw}^2}} \le \sqrt{n} \sigma_0$. Therefore,
$\Pr\bth{\ltwonorm{\bw} -
\sqrt{n} \sigma_0 > \sqrt n} \le
 \exp\pth{-\Theta(n)}$.
As a result, we have
\bals
\ltwonorm{\bv} + \ltwonorm{\be}
&\le
\sqrt{\frac{n \mu_0^2}{ 2e\eta }} + \ltwonorm{\bw} \le
  \pth{\frac{\mu_0}{ \sqrt{2e\eta  } } + \sigma_0+1 }{\sqrt n}, \quad \forall t \in [T].
\eals

When $t = 0$, $\ltwonorm{\bv} = \ltwonorm{f^*(\bS)} \le \mu_0/(2 {\sqrt 2}) \cdot {\sqrt n}$, which completes the proof of (\ref{eq:yt-y-bound}).
\end{proof}

\begin{proof}
{\textbf{\textup{\hspace{-3pt}of
Lemma~\ref{lemma:empirical-loss-convergence-contraction}}}.}
Because $\ltwonorm{\hat \by(t') - \by} \le {\sqrt{n}} c_{\bu}$ holds for all $t' \in [0,t]$, by Lemma~\ref{lemma:weight-vector-movement}, we have
\bal\label{eq:empirical-loss-convergence-pre1}
\norm{\bbw_r(t') - \bbw_r(0)}{2} & \le  R, \quad \forall \, 0 \le t' \le t+1.
\eal%
Define two sets of indices
\bals
E_{i,R} \defeq \set{r \in [m] \colon \abth{\bw_{r}(0)^{\top}\bbx_i} > R }, \quad \bar E_{i,R} \defeq [m] \setminus E_{i,R}.
\eals%
We have
\bal\label{eq:empirical-loss-convergence-contraction-seg1}
&\hat \by_i(t+1) - \hat \by_i(t) = \frac{1}{\sqrt m} \sum_{r=1}^m a_r \pth{ \relu{\bbw_{\bS,r}^{\top}(t+1) \bbx_i}  -  \relu{\bbw_{\bS,r}^{\top}(t) \bbx_i} } \nonumber \\
&=\underbrace{ \frac{1}{\sqrt m} \sum\limits_{r \in E_{i,R}} a_r \pth{ \relu{\bbw_{\bS,r}^{\top}(t+1) \bbx_i}  -  \relu{\bbw_{\bS,r}^{\top}(t) \bbx_i} }}_{\defeq \bD^{(1)}_i} \nonumber \\
&\phantom{=}{+} \underbrace{ \frac{1}{\sqrt m} \sum\limits_{r \in \bar E_{i,R}} a_r \pth{ \relu{\bbw_{\bS,r}^{\top}(t+1) \bbx_i}  -  \relu{\bbw_{\bS,r}^{\top}(t) \bbx_i} }}_{\defeq \bE^{(1)}_i} =\bD^{(1)}_i + \bE^{(1)}_i,
\eal%
and $\bD^{(1)}, \bE^{(1)} \in \RR^n$ are vectors with their $i$-th elements being $\bD^{(1)}_i$ and $\bE^{(1)}_i$ defined on the RHS of
 (\ref{eq:empirical-loss-convergence-contraction-seg1}).
Now we derive the upper bound for $\bE^{(1)}_i$. For all $i \in [n]$ we have
\bal\label{eq:empirical-loss-convergence-contraction-seg2}
\abth{\bE^{(1)}_i} &=  \abth{\frac{1}{\sqrt m}\sum\limits_{r \in \bar E_{i,R}} a_r \pth{ \relu{ \bbw_{\bS,r}(t+1)^\top \bbx_i }  -  \relu{ \bbw_{\bS,r}(t)^\top \bbx_i } } }  \nonumber \\
&\le \frac{1}{\sqrt m}\sum\limits_{r \in \bar E_{i,R}} \abth{ \bbw_{\bS,r}(t+1)^\top \bbx_i  - \bbw_{\bS,r}(t)^\top \bbx_i }
\le \frac{1}{\sqrt m}\sum\limits_{r \in \bar E_{i,R}} \ltwonorm{\bbw_{\bS,r}(t+1) - \bbw_{\bS,r}(t) }  \nonumber \\
&\stackrel{\circled{1}}{=} \frac{1}{\sqrt m} \sum\limits_{r \in \bar E_{i,R}} \ltwonorm{\frac{\eta}{nN} \bth{\bZ_{\bQ}(0)}_{[(r-1)d+1:rd]} \bK^{(s)}_N
\bZ_{\bQ}(0)^{\top} \bZ_{\bS}(t) \pth{\hat \by(t) - \by}
}  \nonumber \\
&\stackrel{\circled{2}}{\le} \frac{c_{\bu}}{\sqrt m}  \sum\limits_{r \in \bar E_{i,R}}  \frac{\eta }{2 \sqrt m}
\le {\eta} c_{\bu}\cdot \frac{\abth{\bar E_{i,R}}}{2m} .
\eal%
Here $\circled{1}, \circled{2}$ follow from (\ref{eq:weight-vector-movement-seg1-pre})
and (\ref{eq:weight-vector-movement-seg1}) in the proof of
Lemma~\ref{lemma:weight-vector-movement}, as well as the fact
that $\ltwonorm{\bK^{(s)}_N} \le 1/2$.
Since $\bW(0) \in \cW_0$, we have
\bal\label{eq:empirical-loss-convergence-contraction-seg3}
&\sup_{\bx \in \cX}\hat v_R(\bW(0),\bx)\le\frac{2R}{\sqrt {2\pi}\kappa}+C_2(m/2,d,1/n),
\eal%
where $\hat v_R(\bW(0),\bx) =  \frac 1m \sum\limits_{r=1}^m \indict{\abth{\bbw_r(0)^{\top}\bx} \le R }$, so that $ \hat v_R(\bW(0),\bbx_i) = \abth{\bar E_{i,R}}/m$.
It follows from (\ref{eq:empirical-loss-convergence-contraction-seg2}),
(\ref{eq:empirical-loss-convergence-contraction-seg3}) and the induction hypothesis that
\bal\label{eq:empirical-loss-convergence-contraction-seg4}
\abth{\bE^{(1)}_i}  &\le \frac{\eta c_{\bu}}{2} \pth{ \frac{2R}{\sqrt {2\pi} \kappa}+ C_2(m/2,d,1/n)} .
\eal%
It follows from (\ref{eq:empirical-loss-convergence-contraction-seg4}) that $\ltwonorm{\bE^{(1)}}$ can be bounded by
\bal\label{eq:empirical-loss-convergence-contraction-E1-bound}
\ltwonorm{\bE^{(1)}} & \le
\frac{\eta c_{\bu} {\sqrt n}}{2} \pth{ \frac{2R}{\sqrt {2\pi} \kappa}+ C_2(m/2,d,1/n)} .
\eal
$\bD^{(1)}_i$ on the RHS of  (\ref{eq:empirical-loss-convergence-contraction-seg1})
is expressed by
\bal\label{eq:empirical-loss-convergence-contraction-seg5}
\bD^{(1)}_i &= \frac{1}{\sqrt m} \sum\limits_{r \in E_{i,R}} a_r \pth{ \relu{\bbw_{\bS,r}^{\top}(t+1) \bbx_i}  -  \relu{\bbw_{\bS,r}^{\top}(t) \bbx_i} } \nonumber \\
&=  \frac{1}{\sqrt m} \sum\limits_{r \in E_{i,R}} a_r \indict{\bbw_{\bS,r}(t)^{\top} \bbx_i \ge 0} \pth{ \bbw_{\bS,r}(t+1)  -  \bbw_{\bS,r}(t) }^\top \bbx_i   \nonumber \\
&=  \frac{1}{\sqrt m} \sum\limits_{r=1}^m a_r \indict{\bbw_{\bS,r}(t)^{\top} \bbx_i \ge 0} \pth{ -\frac{\eta}{n} \bth{\bM}_{[(r-1)d+1:rd]}
\bZ_{\bS}(t) \pth{\hat \by(t) - \by}
 }^\top \bbx_i \nonumber \\
&\phantom{=}{+}  \frac{1}{\sqrt m} \sum\limits_{r \in \bar E_{i,R} } a_r \indict{\bbw_{\bS,r}(t)^{\top} \bbx_i \ge 0} \pth{\frac{\eta}{n}
\bth{\bM}_{[(r-1)d+1:rd]} \bZ_{\bS}(t) \pth{\hat \by(t) - \by}
 }^\top \bbx_i \nonumber \\
&=\underbrace{-\frac{\eta}{nN} \bth{\bH(t)}_i \bK^{(s)}_N \bH(t)^{\top}  \pth{\hat \by(t) - \by}}_{\defeq \bD^{(2)}_i} \nonumber \\
 &\phantom{=}{+} \underbrace{\frac{1}{\sqrt m} \sum\limits_{r \in \bar E_{i,R} } a_r \indict{\bbw_{\bS,r}(t)^{\top} \bbx_i \ge 0} \pth{\frac{\eta}{n}
\bth{\bM}_{[(r-1)d+1:rd]}\bZ_{\bS}(t) \pth{\hat \by(t) - \by}
 }^\top \bbx_i }_{\defeq \bE^{(2)}_i},
\eal%
where $\bH(t) \in \RR^{n \times N}$ is a matrix specified by
\bals
\bH_{pq}(t) = \frac{\bbx_p^\top \bbq_q}{m} \sum_{r=1}^{m} \indict{\bbw_{\bS,r}(t)^\top \bbx_p \ge 0} \indict{\bbw_{r}(0)^\top \bbq_q \ge 0}, \quad
\forall \,p \in [n], q \in [N].
\eals
Let $\bD^{(2)}, \bE^{(2)} \in \RR^n$ be vectors with their $i$-th elements being $\bD^{(2)}_i$ and
$\bE^{(2)}_i$ defined on the RHS of
(\ref{eq:empirical-loss-convergence-contraction-seg5}). $\bE^{(2)}$ can be expressed by $\bE^{(2)} = \frac{\eta}{nN} \tilde \bE^{(2)} \bK^{(s)}_N \bH(t)^{\top}  \pth{\hat \by(t) - \by}$ with $\tilde \bE^{(2)} \in \RR^{n \times N}$ and
\bals
\tilde \bE^{(2)}_{pq} = \frac{1}{m} \sum\limits_{r \in \bar E_{p,R}} \indict{\bbw_{\bS,r}(t)^{\top} \bbx_p \ge 0} \indict{\bbw_{r}(0)^{\top} \bbq_q \ge 0} \bbx_p^\top \bbq_q,
\qquad \abth{\tilde \bE^{(2)}_{pq}}\le\frac{\abth{\bar E_{p,R}}}{m}
\eals
for all $p \in [n], q \in [N]$.  The spectral norm of $\tilde \bE^{(2)}$ is bounded by
\bal\label{eq:empirical-loss-convergence-contraction-seg6}
\ltwonorm{\tilde \bE^{(2)}} \le \fnorm{\tilde \bE^{(2)}} \le \sqrt{nN}\max_{p\in[n]}\frac{\abth{\bar E_{p,R}}}{m}
\stackrel{\circled{1}}{\le} \sqrt{nN}\pth{ \frac{2R}{\sqrt {2\pi} \kappa}+
C_2(m/2,d,1/n)},
\eal%
where $\circled{1}$ follows from (\ref{eq:empirical-loss-convergence-contraction-seg3}).
Also, $\ltwonorm{\bH(t)} \le \fnorm{\bH(t)} \le \sqrt {nN}$ for all $t \ge 0$.
It follows from (\ref{eq:empirical-loss-convergence-contraction-seg6}) that $\ltwonorm{\bE^{(2)}}$ can be bounded by
\bal\label{eq:empirical-loss-convergence-contraction-E2-bound}
\ltwonorm{\bE^{(2)}} &\le \frac{\eta}{nN} \ltwonorm{\tilde\bE^{(2)}}
\ltwonorm{\bK^{(s)}_N} \ltwonorm{\bH(t)}
\ltwonorm{\hat\by(t)-\by} \le \frac{\eta c_{\bu} {\sqrt n}}{2} \pth{\frac{2R}{\sqrt {2\pi} \kappa}+
C_2(m/2,d,1/n)}.
\eal
$\bD^{(2)}_i$ on the RHS of (\ref{eq:empirical-loss-convergence-contraction-seg5}) is expressed by
\bal\label{eq:empirical-loss-convergence-contraction-seg7}
\bD^{(2)} &= -\frac{\eta}{nN} \bH(t)\bK^{(s)}_N \bH(t)^{\top}  \pth{\hat \by(t) - \by} \nonumber \\ &=\underbrace{-\frac{\eta}{nN} \bH(0)\bK^{(s)}_N \bH(0)^{\top}  \pth{\hat \by(t) - \by}}_{\defeq \bD^{(3)}} +  \underbrace{\frac{\eta}{nN}  \pth{\bH(0) - \bH(t)}  \bK^{(s)}_N\bH(0)^{\top}  \pth{\hat \by(t) - \by}}_{\defeq \bE^{(3)}} \nonumber \\
&\phantom{=}+  \underbrace{\frac{\eta}{nN} \bH(t) \bK^{(s)}_N\pth{\bH(0) - \bH(t)}^{\top}  \pth{\hat \by(t) - \by}}_{\defeq \bE^{(4)}} =\bD^{(3)} + \bE^{(3)} + \bE^{(4)}.
\eal%
On the RHS of  (\ref{eq:empirical-loss-convergence-contraction-seg7}), $\bD^{(3)},\bE^{(3)},\bE^{(4)} \in \RR^n$ are vectors which are analyzed as follows.
In order to bound the spectral norm of $\bE^{(3)}$ and $\bE^{(4)}$, we first estimate the upper bound for $\abth{ \bH_{ij}(t) - \bH_{ij}(0) }$ for all $i \in [n]$ and $j \in [N]$. We note that
\bal\label{eq:empirical-loss-convergence-contraction-seg8-pre}
&\indict{\indict{\bbw_{\bS,r}(t)^\top \bbx_i \ge 0} \neq \indict{\bw_{r}(0)^\top \bbx_i \ge 0}} \le \indict{\abth{\bw_{r}(0)^{\top} \bbx_i} \le R} + \indict{\ltwonorm{\bw_{
\bS,r}(t) - \bbw_r(0)} > R}.
\eal%
It follows from (\ref{eq:empirical-loss-convergence-contraction-seg8-pre}) that
\bal\label{eq:empirical-loss-convergence-contraction-seg8}
&\abth{ \bH_{ij}(t) - \bH_{ij}(0) } \nonumber \\
&= \abth{ \frac{\bbx_i^\top \bbq_j}{m} \sum_{r=1}^{m} \pth{ \indict{\bbw_{\bS,r}(t)^\top \bbx_i \ge 0} \indict{\bbw_{r}(0)^\top \bbq_j \ge 0} - \indict{\bw_{r}(0)^\top \bbx_i \ge 0} \indict{\bw_{r}(0)^\top \bbq_j \ge 0} }} \nonumber \\
&\le \frac1m \sum_{r=1}^{m}  \indict{\indict{\bbw_{\bS,r}(t)^\top \bbx_i \ge 0} \neq \indict{\bbw_r(0)^\top \bbx_i \ge 0}
 }  \nonumber \\
&\le \frac1m \sum_{r=1}^{m} \pth{ \indict{\abth{\bbw_r(0)^{\top} \bbx_i} \le R}  +  \indict{\ltwonorm{\bw_{
\bS,r}(t) - \bbw_r(0)} > R}   } \nonumber \\
&\le \hat v_R(\bW(0),\bbx_i) \stackrel{\circled{1}}{\le} \frac{2R}{\sqrt {2\pi} \kappa}+ C_2(m/2,d,1/n),
\eal%
where $\circled{1}$ follows from  (\ref{eq:empirical-loss-convergence-contraction-seg3}).
It follows from (\ref{eq:empirical-loss-convergence-contraction-seg8}) that $\ltwonorm{\bE^{(3)}},\ltwonorm{\bE^{(4)}} $ are bounded by
\bal
\ltwonorm{\bE^{(3)}} &\le
\frac{\eta}{2nN} \cdot \sqrt{nN} \cdot \sqrt{nN}\pth{\frac{2R}{\sqrt {2\pi} \kappa}+ C_2(m/2,d,1/n)} \ltwonorm{\hat\by(t)-\by}
\nonumber \\
&\le \frac{\eta c_{\bu}{\sqrt n}}{2}  \pth{\frac{2R}{\sqrt {2\pi} \kappa}+ C_2(m/2,d,1/n)} ,
\label{eq:empirical-loss-convergence-contraction-E3-bound} \\
\ltwonorm{\bE^{(4)}}
&\le \frac{\eta c_{\bu}{\sqrt n}}{2}  \pth{\frac{2R}{\sqrt {2\pi} \kappa}+ C_2(m/2,d,1/n)} . \label{eq:empirical-loss-convergence-contraction-E4-bound}
\eal
We now approximate $\bD^{(3)}$ in terms of $\bK_{\bS,\bQ} \in \RR^{n \times N}$
with $\bth{\bK_{\bS,\bQ}}_{ij} = K(\bbx_i,\bbq_j)$ \newline for $i \in [n], j \in [N]$. We have
\bal\label{eq:empirical-loss-convergence-contraction-seg9}
 \bD^{(3)} &= -\frac{\eta}{nN} \bH(0) \bK^{(s)}_N  \bH(0)^{\top}  \pth{\hat \by(t) - \by}
 \nonumber \\
&= \underbrace{-\frac{\eta}{nN} \bK_{\bS,\bQ}
\bK^{(s)}_N  \bK_{\bS,\bQ}^{\top}  \pth{\hat \by(t) - \by}}_{\defeq  \bD^{(4)}} + \underbrace{\frac{\eta}{nN} \pth{\bK_{\bS,\bQ}- \bH(0)}\bK^{(s)}_N\bK_{\bS,\bQ}^{\top}  \pth{\hat \by(t) - \by}}_{\defeq \bE^{(5)}} \nonumber \\
&\phantom{=}+ \underbrace{\frac{\eta}{nN}  \bH(0)\bK^{(s)}_N\pth{\bK_{\bS,\bQ}- \bH(0)}^{\top}  \pth{\hat \by(t) - \by}}_{\defeq \bE^{(6)}}.
\eal
We now bound the spectral norm of $\bE^{(5)}$ and $\bE^{(6)}$.
Since $\bW(0) \in \cW_0$,
 we have
\bal\label{eq:empirical-loss-convergence-contraction-seg10}
\ltwonorm{\bK_{\bS,\bQ}- \bH(0)} &\le \fnorm{\bK_{\bS,\bQ}- \bH(0)}
\le \sqrt{nN}  C_1(m/2,d,1/n),
\eal%
 Also, $\ltwonorm{\bK_{\bS,\bQ} } \le \fnorm{\bK_{\bS,\bQ} } \le \sqrt{nN}$. As a result,
\bal
\ltwonorm{\bE^{(5)}} &\le
\frac{\eta}{2nN} \cdot \sqrt{nN} \cdot \sqrt{nN} C_1(m/2,d,1/n) \ltwonorm{\hat\by(t)-\by}
\le \frac{\eta c_{\bu} {\sqrt n}}{2} \cdot C_1(m/2,d,1/n),
\label{eq:empirical-loss-convergence-contraction-E5-bound} \\
\ltwonorm{\bE^{(6)}} &\le \frac{\eta c_{\bu} {\sqrt n}}{2} \cdot C_1(m/2,d,1/n).
\label{eq:empirical-loss-convergence-contraction-E6-bound}
\eal
It follows from (\ref{eq:empirical-loss-convergence-contraction-seg5}),
 (\ref{eq:empirical-loss-convergence-contraction-seg7}),
 and  (\ref{eq:empirical-loss-convergence-contraction-seg9}) that
\bal\label{eq:empirical-loss-convergence-contraction-seg12}
\bD^{(1)}_i &=  \bD^{(4)}_i +  \bE^{(2)}_i+\bE^{(3)}_i + \bE^{(4)}_i +\bE^{(5)}_i + \bE^{(6)}_i.
\eal
It then follows from
(\ref{eq:empirical-loss-convergence-contraction-seg1}) that
\bal\label{eq:empirical-loss-convergence-contraction-seg13}
&\hat \by_i(t+1) - \hat \by_i(t) = \bD^{(1)}_i + \bE^{(1)}_i =\bD^{(4)}_i + \underbrace{\bE^{(1)}_i + \bE^{(2)}_i+\bE^{(3)}_i + \bE^{(4)}_i +\bE^{(5)}_i + \bE^{(6)}_i}_{\defeq \bE_i} \nonumber \\
&=-\frac{\eta}{nN} \bth{\bK_{\bS,\bQ}}_i
\bK^{(s)}_N  \bK_{\bS,\bQ}^{\top}  \pth{\hat \by(t) - \by}
+ \bth{\bE(t+1)}_i
=-\frac{\eta}{n} \bth{\hbKint}_i  \pth{\hat \by(t) - \by}
+ \bth{\bE(t+1)}_i,
\eal
where $\bE(t+1) \in \RR^n$ with its $i$-th element being $\bE_i$, and $\bE(t+1) = \bE^{(1)}
+\bE^{(2)}+\bE^{(3)} + \bE^{(4)} + \bE^{(5)}+ \bE^{(6)}$. It follows from
(\ref{eq:empirical-loss-convergence-contraction-E1-bound}),
(\ref{eq:empirical-loss-convergence-contraction-E2-bound}),
(\ref{eq:empirical-loss-convergence-contraction-E3-bound}),
(\ref{eq:empirical-loss-convergence-contraction-E4-bound}),
(\ref{eq:empirical-loss-convergence-contraction-E5-bound}),
 and (\ref{eq:empirical-loss-convergence-contraction-E6-bound}) that
\bal\label{eq:empirical-loss-convergence-contraction-E-bound}
&\ltwonorm {\bE(t+1)} \lsim \eta c_{\bu} {\sqrt n}
\pth{ 2 \pth{\frac{2R}{\sqrt {2\pi} \kappa}+
C_2(m/2,d,1/n)} +C_1(m/2,d,1/n)}.
\eal
The upper bound for $\ltwonorm {\bE(t+1)} $ in
(\ref{eq:empirical-loss-Et-bound-Em}) follows from (\ref{eq:empirical-loss-convergence-contraction-E-bound}),  Theorem~\ref{theorem:good-random-initialization}, and noting that $\eta c_{\bu} \le \Theta(1)$. Finally, we have
\bals%\label{eq:empirical-loss-convergence-contraction-seg14}
\hat \by(t+1) - \by
&\stackrel{\circled{1}}{=} \pth{\bI-\frac{\eta}{n} \hbKint}\pth{\hat \by(t) - \by} + \bE(t+1) \nonumber \\
&=\pth{\bI-\frac{\eta}{n} \bKint}\pth{\hat \by(t) - \by} +   \bEint(t)\pth{\hat \by(t) - \by}+\bE(t+1)
,
\eals
where $\bEint(t) = \eta/n \cdot \pth{\bKint - \hbKint}$, and $\circled{1}$ follows from
 (\ref{eq:empirical-loss-convergence-contraction-seg13}).
It follows from
(\ref{eq:hatKint-close-to-Kint-spectralnorm}) of
Theorem~\ref{theorem:hatKint-close-to-Kint-supnorm} that \bals
\ltwonorm{\bEint(t)} \lsim
\eta \sqrt{\frac{\log (4n/{\delta})}{N}}
\eals
holds with probability at least $1- \delta$
over $\bQ$, which proves (\ref{eq:empirical-loss-EKint-bound}).
\end{proof}

\begin{proof}
{\textbf{\textup{\hspace{-3pt}of
Lemma~\ref{lemma:weight-vector-movement}}}.}
Let $\bth{\bZ_{\bQ}(0)}_{[(r-1)d+1:rd]}$ denote the submatrix of $\bZ_{\bQ}(0)$ formed by the rows of $\bZ_{\bQ}(0)$ with row indices in $[(r-1)d+1:rd]$.
By the PGD update rule (\ref{eq:PGD-two-layer-nn}), for every $t'' \in [0\relcolon T-1]$, we have
\bal\label{eq:weight-vector-movement-seg1-pre}
&\bbw_{\bS,r}(t''+1) - \bbw_{\bS,r}(t'') = -\frac{\eta}{nN} \bth{\bZ_{\bQ}(0)}_{[(r-1)d+1:rd]}
\bK^{(s)}_N\bZ_{\bQ}(0)^{\top} \bZ_{\bS}(t'') \pth{\hat \by(t'') - \by}.
\eal%
Since
$\ltwonorm{\bth{\bZ_{\bQ}(0)}_{[(r-1)d+1:rd]}} \le
\fnorm{\bth{\bZ_{\bQ}(0)}_{[(r-1)d+1:rd]}} \le {\sqrt {N/m}}$,
and $\ltwonorm{\bZ_{\bQ}(0)^{\top} \bZ_{\bS}(t)} \le \sqrt{nN}$,
it then follows from (\ref{eq:weight-vector-movement-seg1-pre}) that
\bal\label{eq:weight-vector-movement-seg1}
&\ltwonorm{\bbw_{\bS,r}(t''+1) - \bbw_{\bS,r}(t'')} \nonumber \\
&\le \frac{\eta}{nN} \ltwonorm{\bth{\bZ_{\bQ}(0)}_{[(r-1)d+1:rd]}}
\ltwonorm{\bK^{(s)}_N} \ltwonorm{\bZ_{\bQ}(0)^{\top} \bZ_{\bS}(t'')}\ltwonorm{\hat \by(t'')-\by}
\le  \frac{\eta c_{\bu}}{2 \sqrt m}, \, \forall t'' \in [0\relcolon t].
\eal%
Note that (\ref{eq:R}) trivially holds for $t'=0$. For $t' \in [1,t+1]$, it follows from
 (\ref{eq:weight-vector-movement-seg1}) that
\bal\label{eq:weight-vector-movement-proof}
\ltwonorm{ \bbw_{\bS,r}(t') - \bbw_r(0) }
& \le \sum_{t''=0}^{t'-1} \ltwonorm{\bbw_{\bS,r}(t''+1) - \bbw_{\bS,r}(t'')} \le \frac{\eta }{2 \sqrt m}  \sum_{t''=0}^{t'-1} c_{\bu}
\le \frac{\eta c_{\bu}T }{2\sqrt m} = R,
\eal%
which completes the proof.
\end{proof}

\begin{lemma}\label{lemma:bounded-Linfty-vt-sum-et}
Let $\eta\in(0,8)$ and
$h(\cdot) = \sum_{t'=0}^{t-1} h(\cdot,t')$ for $t \in [T]$,
$T \le \hat T$ where
\bals
h(\cdot,t') &= v(\cdot,t') + \hat e(\cdot,t'), \\
v(\cdot,t')  &= \frac{\eta}{n} \sum_{j=1}^n
\Kint(\bbx_j,\bx) \bv_j(t') , \\
\hat e(\cdot,t') &= \frac{\eta}{n}
\sum\limits_{j=1}^n  \Kint(\bbx_j,\bx) \bbe_j(t'),
\eals
and $\bv(t') \in \cV_{t'}$,
$\be(t') \in \cE_{t',\tau}$ for all $0 \le t' \le t-1$.
Suppose that
\bal\label{eq:N-tau-cond-bounded-Linfty-vt-sum-et}
\tau \lsim 1/(\eta T). %\quad N \gsim T^2\log{\frac{2\pth{1+2/s}^{2d}}{\delta}}/\tau^2
\eal
%with a positive number $s \in (0,1/4]$ such that $s \lsim \tau/T$,
Then with probability at least $1 - \exp\pth{-\Theta(n \hat\eps_n^2)}$
over $\bw$,
\bal\label{eq:bounded-h}
\norm{h}{\cHKint} \le B_h = \mu_0 +1+ {\sqrt 2},
\eal
where $B_h$ is also defined in (\ref{eq:B_h}).
\end{lemma}
\begin{proof}
We have $\by = f^*(\bS) + \bw$,
$\bv(t) = -\pth{\bI- \eta \bKint_n }^t f^*(\bS)$,
$\be(t) = \bbe_1(t) + \bbe_2(t)$ with
$\bbe_1(t) = -\pth{\bI-\eta\bKint_n}^t \bw$,
$\ltwonorm{\bbe_2(t)} \lsim {\sqrt n} \tau$.
We define
\bals
\hat e_1(\cdot,t) =\frac{\eta}{n}
\sum\limits_{j=1}^n  \Kint(\bbx_j,\bx) \bth{\bbe_1(t)}_j,
\quad
\hat e_2(\cdot,t) =\frac{\eta}{n}
\sum\limits_{j=1}^n  \Kint(\bbx_j,\bx) \bth{\bbe_2(t)}_j,
\eals
we then have
\bal\label{eq:bounded-Linfty-vt-sum-seg1}
\sum_{t'=0}^{t-1} v(\bx,t') &=-\frac{\eta}{n} \sum\limits_{j=1}^n  \sum_{t'=0}^{t-1}
\bth{\pth{\bI- \eta \bKint_n }^{t'} f^*(\bS)}_j \Kint(\bbx_j,\bx) \nonumber \\
&=-\frac{\eta}{n} \sum\limits_{j=1}^n \sum_{t'=0}^{t-1}
\bth{\bUint \pth{\bI-\eta \bSigmaint }^{t'} {\bUint}^{\top} f^*(\bS)}_j \Kint(\bbx_j,\bx).
\eal
We let $r = \rank(\bKint_n) \in [n]$.  
It follows from (\ref{eq:bounded-Linfty-vt-sum-seg1}) that
\bal\label{eq:bounded-Linfty-vt-sum-seg2}
&\norm{\sum_{t'=0}^{t-1} v(\cdot,t')}{\cHKint}^2
\nonumber \\
&= \frac{\eta^2}{n^2} f^*(\bS)^{\top}
\bUint \sum_{t'=0}^{t-1} \pth{\bI-\eta \bSigmaint}^{t'} {\bUint}^{\top}
\bKint \bUint \sum_{t'=0}^{t-1} \pth{\bI-\eta \bSigmaint}^{t'}
{\bUint}^{\top} f^*(\bS) \nonumber \\
&= \frac 1n \ltwonorm{\eta\pth{\bKint_n}^{1/2} \bUint \sum_{t'=0}^{t-1} \pth{\bI-\eta \bSigmaint}^{t'} {\bUint}^{\top} f^*(\bS)}^2 \nonumber \\
&\le \frac 1n \sum\limits_{i=1}^{r} \frac{\pth{1-
\pth{1-\eta \hlambdaint_i }^t}^2}
{\hlambdaint_i}\bth{{\bUint}^{\top} f^*(\bS)}_i^2
\le \mu_0^2,
\eal
where the last inequality follows from Lemma~\ref{lemma:bounded-Ut-f-in-RKHS}.

Similarly, we have
\bal\label{eq:bounded-Linfty-hat-et-sum-1}
&\norm{\sum_{t'=0}^{t-1} \hat e_1(\cdot,t')}{\cHKint}^2
\le \frac 1n \sum\limits_{i=1}^{r} \frac{\pth{1-
\pth{1-\eta \hlambdaint_i }^t}^2}
{\hlambdaint_i}\bth{{\bUint}^{\top} \bw}_i^2.
\eal
It then follows from the argument in the proof of~\citep[Lemma 9]{RaskuttiWY14-early-stopping-kernel-regression}
that the RHS of (\ref{eq:bounded-Linfty-hat-et-sum-1}) is bounded with high probability. We define a diagonal matrix $\bR \in \RR^{n \times n}$
with $\bR_{ii} = \big(1-(1-\eta \hlambdaint_i )^t\big)^2/\hlambdaint_i$ for $i \in [r]$ and $\bR_{ii} = 0$ for $i \in [r+1,n]$. Then the RHS of (\ref{eq:bounded-Linfty-hat-et-sum-1}) is
$1/n \cdot \tr{\bUint \bR {\bUint}^{\top} \bw \bw^{\top} }$. It follows from~\citep{quadratic-tail-bound-Wright1973}
that
\bal\label{eq:bounded-Linfty-hat-et-sum-2}
&\Prob{1/n \cdot \tr{\bUint \bR {\bUint}^{\top}  \bw \bw^{\top}} -
\Expect{}{1/n \cdot \tr{\bUint \bR {\bUint}^{\top} \bw \bw^{\top} }} \ge u}
\nonumber \\
&\le \exp\pth{-c \min\set{nu/\ltwonorm{\bR},n^2u^2/\fnorm{\bR}^2}}
\eal
for all $u > 0$, and $c$ is a  positive constant. Recall that
$\eta_t = \eta t$ for all $t \ge 0$. We have
\bal\label{eq:bounded-Linfty-hat-et-sum-3}
&\Expect{}{1/n \cdot \tr{\bUint \bR {\bUint}^{\top}  \bw \bw^{\top} }}
\le \frac {\sigma_0^2}n \sum\limits_{i=1}^r
\frac{\pth{1-\pth{1-\eta \hlambdaint_i }^t}^2}{\hlambdaint_i}
\nonumber \\
& \stackrel{\circled{1}}{\le}
\frac {\sigma_0^2}n \sum\limits_{i=1}^r
\min\set{\frac{1}{\hlambdaint_i},\eta_t^2 \hlambdaint_i}
\le
\frac {{\sigma_0^2}\eta_t}n \sum\limits_{i=1}^r
\min\set{\frac{1}{\eta_t\hlambdaint_i},\eta_t \hlambdaint_i}
\nonumber \\
&\stackrel{\circled{2}}{\le}
\frac {{\sigma_0^2}\eta_t}n \sum\limits_{i=1}^r
\min\set{1,\eta_t \hlambdaint_i}
= \frac {{\sigma_0^2}\eta_t^2}n \sum\limits_{i=1}^r
\min\set{{\eta_t}^{-1},\hlambdaint_i} \le {{\sigma_0^2}\eta_t^2} \hat R_{\Kint}^2(\sqrt{{1}/{\eta_t}}) \le 1.
\eal
Here $\circled{1}$ follows from the fact that
$(1-\eta \hlambdaint_i )^t \ge \max\set{0,1-t\eta \hlambdaint_i}$,
and $\circled{2}$ follows from
$\min\set{a,b} \le \sqrt{ab}$ for any nonnegative numbers $a,b$.
Because $t \le T \le \hat T$, we have
$\hat R_{\Kint}(\sqrt{{1}/{\eta_t}}) \le 1/(\sigma_0 \eta_t)$, so the last inequality holds.

Moreover, we have the upper bounds for $\ltwonorm{\bR}$ and $\fnorm{\bR}$
as follows. First, we have
\bal\label{eq:bounded-Linfty-hat-et-sum-4}
\ltwonorm{\bR} &\le \max_{i \in [r] }\frac{\pth{1-\pth{1-\eta \hlambdaint_i }^t}^2}{\hlambdaint_i}
\le \max_{i\in[r]}\min\set{\frac{1}{\hlambdaint_i},\eta_t^2 \hlambdaint_i}
\le \eta_t.
\eal
We also have
\bal\label{eq:bounded-Linfty-hat-et-sum-5}
\frac 1n \fnorm{\bR}^2 &=  \frac 1n
\sum\limits_{i=1}^r
\frac{\pth{1-\pth{1-\eta \hlambdaint_i }^t}^4}{(\hlambdaint_i)^2}
\le \frac {\eta_t^3}n \sum\limits_{i=1}^r
\min\set{\frac{1}{\eta_t^3(\hlambdaint_i)^2},\eta_t (\hlambdaint_i)^2} \nonumber \\
&\le \frac {\eta_t^3}n \sum\limits_{i=1}^r
\min\set{\hlambdaint_i,\frac{1}{\eta_t}}
\le \eta_t^3 \hat R_{\Kint}^2(\sqrt{{1}/{\eta_t}})\le
\frac {\eta_t}{\sigma_0^2}.
\eal
Combining (\ref{eq:bounded-Linfty-hat-et-sum-1})-(\ref{eq:bounded-Linfty-hat-et-sum-5}) with $u=1$ in
(\ref{eq:bounded-Linfty-hat-et-sum-2}), we have
\bals
&\Prob{1/n \cdot \tr{\bUint \bR {\bUint}^{\top} \bw \bw^{\top} } -
\Expect{}{1/n \cdot \tr{\bUint \bR {\bUint}^{\top} \bw \bw^{\top} }} \ge 1} \nonumber \\
&\le \exp\pth{-c \min\set{n/\eta_t,n \sigma_0^2/\eta_t}} 
\le \exp\pth{-nc'/\eta_t} \le
\exp\pth{-c'n\hat\eps_n^2},
\eals
where $c'  = c\min\set{1,\sigma_0^2}$, and the last inequality is due
to the fact that $1/\eta_t \ge \hat\eps_n^2$ since
$t \le T \le \hat T$.
It follows that w.p. at least $1-
\exp\pth{- \Theta(n\hat\eps_n^2)}$,
$\norm{\sum_{t'=0}^{t-1} \hat e_1(\cdot,t')}{\cH_{\Kint}}^2 \le 2$. This event can be chosen simultaneously for $t\in[T]$: the nonnegative diagonal filter $[1-(1-\eta\hlambdaint_i)^t]^2/\hlambdaint_i$ is nondecreasing in $t$ on the positive spectrum, so the quadratic form at $T$ dominates all earlier forms.

We now find the upper bound for $\norm{\sum_{t'=0}^{t-1} \hat e_2(\cdot,t')}{\cHKint}$. We have
\bals
\norm{\hat e_2(\cdot,t')}{\cHKint}^2
&\le \frac{\eta^2}{n^2} \bbe_2^{\top}(t')\bKint\bbe_2(t')
\lsim \eta^2 \hlambdaint_1 \tau^2,
\eals
so that
\bal\label{eq:bounded-Linfty-hat-et-sum-6}
&\norm{\sum_{t'=0}^{t-1} \hat e_2(\cdot,t')}{\cHKint}
\le \sum_{t'=0}^{t-1} \norm{\hat e_2(\cdot,t')}{\cHKint}
\lsim  \sqrt{\hlambdaint_1} \eta T \tau \le 1,
\eal
since $\tau \lsim 1/(\eta T) $.
%We note that Lemma~\ref{lemma:empirical-loss-convergence} requires
%that $N \gsim T^2\log{(n/{\delta})}/\tau^2$.
Finally, we have
\bals
\norm{h}{\cHKint}  &\le \norm{\sum_{t'=0}^{t-1} v(\cdot,t')}{\cHKint}
+\norm{\sum_{t'=0}^{t-1} \hat e_1(\cdot,t')}{\cHKint} + \norm{\sum_{t'=0}^{t-1} \hat e_2(\cdot,t')}{\cHKint} \nonumber \\
&\le \mu_0 +1+ {\sqrt 2}.
\eals
\end{proof}

\begin{theorem}\label{theorem:bounded-NN-class}
Suppose $w \in (0,1)$,
\bal
m &\gsim \max\set{{T^{\frac {15}{2}} d^{\frac 52}}/{w^5}, T^{\frac{25}{2}} d^{\frac 52}}, \label{eq:m-cond-bounded-NN-class} \\
N &\gsim \max\set{T^2\log{(4n/{\delta})}/w^2, T^4\log{(4n/{\delta})}}, \label{eq:N-cond-bounded-NN-class}
\eal
and the neural network $f_t(\cdot) = f(\bW(t),\cdot) $ is trained by PGD in Algorithm~\ref{alg:PGD} with the learning rate $\eta = \Theta(1) \in (0,8)$ and the random initialization $\bW(0) \in \cW_0$. Then for every $t \in [T]$ with $T \le \hat T$ and every $\delta \in (0,1)$, w.p. at least
$1-\exp\pth{-\Theta(n)}-\delta -\exp\pth{-\Theta(n \hat\eps_n^2)}$
over the random noise $\bw$ and the random sample $\bQ$,
$f_t \in \cFnn(\bS,\bQ,\bW(0),T) $, and $f_t$ can be decomposed by
\bal\label{eq:nn-function-class-decomposition}
f_t = h + e \in \cFext(B_h,w),
\eal
where $h \in \cH_{\Kint}(B_h)$ with $B_h$ defined in (\ref{eq:B_h}), $e \in L^{\infty}$ with $\supnorm{e} \le w$.
In addition, $\supnorm{f_t} \le \frac{B_h}{2{\sqrt 2}} + w$.
\end{theorem}
\begin{proof}
It follows from Lemma~\ref{lemma:empirical-loss-convergence}
and its proof that conditioned on an event with probability at
least $1 -  \exp\pth{-\Theta(n)}-\delta$
over $\bw$ and $\bQ$,
% subset $\Omega_1 \subseteq (\cX)^n \times (\cX)^N$ with
% $\Prob{\Omega_1} \ge 1-\Theta\pth{{nN}/{n^{c_d\eps^2_0/8}}} -
% \pth{1+2N}^{2d}\exp(-n^{c_x})$ and a subset $\Omega_2 \subseteq \RR^n$
% such that $\Prob{\Omega_2} \ge 1 -  \exp\pth{-\Theta(n)}$ that
% when the random training data $\bS = \set{\bbx_j}_{j=1}^n$ and $\bQ$ satisfy %$(\bS,\bQ) \in \Omega_1$ and the random noise $\bw \in \Omega_2$, the
% neural network trained on the $\bS$ with $\bQ$ and $\bw$ enjoys the properties %specified by Lemma~\ref{lemma:empirical-loss-convergence}. In particular,
$f \in \cFnn(\bS,\bQ,\bW(0),T)$ with
$\bW(0) \in \cW_0$. Moreover, $f(\cdot) = f(\bW,\cdot)$ with $\bW = \set{\bbw_r}_{r=1}^m \in \cW(\bS,\bQ,\bW(0),T)$, and $\vect{\bW} = \vect{\bW_{\bS}} = \vect{\bW(0)} - \sum_{t'=0}^{t-1} \eta/n \cdot \bM \bZ_{\bS}(t') \bu(t')$ for some $t \in [T]$, where $\bu(t') \in \RR^n, \bu(t') = \bv(t') + \be(t')$ with $\bv(t') \in \cV_{t'}$ and $\be(t') \in \cE_{t',\tau}$ for all $t' \in [0,t-1]$.

We note that $\bbw_r$ is expressed as
\bal\label{eq:bounded-Linfty-function-class-wr}
\bbw_r = \bbw_{\bS,r}(t) &= \bbw_r(0) - \sum_{t'=0}^{t-1} \frac{\eta}{nN} \bth{\bZ_{\bQ}(0)}_{[(r-1)d+1:rd]}\bK^{(s)}_N
\bZ_{\bQ}(0)^{\top} \bZ_{\bS}(t') \bu(t'),
\eal%
where the notation $\bbw_{\bS,r}$ emphasizes that $\bbw_r$ depends on the training features $\bS$.
We define the event
\bals%\label{eq:bounded-Linfty-function-class-Er}
&E_r(R) \defeq \set{ \abth{ \bbw_r(0)^\top \bx} \le R }, \quad r\in [m].
\eals%
We now approximate $ f(\bW,\bx)$ by $g(\bx) \defeq \frac{1}{\sqrt m} \sum_{r=1}^m a_r
\indict{\bbw_r(0)^{\top} \bx \ge 0} \bbw_r^\top \bx$. We have
\bal\label{eq:bounded-Linfty-function-class-seg1}
&\abth{f(\bW,\bx) - g(\bx)} \nonumber \\
&=\frac{1}{\sqrt m} \abth{\sum\limits_{r=1}^m a_r \relu{\bbw_r^\top \bx} - \sum_{r=1}^m a_r \indict{\bbw_r(0)^{\top} \bx \ge 0} \bbw_r^\top \bx}   \nonumber \\
&\le \frac{1}{\sqrt m}  \sum_{r=1}^m  \abth{ a_r \pth{ \indict{E_r(R) } + \indict{\bar E_r(R) }}  \pth{ \relu{\bbw_r^\top \bx} - \indict{\bbw_r(0)^{\top} \bx \ge 0} \bbw_r^\top \bx } } \nonumber \\
&=\frac{1}{\sqrt m} \sum_{r=1}^m   \indict{E_r(R)} \abth{ \relu{\bbw_r^\top \bx} - \indict{\bbw_r(0)^{\top} \bx \ge 0} \bbw_r^\top \bx  }\nonumber \\
&=\frac{1}{\sqrt m} \sum_{r=1}^m   \indict{E_r(R)}\abth{ \relu{\bbw_r^\top \bx} - \relu{\bbw_r(0)^\top \bx}   - \indict{\bbw_r(0)^{\top} \bx \ge 0} (\bbw_r-\bbw_r(0))^\top \bx } \nonumber \\
&\le \frac{2R}{\sqrt m} \sum_{r=1}^m   \indict{E_r(R)}.
\eal%
Plugging $R = \frac{\eta c_{\bu}T }{2\sqrt m}$ in (\ref{eq:bounded-Linfty-function-class-seg1}), we have
\bal\label{eq:bounded-Linfty-function-class-seg2}
\abth{f(\bW,\bx) - g(\bx)} &\le \frac{2R}{\sqrt m} \sum_{r=1}^m   \indict{E_r(R)} = \eta c_{\bu}T \cdot \frac 1m \sum_{r=1}^m
 \indict{E_r(R)} \nonumber \\
&=\eta c_{\bu}T \cdot
\hat v_R(\bW(0),\bx)
\le
 \eta c_{\bu}T \pth{\frac{2R}{\sqrt {2\pi} \kappa} + C_2(m/2,d,1/n)}.
\eal
Using (\ref{eq:bounded-Linfty-function-class-wr}), we can express $g(\bx)$ as
\bal\label{eq:bounded-Linfty-function-class-seg3}
&g(\bx) = \frac{1}{\sqrt m} \sum_{r=1}^m a_r \indict{\bbw_r(0)^{\top}
\bx \ge 0} \bbw_r(0)^\top \bx \nonumber \\
&\phantom{=}{-}\sum_{t'=0}^{t-1} \underbrace{ \frac{1}{\sqrt m}
\sum\limits_{r=1}^m a_r \indict{\bbw_r(0)^{\top} \bx \ge 0}\pth{ \frac{\eta}{n} \bth{\bM}_{[(r-1)d+1:rd]} \bZ_{\bS}(t') \bu(t')
 }^\top \bx }_{\defeq G_{t'}(\bx)} \stackrel{\circled{1}}{=} -\sum_{t'=0}^{t-1}G_{t'}(\bx),
\eal
where $\circled{1}$ follows from the fact that $\frac{1}{\sqrt m} \sum_{r=1}^m a_r \indict{\bbw_r(0)^{\top} \bx \ge 0} \bbw_r(0)^\top \bx = f(\bW(0),\bx) = 0$ due to the symmetric initialization of the two-layer NN. For each $G_{t'}$ on the RHS of
 (\ref{eq:bounded-Linfty-function-class-seg3}),
 \vspace{-.08in}
\bal\label{eq:bounded-Linfty-function-class-Gt}
G_{t'}(\bx) &\stackrel{\circled{2}}{=} \frac{\eta}{nN {\sqrt m}} \sum\limits_{r=1}^m a_r \indict{\bbw_r(0)^{\top} \bx \ge 0} \pth{  \bth{\bZ_{\bQ}(0)}_{[(r-1)d+1:rd]}
\bK^{(s)}_N\bH(0)^{\top} \bu(t')
 }^\top   \bx \nonumber \\
&\phantom{=}{+} \frac{\eta}{nN {\sqrt m}} \sum\limits_{r=1}^m a_r \indict{\bbw_r(0)^{\top} \bx \ge 0}  \pth{  \bth{\bZ_{\bQ}(0)}_{[(r-1)d+1:rd]}
\bK^{(s)}_N \pth{\bH(t') - \bH(0)}  ^{\top}\bu(t')
 }^\top \bx \nonumber \\
&\stackrel{\circled{3}}{=}
\frac{\eta}{nN {\sqrt m}} \sum\limits_{r=1}^m a_r \indict{\bbw_r(0)^{\top} \bx \ge 0} \pth{  \bth{\bZ_{\bQ}(0)}_{[(r-1)d+1:rd]}
\bK^{(s)}_N\bK_{\bS,\bQ}^{\top} \bu(t')
 }^\top   \bx + \nonumber \\
&\phantom{=}+
\frac{\eta}{nN {\sqrt m}} \sum\limits_{r=1}^m a_r \indict{\bbw_r(0)^{\top} \bx \ge 0} \pth{  \bth{\bZ_{\bQ}(0)}_{[(r-1)d+1:rd]}
\bK^{(s)}_N\pth{\bH(0) - \bK_{\bS,\bQ}} ^{\top}\bu(t')
 }^\top   \bx \nonumber \\
&\phantom{=}+\frac{\eta}{nN {\sqrt m}} \sum\limits_{r=1}^m a_r \indict{\bbw_r(0)^{\top} \bx \ge 0}  \pth{  \bth{\bZ_{\bQ}(0)}_{[(r-1)d+1:rd]}
\bK^{(s)}_N \pth{\bH(t') - \bH(0)}  ^{\top}\bu(t')
 }^\top \bx \nonumber \\
&\stackrel{\circled{4}}{=}\underbrace{\frac{\eta}{nN} \sum\limits_{j=1}^n
\sum\limits_{j',p=1}^N  K(\bx,\bbq_{j'})
\bth{\bK^{(s)}_N}_{j'p}
\bth{\bK_{\bS,\bQ}^{\top}}_{pj} \bu_j(t')}_{\defeq D(\bx)} \nonumber \\
&\phantom{=}+ \underbrace{\frac{\eta}{nN} \sum\limits_{j=1}^n
\sum\limits_{j',p=1}^N q_{j'} \bth{\bK^{(s)}_N}_{j'p}
\bth{\bK_{\bS,\bQ}^{\top}}_{pj} \bu_j(t')}_{\defeq E_1(\bx)}
\nonumber \\
&\phantom{=}+
\underbrace{\frac{\eta}{nN {\sqrt m}} \sum\limits_{r=1}^m a_r \indict{\bbw_r(0)^{\top} \bx \ge 0} \pth{  \bth{\bZ_{\bQ}(0)}_{[(r-1)d+1:rd]}
\bK^{(s)}_N\pth{\bH(0) - \bK_{\bS,\bQ}}^{\top} \bu(t')
 }^\top   \bx}_{\defeq E_2(\bx)} \nonumber \\
&\phantom{=}+
\underbrace{\frac{\eta}{nN {\sqrt m}} \sum\limits_{r=1}^m a_r \indict{\bbw_r(0)^{\top} \bx \ge 0}  \pth{  \bth{\bZ_{\bQ}(0)}_{[(r-1)d+1:rd]} \bK^{(s)}_N
\pth{\bH(t') - \bH(0)}  ^{\top}\bu(t')
 }^\top \bx}_{\defeq E_3(\bx)}
\eal
where $\bH(t) \in \RR^{n \times N}$ in $\circled{2}$ is a matrix defined by
\bals
\bH_{pq}(t) = \frac{\bbx_p^\top \bbq_q}{m} \sum_{r=1}^{m} \indict{\bbw_{\bS,r}(t)^\top \bbx_p \ge 0} \indict{\bbw_{r}(0)^\top \bbq_q \ge 0}, \quad
\forall \,p \in [n], q \in [N].
\eals
$\bK_{\bS,\bQ} \in \RR^{n \times N}$ in $\circled{3}$ is specified by
 $\bth{\bK_{\bS,\bQ}}_{ij} = K(\bbx_i,\bbq_j)$ for $i \in [n], j \in [N]$,
 and
 $q_{j'} \defeq \hat h(\bW(0),\bbq_{j'},\bx) - K(\bbq_{j'},\bx)$ for all $j' \in [N]$ in $\circled{4}$.

We now analyze each term on the RHS of (\ref{eq:bounded-Linfty-function-class-Gt}).
Let $h(\cdot,t') \colon \cX \to \RR$ be defined by
\bals
h(\bx,t') \defeq \frac{\eta}{n} \sum\limits_{j=1}^n \Kint(\bx,\bbx_{j}) \bu_j(t'),
\eals
then $h(\cdot,t')$ is an element in the RKHS $\cHKint$ for each $t' \in [0,t-1]$. It follows from (\ref{eq:hatKint-close-to-Kint-S}) in
 Theorem~\ref{theorem:hatKint-close-to-Kint-supnorm}
that with probability at least
$1-\delta$ over $\bQ$,
\bal\label{eq:bounded-Linfty-function-class-hKint-Kint-appro-ori}
\sup_{\bx \in \cX, i \in [n]}
\abth{\hKint (\bx, \bbx_i) - \Kint (\bx, \bbx_i)}
\lsim \sqrt{\frac{\log (4n/{\delta})}{N}}.
\eal
As a result, we have
\bal\label{eq:bounded-Linfty-function-class-D-bound}
\supnorm{D(\bx) - h(\bx,t')} &=
\frac{\eta}{n} \sup_{\bx \in \cX}\abth{ \sum\limits_{j=1}^n
\hKint(\cdot,\bbx_j) \bu_j(t') - \sum\limits_{j=1}^n
\Kint(\cdot,\bbx_j) \bu_j(t')} \nonumber \\
&\lsim \eta \sqrt{\frac{\log (4n/{\delta})}{N}}.
\eal
We define
\bal\label{eq:bounded-Linfty-function-class-h}
h(\cdot) \defeq -\sum_{t'=0}^{t-1} h(\cdot,t'),
\eal
and $e(\cdot) \defeq f(\bW,\cdot) - h(\cdot)$.
We now derive the $L^{\infty}$-norm of $e$.
To this end, we find the upper $L^{\infty}$ bounds for $E_1,E_2,E_3$
on the RHS of (\ref{eq:bounded-Linfty-function-class-Gt}).
Since $\bW(0) \in \cW_0$, $\abth{q_{j'}} \le C_1(m/2,d,1/n)$ for all $j' \in [N]$. Moreover,
$\ltwonorm{\bK_{\bS,\bQ}} \le \sqrt{nN}$,
$\ltwonorm{\bu(t')} \le c_{\bu} \sqrt n$ with high probability,  and we have
\bal\label{eq:bounded-Linfty-function-class-E1-bound}
\supnorm{ E_1} &= \supnorm{\frac{\eta}{nN} \sum\limits_{j=1}^n
\sum\limits_{j',p=1}^N q_{j'} \bth{\bK^{(s)}_N}_{j'p}
\bth{\bK_{\bS,\bQ}^{\top}}_{pj} \bu_j(t')} \nonumber \\
&\le \frac{\eta}{nN} \ltwonorm{\bu(t')} \ltwonorm{\bK^{(s)}_N} \ltwonorm{\bK_{\bS,\bQ}} \sqrt{N} C_1(m/2,d,1/n)
\le \frac{\eta c_{\bu}}{2} \cdot C_1(m/2,d,1/n).
\eal
We now bound the last term on the RHS of (\ref{eq:bounded-Linfty-function-class-Gt}).
Define $\bX' \in \RR^{dm}$ with  $\bth{\bX'}_{(r-1)d+1:rd} = \frac{a_r}{\sqrt m} \indict{\bbw_r(0)^{\top} \bx \ge 0} \bbx$ for all $r \in [m]$, then $\ltwonorm{\bX'} \le 1$.
Moreover, $\ltwonorm{\bK_{\bS,\bQ}- \bH(0)} \le \sqrt{nN} \newline C_1(m/2,d,1/n)$ by (\ref{eq:empirical-loss-convergence-contraction-seg10})
in the proof of Lemma~\ref{lemma:empirical-loss-convergence-contraction},
and $\ltwonorm{\bZ_{\bQ}(0)} \le \sqrt N$, so we have
\bal\label{eq:bounded-Linfty-function-class-E2-bound}
&\supnorm{ E_2}  \le\frac{\eta}{nN } \ltwonorm{
\pth{{\bX'}^{\top} \bZ_{\bQ}(0) \bK^{(s)}_N\pth{\bH(0) - \bK_{\bS,\bQ}}^{\top} \bu(t')}}
\nonumber \\
&\le \frac{\eta}{nN } \ltwonorm{\bX'} \ltwonorm{\bZ_{\bQ}(0)}
\ltwonorm{\bK^{(s)}_N}
\ltwonorm{\bH(0) - \bK_{\bS,\bQ}} \ltwonorm{\bu(t')}\nonumber \\
&\le \frac{\eta}{2nN } \cdot \sqrt N \cdot \sqrt{nN}  C_1(m/2,d,1/n)
\cdot  c_{\bu} \sqrt n
\le \frac{\eta c_{\bu}}{2} \cdot C_1(m/2,d,1/n).
\eal
We have
$\ltwonorm{ \bH_{ij}(t) - \bH_{ij}(0) } \le
\sqrt{nN}\pth{\frac{2R}{\sqrt {2\pi} \kappa}+ C_2(m/2,d,1/n)}$
 by (\ref{eq:empirical-loss-convergence-contraction-seg8})
in the proof of Lemma~\ref{lemma:empirical-loss-convergence-contraction}. As a result, $E_3$ can be bounded by
\bal\label{eq:bounded-Linfty-function-class-E3-bound}
&\supnorm{ E_3 }\le
\frac{\eta}{nN}\ltwonorm{{\bX'}^{\top} \bZ_{\bQ}(0)\bK^{(s)}_N \pth{\bH(t') - \bH(0)}^{\top}
 \bu(t')} \nonumber \\
&\le \frac{\eta}{nN} \ltwonorm{\bX'} \ltwonorm{\bZ_{\bQ}(0)}
\ltwonorm{\bK^{(s)}_N}
\ltwonorm{\bH(t') - \bH(0)} \ltwonorm{\bu(t')} \nonumber \\
&\le\frac{\eta}{2nN } \cdot \sqrt N \cdot
\sqrt{nN}\pth{\frac{2R}{\sqrt {2\pi} \kappa}+ C_2(m/2,d,1/n)}
\cdot  c_{\bu} \sqrt n \nonumber \\
&\le \frac{\eta c_{\bu}}{2} \pth{\frac{2R}{\sqrt {2\pi} \kappa}+ C_2(m/2,d,1/n)}.
\eal
Combining (\ref{eq:bounded-Linfty-function-class-Gt}),
(\ref{eq:bounded-Linfty-function-class-D-bound}),
(\ref{eq:bounded-Linfty-function-class-E1-bound}),
(\ref{eq:bounded-Linfty-function-class-E2-bound}),
and (\ref{eq:bounded-Linfty-function-class-E3-bound}), for
any $t' \in [0,t-1]$, we obtain
\bal\label{eq:bounded-Linfty-function-class-Gt-ht-bound}
&\sup_{\bx \in \cX} \abth{G_{t'}(\bx)-h(\bx,t')}
\le \sup_{\bx \in \cX} \abth{D(\bx) - h(\bx,t')} + \supnorm{E_1} + \supnorm{E_2}
+ \supnorm{E_3} \nonumber \\
&\le \eta c_{\bu} \pth{ \frac{\sqrt{\frac{\log (4n/{\delta})}{N}}}{c_{\bu}}
+  C_1(m/2,d,1/n)
+\frac 12\pth{\frac{2R}{\sqrt {2\pi} \kappa} + C_2(m/2,d,1/n)} }.
\eal
It then follows from (\ref{eq:bounded-Linfty-function-class-seg2}),
(\ref{eq:bounded-Linfty-function-class-seg3}), and
(\ref{eq:bounded-Linfty-function-class-Gt-ht-bound})
that $\supnorm{e}$ can be bounded by
\bal\label{eq:bounded-Linfty-function-class-f-h-bound}
&\supnorm{e}=
\supnorm{f(\bW,\cdot) - h} \le \supnorm{f(\bW,\cdot) - g}
+\supnorm{g-h} \nonumber \\
&\le  \supnorm{f(\bW,\cdot) - g}  + \sum\limits_{t'=0}^{t-1}
\supnorm{G_{t'}-h(\cdot,t')} \nonumber \\
& \stackrel{\circled{2}}{\le}
\eta c_{\bu}T \pth{\frac{2R}{\sqrt {2\pi} \kappa} + C_2(m/2,d,1/n)}
\nonumber \\
&+ \eta c_{\bu} T \pth{ \frac{\sqrt{\frac{\log (4n/{\delta})}{N}}}{c_{\bu}}
+ C_1(m/2,d,1/n)
+\frac 12\pth{\frac{2R}{\sqrt {2\pi} \kappa} + C_2(m/2,d,1/n)} }
\nonumber \\
&\le \eta c_{\bu} T \pth{
\frac{\sqrt{\frac{\log (4n/{\delta})}{N}}}{c_{\bu}}+C_1(m/2,d,1/n)
+\frac 32\pth{\frac{2R}{\sqrt {2\pi} \kappa} + C_2(m/2,d,1/n)}  }
\defeq \Delta_{m,n,N,\eta}.
\eal
We now give estimates for $\Delta_{m,n,N,\eta}$.
Since $\bW(0) \in \cW_0$,
it follows from
Theorem~\ref{theorem:good-random-initialization} that
\bals
\Delta_{m,n,N,\eta}
&\lsim T \sqrt{\frac{\log (4n/{\delta})}{N}}+
 \sqrt{d} m^{-\frac 15} T^{\frac 32},
\eals
By direct calculations, for any $w > 0$, when
\bals
N \gsim T^2 \log (4n/{\delta})/w^2,
\quad
m \gsim {T^{\frac {15}{2}} d^{\frac 52}}/{w^5},
\eals
we have
$\Delta_{m,n,N,\eta} \le w$. It follows from Lemma~\ref{lemma:bounded-Linfty-vt-sum-et} that
with probability at least $1- \exp\pth{-\Theta(n \hat\eps_n^2)}$ over the random noise $\bw$,
\bal\label{eq:bounded-Linfty-function-class-h-bound}
\norm{h}{\cHKint} \le B_h,
\eal
where $B_h$ is defined in (\ref{eq:B_h}),
and $\tau$ is required to satisfy
$\tau \lsim 1/(\eta T)$. Let $\tau = \Theta(1/(\eta T))$.
 Lemma~\ref{lemma:empirical-loss-convergence} requires that
$m \gsim T^{\frac {15}{2}} d^{\frac 52}/{\tau^5}$. As a result,
we have
\bals
m \gsim T^{\frac{25}{2}} d^{\frac 52}.
\eals
We note that the requirement on $N$, $N \gsim T^2\log{(4n/{\delta})}/\tau^2$, in Lemma~\ref{lemma:empirical-loss-convergence}
is satisfied with $N \gsim T^4\log{(4n/{\delta})}$. Finally, it follows from
 the Cauchy-Schwarz inequality that
 $\supnorm{h} \le {B_h}/(2{\sqrt 2})$
 since $\sup_{\bx \in \cX} \Kint(\bx,\bx) \in (0,1/8]$.
It then follows from (\ref{eq:bounded-Linfty-function-class-f-h-bound})
 (\ref{eq:bounded-Linfty-function-class-h-bound}) and the above
 calculations that
\bals
\supnorm{f(\bW,\cdot)}  &\le \supnorm{h}  + \supnorm{e(\bx)} \le \frac{B_h}{2{\sqrt 2}} + w,
\eals
which completes the proof.
\end{proof}

We then have the following lemma about the Rademacher complexity of a localized function class $\set{f \in \cFext(B,w) \colon \Expect{P}{f^2} \le r}$.
\begin{lemma}
[Adapted from {\citep[Lemma~C.9]{Yang2025-generalization-two-layer-regression}}]
\label{lemma:LRC-population-NN}
For every $B,w > 0$ and every $r > 0$,
\bal\label{eq:LRC-population-NN}
&\cfrakR
\pth{\set{f \in \cFext(B,w) \colon \Expect{P}{f^2} \le r}}
\le \varphi_{B,w}(r),
\eal%
where
\bal\label{eq:varphi-LRC-population-NN}
\varphi_{B,w}(r) &\defeq
\min_{Q \colon Q \ge 0} \pth{({\sqrt r} + w) \sqrt{\frac{Q}{n}} +
B
\pth{\frac{\sum\limits_{q = Q+1}^{\infty}\lambdaint_q}{n}}^{1/2}} + w.
\eal

\end{lemma}

\begin{theorem}\label{theorem:LRC-population-NN}
Suppose $w  \in (0,1)$ and
$m,N$ satisfy (\ref{eq:m-cond-bounded-NN-class}) and
(\ref{eq:N-cond-bounded-NN-class}), respectively. Suppose also that $m\gsim n^{2/d}$ and $m/\log m\ge d$.
Suppose the neural network
$f_t = f(\bW(t),\cdot)$ is
trained by PGD in Algorithm~\ref{alg:PGD}
with the learning rate $\eta = \Theta(1)\in (0,8)$
on the random initialization $\bW(0)$, and $T \le \hat T$.
Then for every $t \in [T]$ and every $\delta \in (0,1)$, with probability at least
$1-2/n-\exp\pth{-\Theta(n)}-\delta-6\exp\pth{-\Theta(n\eps_n^2)}$ over the random noise $\bw$, the random training features $\bS$, the random sample $\bQ$, and the random initialization $\bW(0)$,
\bal\label{eq:LRC-population-NN-bound}
&\Expect{P}{(f_t-f^*)^2} - 2 \Expect{P_n}{(f_t-f^*)^2}
\lsim \min_{0 \le Q \le n} \pth{\frac{ B_0Q}{n} + B_h
\pth{\frac{\sum\limits_{q = Q+1}^{\infty}\lambdaint_q}{n}}^{1/2}} + \eps_n^2 +w,
\eal
where $B_0 \defeq {B_h}/{(2{\sqrt 2})} + 1 +{\mu_0}/(2{\sqrt 2})$.
%where $B_0,c_0$ are absolute positive constant depending on $\mu_0$.
\end{theorem}
% \begin{proof}
% This theorem follows from plugging the empirical training loss
% bound (\ref{empirical-loss-bound}) in Lemma~\ref{lemma:empirical-loss-bound}
% in the risk bound (\ref{eq:LRC-population-NN-bound-detail})
% in Theorem~\ref{theorem:LRC-population-NN-detail}, and
% noting that $\Prob{\cW_0} \ge 1-2/n$.
% \end{proof}
\begin{proof}
We first remark that the conditions on $m,N$ are required by
Lemma~\ref{lemma:empirical-loss-convergence} and Theorem~\ref{theorem:bounded-NN-class}.
Define $\cI_0=\set{\bW(0)\in\cW_0}$. Since $T\ge1$, the width assumptions imply the conditions of Theorem~\ref{theorem:good-random-initialization}, which gives $\Prob{\cI_0^c}\le2/n$. If $T$ is chosen from $\bS$, this bound follows by applying that theorem conditionally on $\bS$ and integrating.
Let $0<c_-\le c_+$ be fixed constants from Lemma~\ref{lemma:hat-eps-eps-relation}, and define the training-feature event
\bals
\mathcal A_n=\set{c_-\eps_n^2\le\hat\eps_n^2\le c_+\eps_n^2},
\qquad \Prob{\mathcal A_n^c}\le4\exp(-\Theta(n\eps_n^2)).
\eals
For each fixed realization of $\bS$ and $\bW(0)$ in $\cI_0$, Lemma~\ref{lemma:empirical-loss-convergence} and Theorem~\ref{theorem:bounded-NN-class} give the trajectory and decomposition event $\Omega$ with
\bals
\Prob{\Omega^c\mid\bS,\bW(0)}
\le\exp(-\Theta(n))+\delta+\exp(-\Theta(n\hat\eps_n^2)).
\eals
We regard $\Omega$ as the corresponding event in the joint probability space of $(\bW(0),\bw,\bQ,\bS)$. Both $\mathcal A_n$ and $\cI_0$ are measurable with respect to $(\bS,\bW(0))$. Since $\hat\eps_n^2\ge c_-\eps_n^2$ on $\mathcal A_n$, the tower property gives
\bals
\Prob{\mathcal A_n\cap\cI_0\cap\Omega^c}
&=\Expect{}{\indict{\mathcal A_n\cap\cI_0}\Prob{\Omega^c\mid\bS,\bW(0)}} \\
&\le\exp(-\Theta(n))+\delta+\exp(-\Theta(n\eps_n^2)).
\eals
On $\cI_0\cap\Omega$, we have $\bW(t)\in\cW(\bS,\bQ,\bW(0),T)$ and
\bals
f(\bW(t),\cdot)=f_t=h+e\in\cFext(B_h,w)
\eals
with $h\in\cH_{\Kint}(B_h)$ and $\supnorm{e}\le w$.

We then derive the sharp upper bound for $\Expect{P}{(f_t-f^*)^2}
$ by applying Theorem~\ref{theorem:LRC-population} to the function class
\bals
\cF = \set{F=\pth{f - f^*}^2 \colon f \in
\cFext(B_h,w)  }.
\eals
With $B_0 = {B_h}/(2{\sqrt 2}) + 1 +{\mu_0}/{(2{\sqrt 2})} \ge {B_h}/(2{\sqrt 2}) + w +{\mu_0}/(2{\sqrt 2})$, we have
$\supnorm{F} \le B^2_0$ with $F \in \cF$, so that
$\Expect{P}{F^2} \le B^2_0\Expect{P}{F}$.
Let $T(F) = B^2_0\Expect{P}{F}$ for $F \in \cF$. Then
$\Var{F} \le \Expect{P}{F^2} \le T(F) = B^2_0\Expect{P}{F}$.
We have
\bal\label{eq:LRC-population-NN-seg1}
\cfrakR \pth{\set{F \in \cF \colon T(F) \le r}}
&= \cfrakR
\pth{ \set{(f-f^*)^2 \colon f \in \cFext(B_h,w), \Expect{P}{(f-f^*)^2}
\le \frac r{B^2_0}}} \nonumber \\
&\stackrel{\circled{1}}{\le} 2B_0 \cfrakR \pth{\set{f -f^* \colon
 f \in \cFext(B_h,w), \Expect{P}{(f-f^*)^2} \le \frac{r}{B_0^2}}}\nonumber \\
&\stackrel{\circled{2}}{\le}  4B_0 \cfrakR \pth{ \set{f \in \cFext(B_h,w) \colon \Expect{P}{f^2} \le \frac{r}{4B_0^2}} },
\eal
where $\circled{1}$ is due to the contraction property of
Rademacher complexity in Theorem~\ref{theorem:RC-contraction}.
Since $f^* \in \cFext(B_h,w)$,
$f \in \cFext(B_h,w)$, we have $\frac{f-f^*}{2} \in \cFext(B_h,w)$ due to the fact that
$\cFext(B_h,w)$ is  symmetric and convex, and it follows that $\circled{2}$
holds.

It follows from (\ref{eq:LRC-population-NN-seg1})
and Lemma~\ref{lemma:LRC-population-NN} that
\bal\label{eq:LRC-population-NN-seg2}
B^2_0 \cfrakR \pth{\set{F \in \cF \colon T(F) \le r}}
&\le 4 B_0^3 \cfrakR \pth{ \set{f \colon f \in \cFext(B_h,w) , \Expect{P}{f^2} \le
\frac{r}{4B_0^2}}} \nonumber \\
&\le 4 B_0^3 \varphi_{B_h,w}\pth{\frac{r}{4B_0^2}} \defeq \psi(r).
\eal
$\psi$ defined as the RHS of (\ref{eq:LRC-population-NN-seg2}) is a sub-root function since it is nonnegative, nondecreasing and
$\frac{\psi(r)}{\sqrt r}$ is nonincreasing. Let $r^*$ be the fixed point of $\psi$, and $0 \le r \le r^*$. It follows from~\citet[Lemma 3.2]{bartlett2005} that
$0 \le r \le \psi(r) =  4 B_0^3 \varphi\pth{\frac{r}{4B_0^2}}$.
Therefore, by the definition of $\varphi$ in (\ref{eq:varphi-LRC-population-NN}),
for every $0 \le Q \le n$, we have
\bal\label{eq:LRC-population-NN-seg3}
\frac{r}{4 B_0^3} \le \pth{ \frac{\sqrt r}{2B_0} + w} \sqrt{\frac{Q}{n}} +
B_h
\pth{\frac{\sum\limits_{q = Q+1}^{\infty}\lambdaint_q}{n}}^{1/2}+w.
\eal
Solving the quadratic inequality (\ref{eq:LRC-population-NN-seg3}) for $r$, we have
\bal\label{eq:LRC-population-NN-seg4}
r \le \frac{8B_0^4 Q}{n} + 8B_0^3
\pth{ w \pth{\sqrt{\frac{Q}{n}}+1}
+ B_h \pth{\frac{\sum\limits_{q = Q+1}^{\infty}\lambdaint_q}{n}}^{1/2}
}.
\eal
(\ref{eq:LRC-population-NN-seg4}) holds for every $0 \le Q \le n$,
so we have
\bal\label{eq:LRC-population-NN-seg5}
r \le 8 B_0^3 \min_{0 \le Q \le n} \pth{\frac{ B_0Q}{n} +w \pth{\sqrt{\frac{Q}{n}}+1}
+ B_h
\pth{\frac{\sum\limits_{q = Q+1}^{\infty}\lambdaint_q}{n}}^{1/2}}.
\eal
It then follows from (\ref{eq:LRC-population-NN-seg2}) and
Theorem~\ref{theorem:LRC-population} that there is an event $\mathcal G_x$, depending only on $\bS$, with probability at least $1-\exp(-x)$, on which the following inequality holds simultaneously for every $f_t\in\cFext(B_h,w)$:
\bal\label{eq:LRC-population-NN-risk-E1-bound}
&\Expect{P}{(f_t-f^*)^2} - \frac{K_0}{K_0-1} \Expect{P_n}{(f_t-f^*)^2}-\frac{x\pth{11B_0^2+26B_0^2 K_0}}{n} \le \frac{704K_0}{B_0^2} r^*,
\eal
or
\bal\label{eq:LRC-population-NN-risk-E1-bound-simple}
&\Expect{P}{(f_t-f^*)^2}- 2 \Expect{P_n}{(f_t-f^*)^2} \lsim r^* +\ \frac {x}n
\eal
with $K_0 = 2$ in (\ref{eq:LRC-population-NN-risk-E1-bound}).
It follows from (\ref{eq:LRC-population-NN-seg5})
and (\ref{eq:LRC-population-NN-risk-E1-bound-simple}) that
\bal\label{eq:LRC-population-NN-risk-E1-final}
&\Expect{P}{(f_t-f^*)^2}- 2 \Expect{P_n}{(f_t-f^*)^2} \nonumber \\
&\lsim \min_{0 \le Q \le n} \pth{\frac{ B_0Q}{n} +w \pth{\sqrt{\frac{Q}{n}}+1}
+ B_h
\pth{\frac{\sum\limits_{q = Q+1}^{\infty}\lambdaint_q}{n}}^{1/2}} +
\frac {x}n.
\eal
Let $x=n\eps_n^2$. On $\cI_0\cap\mathcal A_n\cap\mathcal G_{n\eps_n^2}\cap\Omega$, (\ref{eq:LRC-population-NN-bound}) holds. The preceding bounds and a union bound give
\bals
\Prob{\pth{\cI_0\cap\mathcal A_n\cap\mathcal G_{n\eps_n^2}\cap\Omega}^c}
&\le\Prob{\cI_0^c}+\Prob{\mathcal A_n^c}+\Prob{\mathcal G_{n\eps_n^2}^c}
+\Prob{\mathcal A_n\cap\cI_0\cap\Omega^c} \\
&\le\frac2n+\exp(-\Theta(n))+\delta+6\exp(-\Theta(n\eps_n^2)).
\eals
This proves the stated joint probability bound without requiring independence among these events.
\end{proof}

\begin{lemma}\label{lemma:empirical-loss-bound}
Let $\eta\in(0,8)$. Suppose the neural network trained after the $t$-th step of PGD, $f_t = f(\bW(t),\cdot)$, satisfies $\bu(t) = f_t(\bS) - \by = \bv(t) + \be(t)$
with $\bv(t) \in \cV_t$ and $\be(t) \in \cE_{t,\tau}$, and $t \in [T]$ with $T \le \hat T$. If
\bal\label{eq:N-tau-cond-empirical-loss-convergence}
\tau \lsim  \frac{1}{\eta T},
\eal
Then for every $t \in [T]$, with probability at least
$1-\exp\pth{-c'n\hat \eps_n^2}$
over the random noise $\bw$, we have
\bal\label{empirical-loss-bound}
\Expect{P_n}{(f_t-f^*)^2} &\le
\frac{3}{\eta t} \pth{\frac{\mu_0^2}{2e} +\frac{1}{\eta T}+2}.
\eal
\end{lemma}
\begin{proof}

% It follows from Lemma~\ref{lemma:empirical-loss-convergence}
% and its proof that conditioned on an event $\Omega$ with probability at
% least $1 -  \exp\pth{-\Theta(n)} - \exp\pth{-\Theta(n \hat\eps_n^2)}-\Theta\pth{{nN}/{n^{c_d\eps^2_0/8}}} - \pth{1+2N}^{2d}\exp(-n^{c_x})$,
% $f \in \cFnn(\bS,\bQ,\bW(0),T)$ with
% $\bW(0) \in \cW_0$. Moreover, $f(\cdot) = f(\bW,\cdot)$ with $\bW = \set{\bbw_r}_{r=1}^m \in \cW(\bS,\bQ,\bW(0),T)$, and $\vect{\bW} = \vect{\bW_{\bS}} = \vect{\bW(0)} - \sum_{t'=0}^{t-1} \eta/n \bM \bZ_{\bS}(t') \bu(t')$ for some $t \in [T]$, where $\bu(t') \in \RR^n, \bu(t') = \bv(t') + \be(t')$ with $\bv(t') \in \cV_{t'}$ and $\be(t') \in \cE_{t',\tau}$ for all $t' \in [0,t-1]$.
% Also, $\bbe(t') = \bbe_1(t') + \bbe_2(t')$ with
% $\bbe_1(t') = -\pth{\bI_n-\eta\bKint_n}^{t'} \bw$
% and $\ltwonorm{\bbe_2(t')} \lsim {\sqrt n} \pth{\tau +{\eta T n^{c_x}}/{N}}$ for all $t' \in [0,t-1]$.

We have
\bal\label{eq:empirical-loss-bound-seg1}
f_t(\bS) = f^*(\bS) + \bw + \bv(t) + \be(t),
\eal
where $\bv(t) \in \cV_{t}$, $\be(t) \in \cE_{t,\tau}$,
$\bbe(t) = \bbe_1(t) + \bbe_2(t)$ with
$\bbe_1(t) = -\pth{\bI_n-\eta\bKint_n}^{t} \bw$
and $\ltwonorm{\bbe_2(t)} \lsim {\sqrt n} \tau$.
It follows from (\ref{eq:empirical-loss-bound-seg1}) that
\bsal\label{eq:empirical-loss-bound-seg2}
&\Expect{P_n}{(f_t-f^*)^2}
=\frac 1n \ltwonorm{f_t(\bS) - f^*(\bS)}^2  =\frac 1n \ltwonorm{\bv(t)+\bw+\be(t)}^2 \nonumber \\
&=\frac 1n \ltwonorm{-\pth{\bI- \eta \bKint_n }^t f^*(\bS)
+\pth{\bI_n -\pth{\bI_n-\eta\bKint_n}^t }\bw +\bbe_2(t)}^2 \nonumber \\
&\stackrel{\circled{1}}{\le} \frac 3n \sum\limits_{i=1}^n
\pth{1 - \eta \hlambdaint_i }^{2t}
\bth{{\bUint}^{\top} f^*(\bS)}_i^2 + \frac 3n \sum\limits_{i=1}^{n} \pth{1-
\pth{1-\eta \hlambdaint_i  }^t}^2
\bth{{\bUint}^{\top} \bw}_i^2 + \frac 3n \ltwonorm{\bbe_2(t)}^2 \nonumber \\
&
\nonumber \\
&\stackrel{\circled{2}}{\le} \frac{3\mu_0^2}{ 2e\eta t } + \frac 3n \sum\limits_{i=1}^{n} \pth{1-
\pth{1-\eta \hlambdaint_i  }^t}^2
\bth{{\bUint}^{\top} \bw}_i^2
+\frac3n\ltwonorm{\bbe_2(t)}^2\nonumber \\
&\le \frac{3}{\eta t} \pth{\frac{\mu_0^2}{2e} + \frac{1}{\eta T} }   +
3\cdot \underbrace{\frac 1n \sum\limits_{i=1}^{n} \pth{1-\pth{1-\eta \hlambdaint_i  }^t}^2
\bth{{\bUint}^{\top} \bw}_i^2}_{ \defeq E_{\eps}}
\esal

Here $\circled{1}$ follows from the Cauchy-Schwarz inequality,
$\circled{2}$ follows from (\ref{eq:yt-y-bound-seg1}) in the proof of
Lemma~\ref{lemma:yt-y-bound}. The final inequality follows from (\ref{eq:N-tau-cond-empirical-loss-convergence}), with the implicit constant chosen sufficiently small that $\|\bbe_2(t)\|_2\le\sqrt n/(\eta T)$.

We then derive the upper bound for $E_{\eps}$ on the RHS of
(\ref{eq:empirical-loss-bound-seg2}). We define a diagonal matrix $\bR \in \RR^{n \times n}$
with $\bR_{ii} = \big(1-(1-\eta \hlambdaint_i )^t\big)^2$ for $i \in [n]$. Then we have $E_{\eps} = 1/n \cdot  \tr{\bUint \bR {\bUint}^{\top}\bw \bw^{\top}}$.
 It follows from~\citet{quadratic-tail-bound-Wright1973}~that
\bal\label{eq:empirical-loss-bound-E-1}
&\Prob{1/n \cdot \tr{\bUint \bR {\bUint}^{\top} \bw \bw^{\top} } -
\Expect{}{1/n \cdot \tr{\bUint \bR {\bUint}^{\top} \bw \bw^{\top} }} \ge u}
\nonumber \\
&\le \exp\pth{-c \min\set{nu/\ltwonorm{\bR},n^2u^2/\fnorm{\bR}^2}}
\eal
holds for all $u > 0$, and $c$ is a positive constant. With
$\eta_t = \eta t$ for all $t \ge 0$, we have
\bal\label{eq:empirical-loss-bound-E-2}
&\Expect{}{1/n \cdot \tr{\bUint \bR {\bUint}^{\top}\bw \bw^{\top} }}
\le \frac {\sigma_0^2}n \sum\limits_{i=1}^n
\pth{1-\pth{1-\eta \hlambdaint_i }^t}^2 \nonumber \\
&\stackrel{\circled{1}}{\le}
\frac {\sigma_0^2}n \sum\limits_{i=1}^n
\min\set{1,\eta_t^2 (\hlambdaint_i)^2}
\le
\frac {{\sigma_0^2}\eta_t}n \sum\limits_{i=1}^n
\min\set{\frac{1}{\eta_t},\eta_t (\hlambdaint_i)^2} \nonumber \\
&\stackrel{\circled{2}}{\le}
\frac {{\sigma_0^2}\eta_t}n \sum\limits_{i=1}^n
\min\set{\frac{1}{\eta_t}, \hlambdaint_i}
= {{\sigma_0^2}\eta_t} \hat R_{\Kint}^2(\sqrt{{1}/{\eta_t}}) \le
\frac{1}{\eta_t}.
\eal
Here $\circled{1}$ follows from the fact that
$(1-\eta \hlambdaint_i )^t \ge \max\set{0,1-t\eta \hlambdaint_i}$,
and $\circled{2}$ follows from
$\min\set{a,b} \le \sqrt{ab}$ for any nonnegative numbers $a,b$.
Because $t \le T \le \hat T$, we have
$\hat R_{\Kint}(\sqrt{{1}/{\eta_t}}) \le 1/(\sigma_0 \eta_t)$, so the last inequality holds.

Moreover, we have the upper bounds for $\ltwonorm{\bR}$ and $\fnorm{\bR}$
as follows. First, we have
\bal\label{eq:empirical-loss-bound-E-3}
\ltwonorm{\bR} &\le \max_{i \in [n] }
\pth{1-\pth{1-\eta \hlambdaint_i }^t}^2 \le \max_{i\in[n]}\min\set{1,\eta_t^2 (\hlambdaint_i)^2}
\le 1.
\eal
We also have
\bal\label{eq:empirical-loss-bound-E-4}
\frac 1n \fnorm{\bR}^2 &=  \frac 1n
\sum\limits_{i=1}^n
\pth{1-\pth{1-\eta \hlambdaint_i }^t}^4
\le \frac {\eta_t}n \sum\limits_{i=1}^n
\min\set{\frac{1}{\eta_t},\eta_t^{3} (\hlambdaint_i)^4} \nonumber \\
&\le \frac {\eta_t}n \sum\limits_{i=1}^n
\min\set{\hlambdaint_i,\frac{1}{\eta_t}}
=\eta_t\hat R_{\Kint}^2(\sqrt{{1}/{\eta_t}})\le
\frac {1}{\sigma_0^2 \eta_t}.
\eal
Combining (\ref{eq:empirical-loss-bound-E-1})-(\ref{eq:empirical-loss-bound-E-4}), we have
\bals
&\Prob{1/n \cdot \tr{\bUint \bR {\bUint}^{\top} \bw \bw^{\top} } -
\Expect{}{1/n \cdot \tr{\bUint \bR {\bUint}^{\top} \bw \bw^{\top} }} \ge u} \nonumber \\
&\le \exp\pth{-c n\min\set{u, u^2\sigma_0^2 \eta_t}}.
\eals
Let $u = 1/(\eta t)$ in the above inequality. We have
\bals
\exp\pth{-c n\min\set{u, u^2\sigma_0^2 \eta_t}}
= \exp\pth{-c' n/\eta_t} \le
\exp\pth{-c'n\hat \eps_n^2}
\eals
where $c'  = c\min\set{1,\sigma_0^2}$, and the last inequality is due
to the fact that $1/\eta_t \ge \hat\eps_n^2$ since
$t \le T \le \hat T$.
It follows that with probability at least $1-
\exp\pth{- \Theta(n\hat\eps_n^2)}$,
\bal\label{eq:empirical-loss-bound-E-5}
E_{\eps}\le u+\frac{1}{\eta_t} = \frac{2}{\eta_t}.
\eal
It then follows from (\ref{eq:empirical-loss-bound-seg2}),
(\ref{eq:empirical-loss-bound-E-1})-(\ref{eq:empirical-loss-bound-E-5})
that
\bals
\Expect{P_n}{(f_t-f^*)^2} \le
\frac{3}{\eta t} \pth{\frac{\mu_0^2}{2e} +\frac{1}{\eta T}+2}
\eals
with probability at least $1-\exp\pth{-c'n\hat \eps_n^2}$. The same event works for every $t\in[T]$: apply the noise-energy bound at $T$ and use the monotonicity of $[1-(1-\eta\hlambdaint_i)^t]^2$, together with $1/(\eta T)\le1/(\eta t)$.
\end{proof}

\subsection{Proofs for the Approximate Uniform Convergence for the Kernel $\Kint$}

% \subsubsection{Proofs for the Approximate Uniform Convergence for the Kernel $\Kint$}

In this subsection, we present the main theorem, Theorem~\ref{theorem:hatKint-close-to-Kint-supnorm}, regarding the approximate uniform convergence of
$\hKint (\cdot, \bx')$ to $\Kint(\cdot,\bx')$ for every fixed $\bx' \in \cX$.
Theorem~\ref{theorem:hatKint-close-to-Kint-supnorm} is the formal version of Theorem~\ref{theorem:hatKint-close-to-Kint-supnorm-informal} in
Section~\ref{sec:detailed-roadmap-key-results}.
%We then establish Lemma~\ref{lemma:hatKint-gram-close-to-Kint-gram}
\begin{theorem}
\label{theorem:hatKint-close-to-Kint-supnorm}
For every fixed $\bx' \in \cX$ and every $\delta \in (0,1)$, with probability at least
$1-\delta$ over $\bQ = \set{\bbq_i}_{i=1}^N$, we have
\bal\label{eq:hatKint-close-to-Kint-supnorm}
\sup_{\bx \in \cX}\abth{\hKint(\bx,\bx') - \Kint(\bx,\bx')} \lsim \sqrt{\frac{\log(4/\delta)}{N}}.
\eal
As a result, with probability at least
$1-\delta$ over $\bQ$,
\bal
\sup_{\bx \in \cX,i \in [n]}\abth{\hKint(\bx,\bbx_i) - \Kint(\bx,\bbx_i)} \lsim  \sqrt{\frac{\log (4n/{\delta})}{N}}, \label{eq:hatKint-close-to-Kint-S} \\
\ltwonorm{\hbKint - \bKint} \lsim n \sqrt{\frac{\log (4n/{\delta})}{N}}.
\label{eq:hatKint-close-to-Kint-spectralnorm}
\eal
\end{theorem}
\begin{proof}
We define
\bal\label{eq:hatKint-close-to-Kint-supnorm-def-p}
p(\bq,\bx') \defeq \frac 1{N} \sum\limits_{j=1}^N K^{(s)}(\bq,\bbq_j)
K(\bbq_j,\bx'), \quad \forall \bq,\bx' \in \cX.
\eal
It follows from Theorem~\ref{theorem:Kint-bounded} and the fact that
$\abth{K(\bx,\bx')} \le 1/2$ for all $\bx,\bx' \in \cX$ that
$\sup_{\bq,\bx' \in \cX}\abth{p(\bq,\bx')}\le1/4$.
% Moreover, we have
% \bal\label{eq:hatKint-close-to-Kint-supnorm-seg1}
% p(\bq,\bx') = \frac {K^{(s)}(\bbq_i,\bbq_i) K(\bbq_i,\bx') }{N-1} + \frac{1}{N-1}\sum\limits_{j \neq i} K^{(s)}(\bq,\bbq_j) K(\bbq_j,\bx').
% \eal
We now fix $\bx' \in \cX$ in the following arguments.
It follows from (\ref{eq:concentration-RKHS-Hilbert-space-Ksint})
of Lemma~\ref{lemma:concentration-RKHS-Hilbert-space-Kint}
that
%the standard Hoeffding's inequality that,
for every $t > 0$ and every $i \in [N]$,
\bal\label{eq:hatKint-close-to-Kint-supnorm-seg2}
\Prob{ \norm{\frac{1}{N}\sum\limits_{j=1}^N K^{(s)}(\cdot,\bbq_j) K(\bbq_j,\bx')
-\bar K^{(s)} (\cdot,\bx') }{\cH_K} < t}
\ge 1-2\exp\pth{-\Theta(Nt^2)},
\eal
where $\bar K^{(s)} (\cdot,\bx') \defeq
\Expect{\bq}{K^{(s)}(\cdot,\bq) K(\bq,\bx')}$.
The following arguments are conditioned on the event that
(\ref{eq:hatKint-close-to-Kint-supnorm-seg2}) holds.

For each $i \in [N]$, we have  $K^{(s)}(\cdot,\bbq_i) \in \cH_K$, and
$\bar K^{(s)} (\cdot,\bx') \in \cH_K$.
It follows from (\ref{eq:hatKint-close-to-Kint-supnorm-seg2}) that, for all $\bq \in \cX$,
\bal\label{eq:hatKint-close-to-Kint-supnorm-seg3}
&\abth{p(\bq,\bx') - \bar K^{(s)} (\bq,\bx')}
=\abth{\iprod{p(\cdot,\bx') - \bar K^{(s)} (\cdot,\bx')}{K(\cdot,\bq)}_{\cH_K}} \nonumber \\
&\le \norm{\frac{1}{N}\sum\limits_{j=1}^N K^{(s)}(\cdot,\bbq_j) K(\bbq_j,\bx')-\bar K^{(s)} (\cdot,\bx') }{\cH_K}
\cdot \norm{K(\cdot,\bq)}{\cH_K}
\le  \frac{t}{\sqrt 2}.
\eal
Define
\bals
\bar K^{\mathop\mathrm{(int)}}(\bx,\bx')
\defeq \frac 1N \sum\limits_{j=1}^N K(\bx,\bbq_j) \bar K^{(s)}(\bbq_j,\bx'),
\quad \forall \bx,\bx' \in \cX.
\eals
Then it follows from the definition of $\hKint$ in (\ref{eq:hatKint-def})
that for all $\bx \in \cX$,
\bal\label{eq:hatKint-close-to-Kint-supnorm-seg4}
&\abth{\hKint(\bx,\bx') - \bar K^{\mathop\mathrm{(int)}}(\bx,\bx')}
= \abth{\frac 1N \sum\limits_{j=1}^N K(\bx,\bbq_j) p(\bbq_j,\bx')
-\frac 1N \sum\limits_{j=1}^N K(\bx,\bbq_j) \bar K^{(s)}(\bbq_j,\bx')}
\nonumber \\
&\le \frac 1N \sum\limits_{j=1}^N \abth{K(\bx,\bbq_j)}
\abth{p(\bbq_j,\bx')- \bar K^{(s)}(\bbq_j,\bx')}
\le  \frac{t}{2{\sqrt 2}},
\eal
where the last inequality follows from
(\ref{eq:hatKint-close-to-Kint-supnorm-seg3}).

Given the fixed $\bx' \in \cX$,
we now approximate $\Kint(\cdot,\bx')$ by
$\bar K^{\mathop\mathrm{(int)}}(\cdot,\bx')$.
First, it can be verified from the definition of $\bar K^{(s)}$
and Theorem~\ref{theorem:Kint-bounded} that
$\sup_{\bq,\bx'\in\cX}\abth{\bar K^{(s)}(\bq,\bx')}\le1/4$ and $\sup_{\bq\in\cX}\abth{p(\bq,\bx')}\le1/4$.
It then follows from (\ref{eq:concentration-RKHS-Hilbert-space-Kint}) of
Lemma~\ref{lemma:concentration-RKHS-Hilbert-space-Kint}
that
\bal\label{eq:hatKint-close-to-Kint-supnorm-seg5}
\Prob{\norm{\bar K^{\mathop\mathrm{(int)}}(\cdot,\bx') -
\Expect{\bq}{K(\cdot,\bq)\bar K^{(s)}(\bq,\bx')} }{\cH_K}> t}
\le 2\exp\pth{-\Theta(Nt^2)},
\eal
and we have
$\Expect{\bq}{K(\cdot,\bq)\bar K^{(s)}(\bq,\bx')} = \Kint(\cdot,\bx')$.
It follows from
(\ref{eq:hatKint-close-to-Kint-supnorm-seg4})
and (\ref{eq:hatKint-close-to-Kint-supnorm-seg5})
that with probability at least $1-4\exp\pth{-\Theta(Nt^2)}$,
for all $\bx \in \cX$,
\bals
&\abth{\hKint(\bx,\bx') - \Kint(\bx,\bx')}
\le \abth{\hKint(\bx,\bx') - \bar K^{\mathop\mathrm{(int)}}(\bx,\bx')}
+ \abth{\bar K^{\mathop\mathrm{(int)}}(\bx,\bx') -\Kint(\bx,\bx') }
\nonumber \\
&\le  \frac{t}{2{\sqrt 2}}
+ \norm{\bar K^{\mathop\mathrm{(int)}}(\cdot,\bx') -
\Expect{\bq}{K(\cdot,\bq)\bar K^{(s)}(\bq,\bx')} }{\cH_K}
\cdot \norm{K(\cdot,\bx))}{\cH_K} \le \frac{3t}{2{\sqrt 2}},
\eals
Choosing $t=C\sqrt{\log(4/\delta)/N}$ with $C$ sufficiently large proves (\ref{eq:hatKint-close-to-Kint-supnorm}).
Applying this bound with failure probability $\delta/n$ at each training feature and using the union bound proves (\ref{eq:hatKint-close-to-Kint-S}). The spectral-norm bound (\ref{eq:hatKint-close-to-Kint-spectralnorm}) follows from 
\bals
\ltwonorm{\hbKint-\bKint}\le n\max_{i,j\in[n]}\abth{\hKint(\bbx_i,\bbx_j)-\Kint(\bbx_i,\bbx_j)}.
\eals
\end{proof}

\begin{lemma}
\label{lemma:concentration-RKHS-Hilbert-space-Kint}
Suppose that $p$ is a measurable function defined on $\cX$ and
$\sup_{\bx \in \cX} \abth{p(\bx)} \le p_0<\infty$. Then for every $r > 0$,
\bal\label{eq:concentration-RKHS-Hilbert-space-Kint}
\Prob{\norm{\frac{1}{N} \sum_{i=1}^N K(\cdot,\bbq_i)p(\bbq_i)
-\Expect{\bq}{K(\cdot,\bq)p(\bq)} }{\cH_K}> r}
\le 2\exp\pth{-\Theta(Nr^2)}.
\eal
Similarly, for every $r > 0$,
\bal\label{eq:concentration-RKHS-Hilbert-space-Ksint}
\Prob{\norm{\frac{1}{N} \sum_{i=1}^N K^{(s)}(\cdot,\bbq_i)p(\bbq_i)
-\Expect{\bq}{K^{(s)}(\cdot,\bq)p(\bq)} }{\cH_K}> r}
\le 2\exp\pth{-\Theta(Nr^2)}.
\eal
\end{lemma}
\begin{proof}
We work in $\cB=\cH_K$ with its RKHS norm, which is used whenever
the norm subscript is omitted below. Since $K$ is continuous on the
compact space $\cX$, $\cH_K$ is separable. By the parallelogram
identity, $\cH_K$ is $(2,1)$-smooth in the sense of \citet{Pinelis1994}.
If $p\equiv0$, the conclusion is immediate. Otherwise choose $p_0>0$ with $\sup_{\bx\in\cX}\abth{p(\bx)}\le p_0$. We then construct the martingale $\set{f_k}_{k \in [N]}$. For each $k \in [N]$, we define
\bals
%\resizebox{1 \textwidth}{!}{$
f_k \defeq \Expect{}{\frac{1}{{p_0 \sqrt {2N}}} \sum\limits_{i=1}^N \pth{
K(\cdot,\bbq_i)p(\bbq_i)
-\Expect{\bq}{K(\cdot,\bq)p(\bq)}}\longmid
\cF_k}, \forall k \in [N],
%$},
\eals
where $\set{\cF_k}_{k=0}^N$ is an increasing sequence of $\sigma$-algebras, $\cF_k$ is the $\sigma$-algebra generated by $\set{\bbq_t}_{t=1}^k$. $\cF_0$ is the trivial $\sigma$-algebra so that
$f_0 = 0$. We note that
\bals
%\resizebox{1 \textwidth}{!}{
f_N &= \frac{1}{{p_0\sqrt {2N}}} \sum_{i=1}^N \pth{K(\cdot,\bbq_i)p(\bbq_i)
-\Expect{\bq}{K(\cdot,\bq)p(\bq)}}, \nonumber \\
d_k &= f_k - f_{k-1}
= \frac{1}{{p_0\sqrt {2N}}} \pth{K(\cdot,\bbq_k)p(\bbq_k)- \Expect{\bq}{K(\cdot,\bq)p(\bq)}}, \forall k \in [N],
%$}
\eals
and
$f^* = \max_{k \in [N]} \norm{f_k}{}$. For every $k \in [N]$,
we have
\bal\label{eq:concentration-RKHS-Hilbert-space-Kint-seg1}
\norm{d_k}{\cH_K}
&= \norm{\frac{1}{{p_0\sqrt {2N}}} \pth{K(\cdot,\bbq_k)p(\bbq_k)
- \Expect{\bq}{K(\cdot,\bq)p(\bq)}}}{\cH_K} \nonumber \\
&\stackrel{\circled{1}}{\le} \frac{1}{p_0\sqrt {2N}}
\pth{p_0 \norm{K(\cdot,\bbq_k)}{\cH_K} + p_0
\Expect{\bq}{\norm{K(\cdot,\bq)}{\cH_K}}}
\stackrel{\circled{2}}{\le} \frac 1{\sqrt N},
\eal
where $\circled{1}$ follows from the triangle inequality and
the Jensen's inequality, and $\circled{2}$ follows from the
fact that $\norm{K(\cdot,\bbq_k)}{\cH_K} \le 1/\sqrt{2}$.

Extend the martingale by $f_k=f_N$ and $\cF_k=\cF_N$ for $k>N$,
so that $d_k=0$ for $k>N$. The deterministic bound in
(\ref{eq:concentration-RKHS-Hilbert-space-Kint-seg1}) gives
$\sum_{k=1}^{\infty}\operatorname*{ess\,sup}_{\bQ}
\norm{d_k}{\cH_K}^2\le1$.
Applying \citet[Theorem~3.5]{Pinelis1994} with smoothness constant
$D=1$ and $b_*=1$, we have
$\Prob{f^* =\max_{k \in [N]} \norm{f_k}{}> r} \le
2\exp\pth{-\frac{r^2}{2}}$, and it follows that for every $r > 0$,
\bals
\Prob{\norm{\frac{1}{{p_0\sqrt {2N}}} \sum_{i=1}^N \pth{K(\cdot,\bbq_i)p(\bbq_i)-\Expect{\bq}{K(\cdot,\bq)p(\bq)}}}{\cH_K}> r} \le
2\exp\pth{-\frac{r^2}{2}},
\eals
and it follows that
\bals
\Prob{\norm{\frac{1}{N} \sum_{i=1}^N K(\cdot,\bbq_i)p(\bbq_i)-
\Expect{\bq}{K(\cdot,\bq)p(\bq)}}{\cH_K}> r}
\le 2\exp\pth{-\Theta(Nr^2)},
\eals
which completes the proof of
(\ref{eq:concentration-RKHS-Hilbert-space-Kint}) and the constant in $\Theta(Nr^2)$
depends on $p_0$.
% Applying the above inequality to the negative martingale
% $\set{-f_k}_{k=0}^N$, we have
% \bal\label{eq:concentration-RKHS-Hilbert-space-vec-seg1}
% \Prob{\norm{\frac{1}{N} \sum_{i=1}^N K(\cdot,\bbq_i)p(\bbq_i)
% -\Expect{\bq}{K(\cdot,\bq)p(\bq)} }{\cH_K}> r} \le 4\exp\pth{-8Nr},
% \eal

For each $i \in [N]$, we have $K^{(s)}(\cdot,\bbq_i) \in \cH_K$ according to
the definition of  $K^{(s)}$ in (\ref{eq:Kint-def}) and noting that
$\set{{\sqrt \lambda_j} e_j }$ is an orthonormal basis
of $\cH_K$. Moreover,
\[
\norm{K^{(s)}(\cdot,\bbq_i)}{\cH_K}^2=\sum_{j\ge1}\lambda_j^{2s-1}e_j(\bbq_i)^2=K^{(2s-1)}(\bbq_i,\bbq_i)\le\frac12,
\]
because $2s-1\ge1$ and Theorem~\ref{theorem:Kint-bounded} applies.
Repeating the proof of
(\ref{eq:concentration-RKHS-Hilbert-space-Kint}) to
the martingale
\bals
\resizebox{1 \textwidth}{!}{$
f^{(s)}_k = \Expect{}{\frac{1}{{p_0 \sqrt {2N}}} \sum\limits_{i=1}^N \pth{
K^{(s)}(\cdot,\bbq_i)p(\bbq_i)
-\Expect{\bq}{K^{(s)}(\cdot,\bq)p(\bq)}}\longmid
\cF_k}, \forall k \in [N] \cup \set{0}
$},
\eals
proves (\ref{eq:concentration-RKHS-Hilbert-space-Ksint}).
\end{proof}

\vspace{-.1in}
\subsection{More Results about Reproducing Kernel Hilbert Spaces}
%\begin{lemma}[In the proof of
%{\citet[Lemma 8]{RaskuttiWY14-early-stopping-kernel-regression}}]
%\label{lemma:bounded-Ut-f-in-RKHS}
%For any $f \in \cH_{K}(\mu_0)$, we have
%\bal\label{eq:b.ounded-Ut-f-in-RKHS}
%\frac 1n \sum_{i=1}^{\rank(\bK_n)} \frac{\bth{\bU^{\top}f(\bS)}_i^2}{\hat \lambda_i} \le \mu_0^2.
%\eal
%Similarly, for $f \in \cH_{\Kint}(\mu_0)$, we have
%\bals
%\frac 1n \sum_{i=1}^{\rank(\bKint_n)} \frac{\bth{{\bUint}^{\top}f(\bS)}_i^2}{\hlambdaint_i} \le \mu_0^2.
%\eals
%\end{lemma}

\begin{lemma}[In the proof of
{\citet[Lemma 8]{RaskuttiWY14-early-stopping-kernel-regression}}]
\label{lemma:bounded-Ut-f-in-RKHS}
Let $r_K=\rank(\bK_n)$ and
$r_{\mathrm{int}}=\rank(\bKint_n)$.
For every $f\in\cH_K(\mu_0)$,
\[
\bth{\bU^\top f(\bS)}_i=0,
\qquad i>r_K,
\]
and
\[
\frac1n f(\bS)^\top \pth{\bK_n}^{\dagger} f(\bS)
=
\frac1n\sum_{i=1}^{r_K}
\frac{\bth{\bU^\top f(\bS)}_i^2}{\hat\lambda_i}
\le \mu_0^2.
\]
Similarly, for every $f\in\cH_{\Kint}(\mu_0)$,
\[
\bth{{\bUint}^\top f(\bS)}_i=0,
\qquad i>r_{\mathrm{int}},
\]
and
\[
\frac1n f(\bS)^\top\pth{\bKint_n}^{\dagger} f(\bS)
=
\frac1n\sum_{i=1}^{r_{\mathrm{int}}}
\frac{\bth{{\bUint}^\top f(\bS)}_i^2}
     {\hlambdaint_i}
\le \mu_0^2.
\]
\end{lemma}
\begin{proof}
We first prove the results for the kernel $K$. Define the empirical
kernel operator $\hat T_{K,n} \colon \cH_K \to \cH_K$ by
\bals
\pth{\hat T_{K,n}g}(\cdot)
\defeq
\frac{1}{n}\sum_{j=1}^n
K(\bbx_j,\cdot)g(\bbx_j),
\qquad g\in\cH_K.
\eals
Recall that
$\bK_n
=
\bU\bSigma\bU^\top$, $\bSigma
=
\diag\pth{\hat\lambda_1,\ldots,\hat\lambda_n}$,
where
$\hat\lambda_1\ge\cdots\ge\hat\lambda_{r_K}>0$,
$\hat\lambda_{r_K+1}
=
\cdots
=
\hat\lambda_n
=
0$.
For every $i\in[r_K]$, let $\bU_i$ denote the $i$-th column of
$\bU$, and define
\bals
\hat e_i(\cdot)
\defeq
\frac{1}{\sqrt{n\hat\lambda_i}}
\sum_{j=1}^n
\bth{\bU_i}_j K(\bbx_j,\cdot).
\eals
It can be verified that $\set{\hat e_i}_{i=1}^{r_K}$ are the orthonormal
eigenfunctions of $\hat T_{K,n}$ associated with its positive
eigenvalues $\set{\hat\lambda_i}_{i=1}^{r_K}$.
We next show that the evaluation vector of every function in
$\cH_K$ has no component in the null space of $\bK_n$. For every
$i>r_K$, we have $\bK_n\bU_i=\bzero$, and hence
\bals
\norm{
\sum_{j=1}^n
\bth{\bU_i}_jK(\bbx_j,\cdot)
}{\cH_K}^2
&=
\bU_i^\top\bK\bU_i
=
n\bU_i^\top\bK_n\bU_i
=
0.
\eals
Consequently,
\bals
\bth{\bU^\top f(\bS)}_i
&=
\sum_{j=1}^n\bth{\bU_i}_j f(\bbx_j) =
\iprod{
f
}{
\sum_{j=1}^n
\bth{\bU_i}_jK(\bbx_j,\cdot)
}_{\cH_K}
=
0,
\qquad i>r_K.
\eals
For every $i\in[r_K]$, the reproducing property gives
\bals
\iprod{f}{\hat e_i}_{\cH_K}
&=
\frac{1}{\sqrt{n\hat\lambda_i}}
\sum_{j=1}^n
\bth{\bU_i}_j f(\bbx_j)
=
\frac{
\bth{\bU^\top f(\bS)}_i
}{
\sqrt{n\hat\lambda_i}
}.
\eals
Since
\bals
\bK_n^\dagger
=
\bU
\diag\pth{
\hat\lambda_1^{-1},
\ldots,
\hat\lambda_{r_K}^{-1},
0,\ldots,0
}
\bU^\top,
\eals
we have
\bals
\frac1n f(\bS)^\top\bK_n^\dagger f(\bS)
&=
\frac1n\sum_{i=1}^{r_K}
\frac{\bth{\bU^\top f(\bS)}_i^2}{\hat\lambda_i} =
\sum_{i=1}^{r_K}
\iprod{f}{\hat e_i}_{\cH_K}^2 \le
\norm{f}{\cH_K}^2
\le
\mu_0^2,
\eals
where the first inequality follows from Bessel's inequality. The results for $\Kint$ are proved by a similar argument.

\end{proof}

\begin{lemma}
\label{lemma:auxiliary-lemma-1}
For any positive real number $a \in (0,1)$ and positive integer $t$,
we have
\bal\label{eq:auxiliary-lemma-1}
(1-a)^t \le e^{-ta} \le \frac{1}{eta}.
\eal
\end{lemma}
\begin{proof}
The result follows from the facts that
$\log(1-a) \le -a$ for $a \in (0,1)$ and $\sup_{u \in \RR}
ue^{-u} \newline \le 1/e$.
\end{proof}

\begin{lemma}\label{lemma:hat-eps-eps-relation}
For all sufficiently large $n$, with probability at least $1-4\exp(-\Theta(n\eps_n^2))$,
\bal
\eps_n^2 \lsim\hat \eps_n^2, \quad \hat \eps_n^2 \lsim \eps_n^2.
\label{eq:bound-eps-n-hat-eps-n}
\eal
Similarly, with probability at least $1-4\exp(-\Theta(n\eps_{K,n}^2))$,
\bal
\eps_{K,n}^2 \lsim \hat \eps_{K,n}^2, \quad \hat \eps_{K,n}^2 \lsim \eps_{K,n}^2.
\label{eq:bound-eps-Kn-hat-eps-Kn}
\eal
%where $c_1$, $c_{\sigma_0,1}$, $c_{\sigma_0,2}$ are positive constants depending on $\sigma_0$.
\end{lemma}
\begin{proof}
We verify the hypotheses of \citet[Lemma~B.7]{yang2024gradientdescentfindsoverparameterized} for both kernels, $K$ and $\Kint$. The kernel $K$ is continuous and satisfies $K(\bx,\bx)=1/2$ for every $\bx\in\cX$. By Theorems~\ref{theorem:Kint-bounded} and~\ref{theorem:spectrum-Kint}, $\Kint$ is continuous and
\bals
\Kint(\bx,\bx)=\sum_{j\ge1}\lambda_j^{s+2}e_j(\bx)^2>0,
\qquad \forall\bx\in\cX.
\eals
The strict positivity follows from $\sum_{j\ge1}\lambda_j e_j(\bx)^2=K(\bx,\bx)=1/2$ and positivity of every $\lambda_j$. Compactness of $\cX$, together with $\lambda_1\le1/2$ and $s\ge1$, gives
\bals
0<c_{\mathrm{int}}\defeq\min_{\bx\in\cX}\Kint(\bx,\bx)
\le\max_{\bx\in\cX}\Kint(\bx,\bx)\le\frac18.
\eals
Here $c_{\mathrm{int}}$ may depend on the fixed design distribution and on $s$, but is independent of $n$. Both kernels satisfy the required diagonal bounds. Their largest population eigenvalues are $\lambda_1>0$ and $\lambdaint_1=\lambda_1^{s+2}>0$. For all sufficiently large $n$, the sample-size requirements
\bals
n\gsim\max\set{\lambda_1^{-1},(\lambdaint_1)^{-1}}
\eals
are satisfied. Applying the cited lemma to $K$ yields (\ref{eq:bound-eps-Kn-hat-eps-Kn}), and applying it to $\Kint$ yields (\ref{eq:bound-eps-n-hat-eps-n}).
\end{proof}

\begin{lemma}
[{\citet[Lemma B.9]{yang2024gradientdescentfindsoverparameterized}}]
\label{lemma:sub-root-fix-point-properties}
Suppose $\psi \colon [0,\infty) \to [0,\infty)$ is a sub-root function with the unique fixed point $r^*$. Then the following properties hold.

\begin{itemize}[leftmargin=8pt]
\item[(1)] Let $a \ge 0$. Then $\psi(r) + a$ as a function of $r$ is also a sub-root function with fixed point $r^*_a$, and
$r^* \le r^*_a \le r^* + 2a$.
\item[(2)] Let $b \ge 1$, $c \ge 0$. Then $\psi(br+c)$ as a function of $r$ is also a sub-root function with fixed point $r^*_b$, and
$r^*_b \le br^* +2c/b$.
\item[(3)] Let $b \ge 1$. Then $\psi_b(r) = b\psi(r)$ is also a sub-root function with fixed point $r^*_b$, and
$r^*_b \le b^2r^* $.
\end{itemize}
\end{lemma}

\section{More Results about $\cHKint$}
\label{sec:RKHS-Kint-more-results}

\begin{theorem}
\label{theorem:Kint-bounded}
The RHS of (\ref{eq:Ks-def}), $\sum\limits_{j \ge 1} \lambda_j^{s} e_j(\bx)e_j(\bx')$,
converges uniformly on $\cX \times \cX$ with $s \ge 1$, and $K^{(s)}$ is well-defined. Moreover, $\sup_{\bx,\bx' \in \cX} \abth{K^{(s)}(\bx,\bx')} \le 1/2$.
\end{theorem}
\begin{proof}
It follows from the proof of Mercer's theorem, such as that in
\citep{Sun2005-RHKS-noncompact-domain}, that the convergence
on the RHS of
\bals
K^{(s)}(\bx,\bx') = \sum\limits_{j \ge 1} \lambda_j^{s}
e_j(\bx)e_j(\bx'), \quad \forall \bx,\bx' \in \cX,
\eals
with $s \ge 1$ is uniform and absolute.
Moreover, by the definition of the eigenvalue and eigenfunction,
we have $T_K e_1 = \lambda_1 e_1$, so that
\bals
\lambda_1^2 &=
\lambda_1^2 \int_{\cX} e_1^2(\bx) \diff \mu(\bx) =
\int_{\cX} \pth{\lambda_1 e_1(\bx)}^2 \diff \mu(\bx)
= \int_{\cX} \pth{T_K e_1}(\bx)^2 \diff \mu(\bx)
\nonumber \\
&= \int_{\cX} \pth{\int_{\cX} K(\bx,\bx') e_1(\bx') \diff \mu(\bx')}^2
\diff \mu(\bx) \nonumber \\
&\le \int_{\cX} \pth{\int_{\cX} K^2(\bx,\bx') \diff \mu(\bx') \cdot
\int_{\cX} e_1^2(\bx') \diff \mu(\bx') } \diff \mu(\bx)
\le \frac 14,
\eals
where the first inequality is due to the H\"{o}lder's inequality.
It follows that $\lambda_j \le 1/2 < 1$ for all $j \ge 1$. As a result,
for all $\bx \in \cX$, we have
\bals
K^{(s)}(\bx,\bx) = \sum\limits_{j \ge 1} \lambda_j^{s}
e_j^2(\bx) \le \sum\limits_{j \ge 1} \lambda_j
e_j^2(\bx)= K(\bx,\bx) = \frac 12.
\eals
Also, for all $\bx,\bx' \in \cX$, it follows from the Cauchy-Schwarz inequality that
\bals
\abth{K^{(s)}(\bx,\bx')} = \abth{\sum\limits_{j \ge 1} \lambda_j^{s}
e_j(\bx)e_j(\bx') }\le \pth{\sum\limits_{j \ge 1} \lambda_j^s
e_j^2(\bx)}^{1/2} \cdot
\pth{\sum\limits_{j \ge 1} \lambda_j^s
e_j^2(\bx')}^{1/2} \le \frac 12.
\eals
\end{proof}

\begin{theorem}
\label{theorem:spectrum-Kint}
Let $\set{e_j}_{j \ge 1} \subseteq L^2(\cX,\mu)$ be
a countable orthonormal basis of $\overline{\range(T_K)}=(\ker T_K)^\perp$ consisting of the
eigenfunctions of the integral operator $T_K \colon L^2(\cX,\mu) \to L^2(\cX,\mu)$, $\pth{T_K f}(\bx) \defeq \int_{\cX} K(\bx,\bx') f(\bx') \diff \mu(\bx')$,
a positive, self-adjoint, and compact operator on $L^2(\cX,\mu)$.
Let $\set{\lambda_j}_{j \ge 1}$ be a sequence with $\frac 12 \ge \lambda_1 \ge \lambda_2 \ge \ldots > 0$ such that $e_j$ is the eigenfunction of $T_K$ with
$\lambda_j$ being the corresponding eigenvalue.
Then  $e_j$ is the eigenfunction of $T_{\Kint}$ with
$\lambda_j^{s+2}$ being the corresponding eigenvalue. That is,
$T_{\Kint} e_j = \lambda_j^{s+2} e_j$, so that
$\lambdaint_j = \lambda_j^{s+2}$ for all $j \ge 1$. Moreover, $T_{\Kint}$ vanishes on $\ker T_K$, and $\ker T_{\Kint}=\ker T_K$.
\end{theorem}
\begin{proof}
First, it follows from the Mercer's theorem that
\bals
K(\bv,\bv') = \sum\limits_{j \ge 1} \lambda_j
e_j(\bv)e_j(\bv'), \quad \forall \bv,\bv' \in \cX,
\eals
and the convergence on the RHS of the above equality is uniform
and absolute.
Then it follows from the definition of $\Kint$ in
(\ref{eq:Kint-def}) that
\bal\label{eq:spectrum-Kint-seg1}
&\Kint (\bx,\bx') =  \int_{\cX \times \cX} K(\bx,\bv) K^{(s)}(\bv,\bv')  K(\bv',\bx')  \diff \mu(\bv) \otimes \mu(\bv') \nonumber \\
&\stackrel{\circled{1}}{=} \int_{\cX}  \pth{\int_{\cX} \sum\limits_{j \ge 1} \lambda_j
e_j(\bx)e_j(\bv) \cdot
\sum\limits_{j \ge 1} \lambda_j^{s} e_j(\bv)e_j(\bv')
\diff \mu(\bv)} \cdot \sum\limits_{j \ge 1} \lambda_j
e_j(\bv')e_j(\bx') \mu(\bv') \nonumber \\
&\stackrel{\circled{2}}{=} \int_{\cX}
\pth{\sum\limits_{j \ge 1} \lambda_j^{s+1}
e_j(\bx)e_j(\bv')
} \cdot \sum\limits_{j \ge 1} \lambda_j
e_j(\bv')e_j(\bx') \mu(\bv')
\stackrel{\circled{3}}{=}
\sum\limits_{j \ge 1} \lambda_j^{s+2} e_j(\bx)e_j(\bx')
\eal
where $\circled{1}$ follows from the Fubini's Theorem,
and $\circled{2}$,$\circled{3}$ follow by the orthogonality of the
orthogonal basis $\set{e_j}_{j \ge 1}$.

It follows from (\ref{eq:spectrum-Kint-seg1}) that
for all $j \ge 1$,
\bals
&\pth{T_{\Kint} e_j}(\bx)
= \int_{\cX} \pth{\sum\limits_{j' \ge 1} \lambda_{j'}^{s+2} e_{j'}(\bx)e_{j'}(\bx')}
e_j (\bx') \diff \mu(\bx')
=\lambda_j^{s+2} e_j(\bx),
\eals
which proves that $\lambdaint_j = \lambda_j^{s+2}$ for all $j \ge 1$. If $h\in\ker T_K$, then $\iprod{h}{e_j}_{L^2(\cX,\mu)}=0$ for all $j$, so the same kernel expansion gives $T_{\Kint}h=0$. Conversely, if $T_{\Kint}h=0$, then $\lambda_j^{s+2}\iprod{h}{e_j}_{L^2(\cX,\mu)}=0$ for every $j$. Since every $\lambda_j$ is positive, $h\in(\overline{\range(T_K)})^\perp=\ker T_K$.
\end{proof}

%It is known, such as~\citet[Theorem 3.1]{du2018gradient-gd-dnns},
%that $\bK_n$ is non-singular. Based on this fact, we have the following
%propositions showing that
%$\bKint_n$ is also non-singular.
%\begin{proposition}
%\label{proposition:Kint-gram-nonsingular}
%If $\bbx_i \neq \bbx_j$ for all $i,j \in [n]$ and $i \neq j$, then
%$\bKint_n$ is also non-singular.
%\end{proposition}
%\begin{proof}
%\citep[Theorem 3.1]{du2018gradient-gd-dnns} shows that
%$\bK_n$ is non-singular.
%Define the feature mapping
%$\Phi(\bx) \defeq \bth{{\sqrt \lambda_1} e_1(\bx),
%{\sqrt \lambda_2} e_2(\bx), \ldots,   }$.
%Since $\bth{\bK_n}_{ij} = 1/n
%\cdot \Phi(\bbx_i)^{\top} \Phi(\bbx_j)$,
%the non-singularity of $\bK$ indicates that the feature maps
%on the data $\bS$,
%$\set{\Phi(\bbx_i)}_{i=1}^n$, are linearly independent.
%
%On the other hand, Theorem~\ref{theorem:spectrum-Kint}
%shows that the $\set{\lambda^{s'}_j, e_j}_{j \ge 1}$ are the eigenvalues and the corresponding eigenfunctions of the integral operator
%$T_{\Kint}$, where $s' = s+2$.
%Let $\tilde \Phi \defeq \bth{{\lambda_1^{\frac{s'}{2}}} e_1(\bx),
%{\lambda_2^{\frac{s'}{2}}} e_2(\bx), \ldots,   }$.
%Then $\bth{\bKint_n}_{ij} = 1/n \cdot \tilde
%\Phi(\bbx_i)^{\top} \tilde \Phi(\bbx_j)$.
%Because $\set{\Phi(\bbx_i)}_{i=1}^n$ are linearly independent,
%it can be verified by definition that
%$\set{\tilde \Phi(\bbx_i)}_{i=1}^n$ are also linearly independent,
%so that $\bKint_n$ is not singular.
%\end{proof}

\section{More Result about the Eigenvalue Decay Rate}
\label{sec:more-results-edr}

\begin{proof}
{\textbf{\textup{\hspace{-3pt}of
Proposition~\ref{proposition:EDR-general-distribution}}}.}
We first let $\lambda_j = \mu_j(K,\cX,\phi^*(\mu))$, and let $S(d-1) = 2\pi^{d/2}/\Gamma(d/2)$ denote the surface area of $\unitsphere{d-1}$.
Let $\hat \mu  = \phi^*(\mu)$ with probability density function $\hat p$.
For any measurable subset $\bA \subseteq \cX$, we have
\bal\label{eq:EDR-general-distribution-seg1} 
\hat \mu(\bA) = \int_{\bx \in \bA} \int_0^{\infty} p(r\bx) r^{d-1} \diff r \diff \tilde \sigma =S(d-1) \int_{\bx \in \bA} \int_0^{\infty} p(r\bx) r^{d-1} \diff r\diff \sigma(\bx), \eal
where $\tilde \sigma$ is the surface measure on $\cX$, and $\sigma = \tilde\sigma/S(d-1)$ is the uniform probability measure on $\cX$.
It follows from (\ref{eq:EDR-general-distribution-seg1}) that
\bals \frac{\diff \hat \mu}{\diff \sigma} = q, \eals
where, for $\sigma$-almost every $\bx\in\cX$,
\bal \label{eq:EDR-general-distribution-seg2} q(\bx) =S(d-1)\int_0^{\infty} p(r \bx) r^{d-1} \diff r &\lsim S(d-1) \int_0^{\infty}r^{d-1}({1+r^2})^{-(d+1)/2} \diff r \nonumber \\
&\lsim \frac{S(d-1)}{2} B(d/2,1/2), \eal
and $B$ denotes the Beta function. The last inequality of (\ref{eq:EDR-general-distribution-seg2}) follows from Lemma~\ref{lemma:beta-function-equiv}.

We have $B(d/2,1/2) = \sqrt{\pi} \frac{\Gamma(d/2)}{\Gamma((d+1)/2)} \asymp 1/\sqrt {d}$. It then follows from (\ref{eq:EDR-general-distribution-seg2})
that $q(\bx) \newline \lsim S(d-1)/\sqrt {d}$.
Set $\nu=\hat\mu=q\sigma$ and $Q=\|q\|_{L^\infty(\sigma)}$. The preceding estimate gives $Q<\infty$, and $\int_{\cX}q\diff\sigma=1$ implies $Q\ge1$. We compare the integral operators
\bals (T_{K,\xi}f)(\bx)=\int_{\cX}K(\bx,\by)f(\by)\diff\xi(\by), \qquad \xi\in\{\sigma,\nu\}. \eals
These are positive, self-adjoint, compact operators on $L^2(\xi)$. Let $\lambda_j^\sigma=\mu_j(K,\cX,\sigma)$ denote the positive eigenvalues of $T_{K,\sigma}$, counted with multiplicity and ordered nonincreasingly. By the spherical uniform EDR in \citet{HuWLC21-regularization-minimax-uniform-spherical,BiettiM19,BiettiB21}, there are constants $c_0,C_0>0$, depending only on $d$, such that
\bals c_0j^{-b}\le\lambda_j^\sigma\le C_0j^{-b}, \qquad b=\frac{d}{d-1},\quad j\ge1. \eals

\noindent For $f\in L^2(\nu)$, define $qf=0$ on $\{q=0\}$. Since $0\le q\le Q$, we have $qf\in L^2(\sigma)$ and
\bals \|qf\|_{L^2(\sigma)}^2 =\int_{\cX}q^2f^2\diff\sigma \le Q\int_{\cX}qf^2\diff\sigma =Q\|f\|_{L^2(\nu)}^2. \eals
Boundedness of $K$ justifies Fubini's theorem. Using $\diff\nu=q\diff\sigma$ gives
\bal\label{eq:EDR-general-distribution-quadratic-form} 
\iprod{T_{K,\nu}f}{f}_{L^2(\nu)}=\iint_{\cX^2}K(\bx,\by)f(\bx)f(\by)\diff\nu(\bx)\diff\nu(\by)=\iprod{T_{K,\sigma}(qf)}{qf}_{L^2(\sigma)}. \eal
Let $e_1,e_2,\ldots$ be real $L^2(\sigma)$-orthonormal eigenfunctions corresponding to $\lambda_1^\sigma,\lambda_2^\sigma,\ldots$. They belong to $L^2(\nu)$ because $q\le Q$. Fix $j\ge1$. If $f\in L^2(\nu)$ satisfies $\iprod{f}{e_a}_{L^2(\nu)}=0$ for $1\le a<j$, then $qf$ is orthogonal to these eigenfunctions in $L^2(\sigma)$. The spectral theorem gives
\bals \iprod{T_{K,\nu}f}{f}_{L^2(\nu)} =\iprod{T_{K,\sigma}(qf)}{qf}_{L^2(\sigma)} \le\lambda_j^\sigma\|qf\|_{L^2(\sigma)}^2 \le Q\lambda_j^\sigma\|f\|_{L^2(\nu)}^2. \eals
Write $\lambda_j=\mu_j(K,\cX,\nu)$, listing eigenvalues with multiplicity and appending zeros if the rank is finite. The orthogonality conditions impose at most $j-1$ linear constraints. The min--max principle on $L^2(\nu)$ therefore gives
\bals \lambda_j\le Q\lambda_j^\sigma\le QC_0j^{-b}. \eals

We now prove the lower bound $\lambda_j \gsim j^{-b}$ for all $j \ge 1$.
For $\tau>0$, let $E_\tau$ be the direct sum of all eigenspaces of $T_{K,\sigma}$ whose eigenvalues are at least $\tau$. When $m=\dim E_\tau>0$, choose a real $L^2(\sigma)$-orthonormal eigenbasis $e_1,\ldots,e_m$ of $E_\tau$, ordered so that $T_{K,\sigma}e_a=\lambda_a^\sigma e_a$. These eigenfunctions have continuous representatives, given by $e_a=(\lambda_a^\sigma)^{-1}T_{K,\sigma}e_a$.

We first claim that
\bal\label{eq:EDR-general-distribution-projector-diagonal} 
\sum_{a=1}^m e_a(\bx)^2=m, \qquad \forall\bx\in\cX. \eal
For an orthogonal matrix $\bR\in \RR^{d \times d}$, we define $(U_{\bR}f)(\bx)=f(\bR^{-1}\bx)$. Rotation invariance of $\sigma$ makes $U_{\bR}$ unitary. The identity $K(\bR\bx,\bR\by)=K(\bx,\by)$ and a change of variables give $T_{K,\sigma}U_{\bR}=U_{\bR}T_{K,\sigma}$. For every $f$ in an eigenspace with eigenvalue $\lambda\ge\tau$, commutation gives
\bals T_{K,\sigma}(U_{\bR}f)=U_{\bR}(T_{K,\sigma}f)=\lambda U_{\bR}f. \eals
Therefore, $U_{\bR}$ maps each such eigenspace into itself. Since $E_\tau$ is the direct sum of these complete eigenspaces, $U_{\bR}E_\tau\subseteq E_\tau$. The same argument applies to $U_{\bR^{-1}}=U_{\bR}^{-1}$, which gives $U_{\bR}E_\tau=E_\tau$. There is an orthogonal matrix $\bO_{\bR}\in\RR^{m\times m}$ such that
\bals \big(e_1(\bR^{-1}\bx),\ldots,e_m(\bR^{-1}\bx)\big)^\top =\bO_{\bR}\big(e_1(\bx),\ldots,e_m(\bx)\big)^\top. \eals
This equality holds everywhere because both sides are continuous and $\sigma$ has full support. Taking squared Euclidean norms shows that $F(\bx)=\sum_{a=1}^m e_a(\bx)^2$ is invariant under rotations. Transitivity of the rotation action on $\cX$ makes $F$ constant. Its constant value is
\bals F(\bx)=\int_{\cX}F(\by)\diff\sigma(\by) =\sum_{a=1}^m\int_{\cX}e_a(\by)^2\diff\sigma(\by)=m, \eals
which proves (\ref{eq:EDR-general-distribution-projector-diagonal}).

\noindent Fix $j\ge1$ and we set 
\bals
L=\lceil2Qj\rceil,  \quad \tau=c_0L^{-b}.
\eals 
Since $\lambda_L^\sigma\ge\tau$, the subspace $E_\tau$ has dimension $m\ge L$. We define $\bA\in\RR^{m\times m}$ by
\bals \bA_{ab}=\int_{\cX}q(\bx)e_a(\bx)e_b(\bx)\diff\sigma(\bx), \qquad a,b\in[m]. \eals
Orthonormality of $\set{e_a}_{a=1}^m$ and $0\le q\le Q$ give $0\preceq\bA\preceq Q\bI_m$. By (\ref{eq:EDR-general-distribution-projector-diagonal}) and normalization of $q$,
\bals \tr{\bA}=\int_{\cX}q(\bx)\sum_{a=1}^m e_a(\bx)^2\diff\sigma(\bx)=m. \eals
If $k$ eigenvalues of $\bA$ are at least $1/2$, then
\bals m\le kQ+\frac{m-k}{2}, \qquad k\ge\frac{m}{2Q-1}\ge\frac{m}{2Q}\ge j. \eals
For every $f\in L^2(\nu)$, the quadratic-form identity (\ref{eq:EDR-general-distribution-quadratic-form}) and the spectral theorem give
\bals \iprod{T_{K,\nu}f}{f}_{L^2(\nu)} =\iprod{T_{K,\sigma}(qf)}{qf}_{L^2(\sigma)} \ge\tau\sum_{a=1}^m\big|\iprod{f}{e_a}_{L^2(\nu)}\big|^2. \eals
Let $F_j\subseteq\RR^m$ be spanned by $j$ orthonormal eigenvectors of $\bA$ corresponding to the largest $j$ eigenvalues at least $1/2$. For $\bz\in F_j$, set $f_{\bz}=\sum_{a=1}^m z_ae_a$, viewed as an element of $L^2(\nu)$. Then
\bals \|f_{\bz}\|_{L^2(\nu)}^2=\bz^\top\bA\bz\ge\frac12\|\bz\|_2^2,\qquad \sum_{a=1}^m\big|\iprod{f_{\bz}}{e_a}_{L^2(\nu)}\big|^2 =\bz^\top\bA^2\bz\ge\frac12\bz^\top\bA\bz. \eals
Consequently, $\{f_{\bz}:\bz\in F_j\}$ is a $j$-dimensional subspace of $L^2(\nu)$ on which
\bals \iprod{T_{K,\nu}f_{\bz}}{f_{\bz}}_{L^2(\nu)} \ge\frac\tau2\|f_{\bz}\|_{L^2(\nu)}^2. \eals
The min--max principle on $L^2(\nu)$ gives
\bals \lambda_j\ge\frac\tau2=\frac{c_0}{2}\lceil2Qj\rceil^{-b} \ge\frac{c_0}{2(3Q)^b}j^{-b}, \eals
where $Q\ge1$ gives $\lceil2Qj\rceil\le3Qj$. This bound for every $j$ also proves that $T_{K,\nu}$ has infinitely many positive eigenvalues. Combining the upper and lower bounds yields
\bal \label{eq:EDR-general-distribution-spectrum-comparison} 
\frac{c_0}{2(3Q)^b}j^{-b}\le\mu_j(K,\cX,\phi^*\mu)\le QC_0j^{-b}, \qquad j\ge1. 
\eal
Since $d$ is fixed and $Q<\infty$ follows from the original density assumption, (\ref{eq:EDR-general-distribution-spectrum-comparison}) proves $\mu_j(K,\cX,\phi^*\mu)\asymp j^{-d/(d-1)}$.
\end{proof}

\begin{lemma}
\label{lemma:beta-function-equiv}
The beta function defined on the real numbers, $B(a,b)=\int_{0}^{1} t^{a-1}(1-t)^{b-1} \diff t$ for all $a,b > 0$, can be written as
\bals
B(a,b)= \int_{0}^{\infty} u^{a-1}(1+u)^{-(a+b)} \diff u.
\eals
\end{lemma}
\begin{proof}
This lemma can be verified by the change of variable $t=\frac{u}{1+u}$.
%We use the change of variables
%$t=\frac{u}{1+u}, u=\frac{t}{1-t}$,
%so that
%$\diff t=\frac{1}{(1+u)^2} \diff u, 1-t=\frac{1}{1+u}$.
%Substituting into the integral gives
%$B(a,b)
%=\int_{0}^{\infty}
%u^{a-1}(1+u)^{-(a+b)} \diff u$,
%for $a,b>0$.
\end{proof}

\vskip 0.2in
\bibliography{ref}

\end{document}